\documentclass[portrait]{article}
\usepackage[margin = 1in]{geometry}
\usepackage{setspace}
\usepackage{graphicx,amsmath,amsfonts,amssymb,booktabs,multirow,color}
\usepackage{xfrac}
\usepackage[section]{algorithm}
\usepackage{algorithmic}
\usepackage[hidelinks]{hyperref}
\usepackage{subcaption}

\DeclareMathOperator*{\argmax}{\arg\!\max}

\DeclareMathOperator*{\tr}{tr}

\makeatletter
\newenvironment{subtheorem}[1]{%
  \def\subtheoremcounter{#1}%
  \refstepcounter{#1}%
  \protected@edef\theparentnumber{\csname the#1\endcsname}%
  \setcounter{parentnumber}{\value{#1}}%
  \setcounter{#1}{0}%
  \expandafter\def\csname the#1\endcsname{\theparentnumber.\Alph{#1}}%
  \ignorespaces
}{%
  \setcounter{\subtheoremcounter}{\value{parentnumber}}%
  \ignorespacesafterend
}
\makeatother
\newcounter{parentnumber}

\usepackage{amsthm}

\newtheorem{theorem}{Theorem}

\title{Certifiably Optimal Sparse Inverse Covariance Estimation}

\author{Dimitris Bertsimas \and Jourdain Lamperski \and Jean Pauphilet}
\date{%
	Operations Research Center, Massachusetts Institute of Technology, Cambridge, MA. \\ 
	\texttt{ \{dbertsim, jourdain, jpauph\}@mit.edu } \\[2ex]%
    June 2019
}
\begin{document}

\maketitle 

\begin{abstract}
We consider the maximum likelihood estimation of sparse inverse covariance matrices. We demonstrate that  current heuristic approaches primarily encourage robustness, instead of the desired sparsity. We give a novel approach that solves the cardinality constrained likelihood problem to certifiable optimality. The approach uses techniques from mixed-integer optimization and convex optimization, and provides a high-quality solution with a guarantee on its suboptimality, even if the algorithm is terminated early. Using a variety of synthetic and real datasets, we demonstrate that our approach can solve problems where the dimension of the inverse covariance matrix is up to $1,000$s. We also demonstrate that our approach produces significantly sparser solutions than Glasso and other popular learning procedures, makes less false discoveries, while still maintaining state-of-the-art accuracy. 
\end{abstract}

\section{Introduction} \label{sec:introduction} 
Estimating inverse covariance (precision) matrices is a fundamental task in modern multivariate analysis. Applications include undirected Gaussian graphical models \cite{lauritzen1996graphical}, high dimensional discriminant analysis \cite{cai2011constrained}, portfolio allocation \cite{fan2008high,fan2012vast}, complex data visualization \cite{tokuda2011visualizing}, amongst many others, see \cite{fan2014challenges} for a review. For example, in the context of undirected Gaussian graphical models, estimating the precision matrix corresponds to inferring the conditional independence structure on the related graphical model; zero entries in the precision matrix indicate that variables are conditionally independent.  

{Sparsity of the true precision matrix is a prevailing assumption \cite{yuan2007model,bickel2008covariance,lam2009sparsistency,el2010high,rigollet2012estimation} for two reasons.} 
\begin{enumerate} 
\item The covariance matrix is often estimated empirically using the maximum likelihood estimator: 
\begin{equation}
\label{eq:empCov}
\overline{\mathbf{\Sigma}} = \dfrac{1}{n} \sum_{i=1}^n (x^{(i)}-\bar x) (x^{(i)}-\bar x)^T,
\end{equation}
where the number of samples $n$ can be lower than the space dimension $p$. When this is the case, it is known that the empirical covariance matrix\footnote{Note that $\overline{\mathbf{\Sigma}}$ is not the only estimate of the covariance matrix. In particular,  $\tfrac{n}{n-1} \overline{\mathbf{\Sigma}}$ is a widely-used unbiased estimator of the covariance matrix. In this paper, we will only consider $\overline{\mathbf{\Sigma}}$, which we might refer to as the empirical or sample covariance matrix.} $\overline{\mathbf{\Sigma}}$ is singular, and thus does not accurately model the true covariance matrix. Moreover, the empirical covariance matrix can not be inverted to obtain an estimate of the precision matrix. Assuming sparsity of the true precision matrix is required for the precision matrix estimation problem to be well-defined. 

\item {In many applications, we use models to improve our knowledge of a given phenomenon and it is fair to admit that humans are limited in their ability to understand complex models. As Rutherford D. Roger said `We are drowning in information but starving for knowledge'. Models which only involve a small number variables, i.e. sparse models, are inherently simple. Sparse models with high predictive power can thus be extremely valuable in practice. We refer skeptic readers to the first chapter of \cite{hastie2015statistical}, which makes a strong case for sparsity in statistical learning.}
\end{enumerate} 
The most common method for encouraging sparsity in precision matrix estimation involves solving a $\ell_1$-regularized maximum likelihood problem. The problem is convex and can be solved in high dimensions. Though this approach is tractable, solutions suffer from similar drawbacks as Lasso solutions in linear regression \cite{bertsimas2016best}. For example, one drawback is the $\ell_1$-penalty introduces extra bias when estimating nonzero entries in the precision matrix with large absolute values \cite{lam2009sparsistency}.

In this paper, we seek to confront these drawbacks by solving the cardinality constrained optimization problem for which the $\ell_1$-regularized problem is a convex surrogate. The cardinality constrained problem parallels the relation the best subset selection (or feature selection) problem plays in linear regression with Lasso. The main goal of this work is to solve the cardinality constrained problem for problem sizes of interest, and compare the solutions with current approaches. A summary of the contributions in this paper is given below.  

\begin{enumerate} 
\item Recent results in linear regression establish that Lasso can be viewed as a robust optimization problem for an appropriately chosen uncertainty set \cite{xu2009robust,bertsimas2011theory}. {In a seminal paper on precision matrix estimation, \cite{banerjee2008model} already uncovered a similar connection, suggesting that the $\ell_1$-regularization approach is primarily encouraging robustness and that sparsity is a fortunate by-product. We generalize their result and show that a wide family of regularization can indeed be viewed as a robust version of the inverse covariance estimation problem.}

\item We formulate the cardinality constrained maximum likelihood problem for the inverse covariance matrix as a binary optimization problem. {We show that the resulting discrete optimization problem is non-smooth in general, but that adding some well-chosen regularization penalty leads to a smooth convex discrete optimization problem. In particular, we show that the well-known big-$M$ formulation  or the Ridge regularization term satisfy this property.}

\item We propose a combination of {outer-approximation algorithm} and first-order methods to solve the mixed-integer convex problem. To our knowledge, this is the first time in which {such a scheme} is used to solve a mixed-integer nonlinear optimization problem with semidefinite constraints. It is well-known that problems of this type are notoriously hard to solve, and we observe that our approach significantly outperforms available mixed-integer nonlinear solvers. An advantage of our approach over existing approaches is that it provides near optimal solutions fast, and a guarantee on the solutions suboptimality if the method is terminated early.

\item We report computational results with both synthetic and real-world datasets that show that our proposed approach can deliver near optimal solutions in a matter of seconds, and provably optimal solutions in a matter of minutes for {$p$ in the $100$s and $k$ in the $10$s. The algorithm also provides high-quality solutions to problems in the $1,000$s, 
but a certificate of optimality is  more computationally expensive  for those sizes.}

\item We investigate empirically statistical properties of solutions for the cardinality constrained problem. We compare solutions with $\ell_1$-regularized estimates and other popular learning procedures, and observe that cardinality-constrained estimates {recover the sparsity pattern of the true underlying precision matrix with comparable accuracy as state-of-the-art but significantly better false detection rate and predictive power.} 

\item {Finally, we show the modeling power of our framework and illustrate how it can be easily adapted to estimate Gaussian graphical with more structural information.}
\end{enumerate}  

The structure of the paper is as follows: In Section \ref{sec:overview}, we describe the problem of interest and provide a more detailed overview of relevant results from the literature. We generalize existing results about the equivalence between regularization and robustness. From this perspective, $\ell_1$-regularized approaches primarily encourage robustness instead of sparsity, which could  explain the known drawbacks of these techniques. In Section \ref{sec:IO} (supplemented by Appendix \ref{sec:A.dual}), we provide a mixed-integer formulation for the cardinality-constrained problem. Though non-smooth in general, we show that adding big-$M$ constraints or a ridge penalty term turns the problem  into a smooth convex integer optimization problem, for which we propose an efficient cutting-plane procedure. We also discuss practical implementation and parameter tuning in Section \ref{sec:IO.cv} and Appendix \ref{sec:A.bigM}. In Section \ref{sec:covsel}, we describe and numerically compare first-order and coordinate descent methods to solve variants of the covariance selection problem, used in our algorithm to provide valid cuts. We perform a variety of computational tests in Section \ref{sec:exp} and Appendix \ref{sec:A.comptime}, and use synthetic and real datasets to assess the algorithmic and statistical performance of our approach. Section \ref{sec:structural} illustrates the modeling power of our approach by discussing  extensions to cases where structural information about the correlation structure is available. In Section \ref{sec:conclusion}, we provide concluding remarks. 

% -------------------------------------------------------------------------------------------------
% SECTION: OVERVIEW AND PRELIMINARIES 
% -------------------------------------------------------------------------------------------------
\section{Overview and Preliminaries} \label{sec:overview} 
In this section, we provide a description of the problem formulation and an overview of current approaches for inducing sparsity in inverse covariance estimation. {Previous work \cite{banerjee2008model} showed that the $\ell_1$-regularization approach is equivalent to a robust optimization problem with an appropriately chosen uncertainty set. We generalize their result and discuss practical implications.} In particular, this equivalence suggests that current approaches are primarily encouraging robustness, not sparsity. 

% -------------------------------------------------------------------------------------------------
% PROBLEM DESCRIPTION 
% -------------------------------------------------------------------------------------------------
\subsection{Problem Description} \label{sec:overview.description} 
Let us consider a Gaussian random variable $X \sim N(\boldsymbol{\mu},\mathbf{\Sigma})$ with unknown mean $\boldsymbol{\mu} \in \mathbb{R}^p$ and covariance $\mathbf{\Sigma} \in S_{++}^p$, where $S_{++}^p$ denotes the set of symmetric positive definite matrices in $\mathbb{R}^{p \times p}$. Given a random sample $x^{(1)},...,x^{(n)}$ of $X$, we seek to estimate the precision matrix $\mathbf{\Sigma}^{-1}$. Let $\overline{\mathbf{\Sigma}} \in \mathbb{R}^{p \times p}$ be the empirical covariance matrix corresponding to the $n$ observations as defined in \eqref{eq:empCov}. The maximum likelihood estimate of $\mathbf{\Sigma}^{-1}$ is the solution of the optimization problem 
\begin{equation}  \label{eq:MLE} 
\min_{\mathbf{\Theta} \succ \mathbf{0}} \quad \langle \overline{\mathbf{\Sigma}}, \mathbf{\Theta} \rangle - \log \det \mathbf{\Theta},
\end{equation}
where the expression $\langle \cdot, \cdot \rangle$ is the usual trace inner product $\langle \overline{\mathbf{\Sigma}}, \mathbf{\Theta} \rangle = \tr(\overline{\mathbf{\Sigma}}^{\top} \mathbf{\Theta})$ and the objective function in \eqref{eq:MLE} is the negative Gaussian log-likelihood of the data \cite{yuan2007model}. 

As mentioned in introduction, a more interesting problem in practice is the cardinality-constrained version of \eqref{eq:MLE} 
\begin{equation}  \label{eq:L0MLE} 
\min_{\mathbf{\Theta} \succ \mathbf{0}} \quad \langle \overline{\mathbf{\Sigma}}, \mathbf{\Theta} \rangle - \log \det \mathbf{\Theta} \quad \text{s.t.} \quad \Vert \mathbf{\Theta} \Vert_0 \leqslant k,
\end{equation}
where $k \in \mathbb{Z}_+$, and $\Vert \mathbf{\Theta} \Vert_0 := \sum_{i > j} 1_{\Theta_{ij} \neq 0}$ counts the number of nonzero entries in the strictly lower triangular part of $\mathbf{\Theta}$. 

Problem \eqref{eq:L0MLE} parallels the role best subset selection plays in the context of linear regression. Like best subset selection, the cardinality constraint makes it computationally challenging and indeed NP-hard {\cite{chickering1996learning}}. There is also the extra difficulty that the problem is a minimization over positive definite matrices $S_{++}^p$. To our knowledge, the problem has yet to be considered in the literature as a discrete optimization problem over positive definite matrices. Thus, this paper provides the first {provably exact} optimization approach for solving Problem \eqref{eq:L0MLE}. { Closest to our approach are recent works for approximately solving a variant of Problem \eqref{eq:L0MLE} with an $\ell_0$ penalty instead of a constraint. \cite{marjanovic2015l_} propose a coordinate descent method to find good stationary solutions. \cite{liu2016sparse} approximate the $\ell_0$ pseudo-norm by a series of ridge penalties and implement a variant of the alternating direction method of multipliers. }

At the core of our methodology is the exploitation of novel techniques in discrete optimization. Recently, best subset selection and other cardinality constrained problems have been { solved  in high dimensions, } using discrete optimization \cite{bertsimas2014least,bertsimas2016best,bertsimas2017sparse}. These approaches exploit the significant progress  in mixed-integer optimization in the past decades and motivate our approach.

% -------------------------------------------------------------------------------------------------
% NOTATIONS
% -------------------------------------------------------------------------------------------------
\subsection{Notations} \label{sec:overview.notations}
{In the remaining of the paper, we will use bold characters to denote matrices or matrix-valued functions. Unless otherwise stated, all norms on matrices are vector norms and matrices are $p \times p$ matrices.

Let us recall some linear algebra identities, which will be useful in Section \ref{sec:covsel.cd}. For any invertible matrix $\mathbf{A}$ and vectors $u$, $v$, we can compute the determinant of  $\mathbf{A} + u v^T$ \cite[][Eqn. 6.2.3]{meyer2000matrix}
\begin{align*}
\det (\mathbf{A} + u v^T) = \det (\mathbf{A}) \, (1+v^T \mathbf{A}^{-1} u ),
\end{align*}
and its inverse \cite[Woodbury-Sherman-Morrison Formula in ][Eqn. 3.8.2]{meyer2000matrix}
\begin{align*}
(\mathbf{A} + u v^T)^{-1} = \mathbf{A}^{-1} - \dfrac{1}{1+v^T \mathbf{A}^{-1} u} \mathbf{A}^{-1} u v^T \mathbf{A}^{-1} .
\end{align*}

By default, all vectors are $p$-dimensional vectors. We will denote by $e_i$, $i=1,\dots,p$ the unit vectors with $1$ at the $i$th coordinate and zero elsewhere, and $e$ the vector of all ones. }

% -------------------------------------------------------------------------------------------------
% CURRENT APPROACHES
% -------------------------------------------------------------------------------------------------
\subsection{Current Approaches} \label{sec:overview.compet}
A variety of convex and nonlinear based optimization methods have been proposed to induce sparsity using the maximum likelihood problem \cite{fan2016overview}. Many of these methods can be interpreted as convex relaxation for Problem \eqref{eq:L0MLE}, the most common of which being the $\ell_1$-regularized negative log-likelihood minimization
\begin{equation}  \label{eq:L1MLE} 
\min_{\mathbf{\Theta} \succ \mathbf{0}} \quad \langle \overline{\mathbf{\Sigma}}, \mathbf{\Theta} \rangle - \log \det \mathbf{\Theta} + \lambda \Vert \mathbf{\Theta} \Vert_1,
\end{equation}
where $\Vert \mathbf{\Theta} \Vert_1 := \sum_{i,j} | \Theta_{ij} |$ is the $\ell_1$ vector norm. In practice, it has  been observed that the penalty term shrinks the coefficients of $\mathbf{\Theta}$ towards zero, and produces a sparse solution by setting many coefficients equal to zero. Problem \eqref{eq:L1MLE} was originally motivated by the development and successes of Lasso as a convex surrogate for the best subset selection problem \cite{yuan2007model}. The problem is well-studied in the literature \cite{yuan2007model,banerjee2008model,friedman2008sparse,rothman2008sparse,scheinberg2009sinco} and solved efficiently with a block coordinate descent procedure. {\cite{banerjee2008model} originally proposed the block coordinate descent schema and solved each sub-problem using Nesterov's first-order method. \cite{friedman2008sparse} then suggested a modified version of the algorithm, commonly referred to as Graphical Lasso or Glasso for each sub-problem is reformulated as a Lasso regression problem and solved as such. \cite{mazumder2012exact, mazumder2012graphical} then further improved the Glasso algorithm through smart feature screening rules. More recently, \cite{krishnamurthy201118} used coordinate descent to solve each sub-problem and released an \verb|R| package which can solve \eqref{eq:L1MLE} for a whole regularization path in a short amount of time - within a minute for $p=1,000$. Coordinate descent \cite{scheinberg2009sinco}, alternating linearization \cite{scheinberg2010sparse}, quadratic approximation and Newton's method \cite{hsieh2011sparse,oztoprak2012newton,hsieh2013big}, and stochastic proximal methods \cite{atchade2015scalable} have also been explored.}

{In earlier work, \cite{meinshausen2006high} proposed an efficient algorithm to discover the sparsity pattern of $\mathbf{\Sigma}^{-1}$ by fitting a Lasso model to each variable, using the others as predictors. It has later been shown \cite{banerjee2008model,friedman2008sparse} that their approach can be viewed as an approximation of Problem \eqref{eq:L1MLE}. More recently, \cite{fattahi2017graphical} proposed a simple thresholding heuristic and explored its connection with the graphical lasso \eqref{eq:L1MLE} }

Though the problem is tractable, it shares in the statistical shortcomings of its motivator, Lasso. Problem \eqref{eq:L1MLE} leads to biased estimates because the $\ell_1$-norm penalty term penalizes large entries more than the smaller entries \cite{lam2009sparsistency}. Accordingly, upon increasing the degree of regularization, \eqref{eq:L1MLE} sets more entries of $\mathbf{\Theta}$ to zero but leaves true predictors {outside of the support}. Thus, as soon as certain regularity conditions on the data are violated, Problem \eqref{eq:L1MLE} becomes suboptimal as a variable selector and in terms of delivering a model with good predictive performance. In contrast, Problem \eqref{eq:L0MLE} chooses variables to enter the active set without shrinking the entries in $\mathbf{\Theta}$. \cite{lam2009sparsistency} discuss other statistical shortcomings of \eqref{eq:L1MLE}. 
 
To address these shortcomings, other relaxation of \eqref{eq:L0MLE} have been proposed using smooth nonconvex penalties such as smoothly clipped absolute deviation (SCAD) \cite{fan2001variable} and minimax concave penalty (MCP) \cite{zhang2010nearly}, which are folded concave penalties that do not introduce extra bias for estimating nonzero entries with large absolute values. Theoretical properties of these methods are well studied \cite{rothman2008sparse,lam2009sparsistency}. However, these formulations are nonconvex and cannot provide a guarantee on how close their optimal solution is to the optimal solution of Problem \eqref{eq:L0MLE}. 

Estimators and approaches other than using maximum likelihood have also been proposed for inducing sparsity. Two such estimators are the constrained $\ell_1$-minimization for inverse matrix estimation (CLIME) estimator \cite{cai2011constrained} and the graphical Dantzig selector \cite{yuan2010high}. Rank and factor based methods have also been proposed; for a more complete survey of the different methods, see \cite{fan2016overview}. 

{From an optimization perspective, mixed-integer semi-definite optimization (MI-SDP) has received a lot of attention in recent years, for they naturally appear in robust optimization problems with ellipsoidal uncertainty sets \cite{ben2009robust} or as reformulations of combinatorial problems \cite{sotirov2012sdp}.  Problem-specific MI-SDP strategies have been developed for problems such as binary quadratic programming \cite{helmberg1998solving}, robust truss topology \cite{yonekura2010global} or the max-cut problem \cite{rendl2010solving}. More recently, rounding and Gomory cuts \cite{ccezik2005cuts,atamturk2010conic}, branch-and-bound \cite{gally2018framework} and outer-approximation schemes \cite{lubin2018polyhedral} have also been developed, in an attempt to provide the same level of general-purpose solvers for MI-SDP as there are for mixed-integer linear optimization. Our approach is similar to the outer-approximation procedure described by \cite{lubin2018polyhedral} but leverages the specific dependency between the binary and continuous variables in our problem. It also disconnects the combinatorial aspect of the problem from its SDP component, allowing us to benefit both from advances in mixed-integer linear optimization and tailor-made semidefinite strategies.} 

% -------------------------------------------------------------------------------------------------
% ROBUST INTERPRETATION 
% -------------------------------------------------------------------------------------------------
\subsection{Equivalence between Regularization and Robustness}  \label{sec: overview.robust}
{As originally enunciated by \cite{banerjee2008model}, the $\ell_1$-regularization in \eqref{eq:L1MLE} is the aftermath of a robust optimization problem. } Indeed, one can prove a clear equivalence between regularization and robustification in the case of sparse inverse covariance problems:

\begin{subtheorem}{theorem}\label{thm:robust}
\begin{theorem}\label{thm:robust.A}
For any vector norm $\Vert \cdot \Vert$, 
\begin{align*}
\min_{\mathbf{\Theta} \succ 0} \: \langle \overline{\mathbf{\Sigma}}, \mathbf{\Theta} \rangle - \log \det \mathbf{\Theta} + \lambda \| \mathbf{\Theta} \| &= \min_{\mathbf{\Theta} \succ 0} \: \max_{ \mathbf{U}: \|  \mathbf{U} \|_\star \leqslant \lambda} \: \langle \overline{\mathbf{\Sigma}} + \mathbf{U}, \mathbf{\Theta} \rangle - \log \det \mathbf{\Theta},
\end{align*}
where $\| \cdot \|_\star$ denotes the dual norm of $\Vert \cdot \Vert$. \end{theorem}

{\begin{theorem}\label{thm:robust.B}
For any $(p,q)$-induced norm $\Vert \cdot \Vert_{(p,q)} $, 
\begin{align*}
\min_{\mathbf{\Theta} \succ 0} \: \langle \overline{\mathbf{\Sigma}}, \mathbf{\Theta} \rangle - \log \det \mathbf{\Theta} + \lambda \| \mathbf{\Theta} \|_{(p,q)}  &= \min_{\mathbf{\Theta} \succ 0} \: \max_{ \mathbf{U} \in \mathcal{U}_{(p,q)} } \: \langle \overline{\mathbf{\Sigma}} + \lambda \mathbf{U}, \mathbf{\Theta} \rangle - \log \det \mathbf{\Theta},
\end{align*}
with $\mathcal{U}_{(p,q)} := \left\lbrace u v^T : \| u \|_p = 1, \, \| v \|_{q^\star} =1 \right\rbrace $ and $q^\star$ defined such that $\sfrac{1}{q} + \sfrac{1}{q^\star} =1$.\end{theorem} }
\end{subtheorem}
Let us recall that for any matrix $\mathbf{A}$ and $p,q \in \mathbb{Z}_+ \cup \{ \infty \}$, the $(p,q)$-induced norm of $\mathbf{A}$ is defined as 
\begin{align*}
\Vert \mathbf{A} \Vert_{(p,q)} &:= \max_{u : \| u \|_p = 1} \quad \| \mathbf{A} u \|_q. 
\end{align*}
In particular, the operator norm or the largest singular value of $\mathbf{A}$ is equal to its $(2,2)$-induced norm.
\begin{proof}
Theorem \ref{thm:robust.A} follows directly from the definition of the dual norm
\begin{align*}
\| \mathbf{\Theta} \| &= \max_{ \mathbf{U}: \|  \mathbf{U} \|_\star \leqslant 1} \: \langle  \mathbf{U}, \mathbf{\Theta} \rangle.
\end{align*}
Theorem \ref{thm:robust.B} follows from the fact that the dual norm of the $\ell_q$-norm is the $\ell_{q^\star}$-norm, so that: 
\begin{align*}
\Vert \mathbf{A} \Vert_{(p,q)} &= \max_{u : \| u \|_p = 1} \quad \| \mathbf{A} u \|_q \quad = \max_{u : \| u \|_p = 1}  \quad \max_{v : \| v \|_{q^\star} = 1} v^T \mathbf{A} u.
\end{align*}
\end{proof}
{In the result above, the matrix $\mathbf{U}$ should be interpreted as the amount of noise on the covariance matrix $\overline{\mathbf{\Sigma}}$ one wishes to be protected against. Similar equivalence results have been proved in a wide range of other statistical settings \cite{bertsimas2017characterization}. From a Bayesian perspective, regularization can also be derived by imposing some prior distribution on the entries of $\mathbf{\Theta}$ and there is a one-to-one correspondence between the class of prior distributions, the corresponding uncertainty set in the robust perspective and the resulting penalty. }

{In addition to this robustness property, the $\ell_1$-norm is fortunately sparsity-inducing. Killing two birds with one stone, $\ell_1$-regularization has naturally received a lot of attention from the statistical community. Yet, it is fair to admit that the robustness interpretation of the $\ell_1$-norm has been neglected and that many variants of \eqref{eq:L1MLE} use the $\ell_1$-norm solely for sparsity, even though it makes little sense from a robust perspective. For instance, diagonal entries of $\mathbf{\Theta}$ should be nonzero - a consequence of Hadamard's inequality and the constraint $\mathbf{\Theta} \succ 0$. This motivates the fact that diagonal entries are excluded from the cardinality constraint in \eqref{eq:L0MLE}. Similarly, many derivatives of \eqref{eq:L1MLE} exclude diagonal entries from the $\ell_1$-penalty, which, from a robust point of view, is equivalent to considering that diagonal entries of $\overline{\mathbf{\Sigma}}$ are noiseless. To avoid such unrealistic assumptions, robustness and sparsity should, in our opinion, be considered as two distinct properties and be treated as such.}

% -------------------------------------------------------------------------------------------------
% SECTION : mixed-integer OPTIMIZATION FORMULATIONS
% -------------------------------------------------------------------------------------------------
\section{Integer Optimization Perspective}  \label{sec:IO}
{We first formulate Problem \eqref{eq:L0MLE} as binary optimization problem in Section \ref{sec:IO.general}, and prove that it is non-smooth in general. In practice, introducing big-$M$ constants is a simple way to linearize such mixed-integer bilinear problems. Yet, choosing the right big-$M$ values is hard, making these reformulations not always amenable for computation. We show in Section \ref{sec:IO.regularization} that big-$M$ formulations can be viewed as a special case of regularization. With regularization as a unifying perspective, we prove that a certain class of penalty functions leads to smooth convex integer optimization problems and propose a general cutting-plane algorithm to solve them in Section \ref{sec:IO.algo}. We believe our approach provides a novel perspective on the big-$M$ paradigm. In particular, we regard big-$M$ more as a smoothing technique than a simple modeling trick and reveal promising alternatives, such as ridge regularization.}
% -------------------------------------------------------------------------------------------------
% PROBLEM FORMULATION
% -------------------------------------------------------------------------------------------------
\subsection{Problem Formulation}  \label{sec:IO.general}
Let us introduce binary variables $\mathbf{Z}_{ij}$ to encode the support of the inverse covariance matrix $\mathbf{\Theta}$. The set of feasible supports is 
\begin{align*}
\mathcal{S}_p^k = \left\lbrace \mathbf{Z} \in \{0,1\}^{p \times p} : \forall i, {Z}_{ii} = 1 \: \mbox{ and } \: \forall i > j, {Z}_{ij} = {Z}_{ji}  \: \mbox{ and } \: \sum_{i, j>i} {Z}_{ij} \leqslant k \right\rbrace.
\end{align*} The first set of constraints {allows diagonal elements of $\mathbf{\Theta}$ to take nonzero values}. The second set of constraints follows from the fact that $\mathbf{\Theta}$ is symmetric. With these notations, we formulate the cardinality constrained Problem \eqref{eq:L0MLE} as the mixed-integer optimization problem
\begin{equation*}
\min_{\mathbf{Z} \in \mathcal{S}_p^k, \mathbf{\Theta} \succ \mathbf{0}} \quad  \langle \overline{\mathbf{\Sigma}}, \mathbf{\Theta } \rangle - \log \det \mathbf{\Theta} \quad \mbox{  s.t.  } {\Theta}_{ij} = 0 \mbox{  if  } {Z}_{ij} = 0 \: \forall (i,j),
\end{equation*} 
which can be considered as a binary-only optimization problem 
\begin{equation}
\label{eq:BinForm} 
\min_{\mathbf{Z} \in \mathcal{S}_p^k } \quad h(\mathbf{Z}),
\end{equation}
with the objective function 
\begin{equation}
\label{eq:h(z)}
h(\mathbf{Z}) := \min_{\mathbf{\Theta} \succ \mathbf{0}} \quad \langle \overline{\mathbf{\Sigma}}, \mathbf{\Theta } \rangle - \log \det \mathbf{\Theta} \quad \mbox{  s.t.  } {\Theta}_{ij} = 0 \mbox{  if  } {Z}_{ij} = 0 \: \forall (i,j).
\end{equation}
The inner-minimization problem defining $h(\mathbf{Z})$ is a so-called covariance selection problem \cite{dempster1972covariance}, which is a well-studied problem in the literature, and can be efficiently solved. In Section \ref{sec:covsel}, we discuss more details of how the problem can be solved using tailored first-order methods \cite{dahl2008covariance} or coordinate descent schemes \cite{scheinberg2009sinco,krishnamurthy201118}. Note that the problem is always feasible since the identity matrix satisfies all the constraints. Fortunately, as a function of $\mathbf{Z}$, $h(\mathbf{Z})$ is convex (see proof in Appendix \ref{sec:A.dual}). However, $h(\mathbf{Z})$ is piece-wise constant and exhibits strong discontinuities. In the following subsection, we explore techniques to reformulate or approximate $h(\mathbf{Z})$ in a smooth convex way, through the unifying lens of regularization. 

\subsection{Smoothing through regularization} \label{sec:IO.regularization}
In this section, we explore a regularized version of \eqref{eq:h(z)}, 
\begin{equation*}
\label{eq:hreg}
\tilde{h}(\mathbf{Z}) := \min_{\mathbf{\Theta} \succ \mathbf{0}} \quad \langle \overline{\mathbf{\Sigma}}, \mathbf{\Theta } \rangle - \log \det \mathbf{\Theta} + \Omega(\mathbf{\Theta}) \quad \mbox{  s.t.  }  {\Theta}_{ij} = 0 \mbox{  if  } {Z}_{ij} = 0 \: \forall (i,j),
\end{equation*}
where $\Omega$ is regularizer, that is, a convex function of $\mathbf{\Theta}$. In particular, we are interested in two special cases:
\paragraph{Big-$M$ regularization:} A traditional way to express the dependency between $\mathbf{Z}$ and $\mathbf{\Theta}$ in \eqref{eq:h(z)} is to use big-$M$ constraints
\begin{align*}
\tilde{h}(\mathbf{Z}) &:= \min_{\mathbf{\Theta} \succ \mathbf{0}} \quad \langle \overline{\mathbf{\Sigma}}, \mathbf{\Theta } \rangle - \log \det \mathbf{\Theta} \quad \mbox{  s.t.  } |{\Theta}_{ij} |\leqslant M_{ij} {Z}_{ij} \: \forall (i,j).
\end{align*}
$M_{ij} \in \mathbb{R}_+$ are constants chosen sufficiently large such that if $\mathbf{\Theta}^*$ is a minimizer for Problem \eqref{eq:L0MLE}, then $| \Theta_{ij}^* | \leqslant M_{ij} z_{ij}$. In this case, $\min_{\mathbf{Z}} \tilde{h}(\mathbf{Z}) = \min_{\mathbf{Z}}  h(\mathbf{Z})$, i.e., $h$ and $\tilde h$ have the same minimum with%\footnote{Note that we do not nor need to have $\tilde{h}(\mathbf{Z}) = h(\mathbf{Z})$ for all feasible support $\mathbf{Z}$.} . However, a small value of $M_{ij}$ may lead to a different solution than Problem \eqref{eq:h(z)} and thus it is critical to find appropriate values for $\mathbf{M}$. This well-known formulation trick fits in our framework by choosing the regularizer:
\begin{align*}
\Omega(\mathbf{\Theta}) = \begin{cases} 0 & \mbox{ if  } |\Theta_{ij}| \leqslant M_{ij}, \\ +\infty & \mbox{ otherwise}. \end{cases}
\end{align*}
\paragraph{Ridge (or $\ell_2^2$) regularization:} One can choose \begin{align*}
\Omega(\mathbf{\Theta}) = \dfrac{1}{2\gamma} \| \mathbf{\Theta} \|_2^2 =\dfrac{1}{2\gamma} \sum_{i,j} \Theta_{ij}^2,
\end{align*}
for some positive constant $\gamma$. Whatever $\gamma > 0$, $\Omega(\mathbf{\Theta}) > 0$, so $\tilde{h}$ is not a reformulation but an upper-approximation of $h$. Ideally, one would like to minimize $\tilde{h}$ for $1/\gamma \rightarrow 0$. However, as previously seen, regularization induces desirable robustness properties, so having $1/\gamma >0$ may be beneficial from a statistical perspective. 

Under some weak assumptions on $\Omega$, which are satisfied in the special cases of big-$M$ and ridge regularization, one can reformulate $\tilde{h}(\mathbf{Z})$ using strong duality:
\begin{theorem} \label{thm:dualL0regularized}
For any $\mathbf{Z} \in \{0,1\}^{p \times p}$ such that ${Z}_{ii} = 1$ for all $i = 1,\dots,p$, 
\begin{align*}
\tilde{h}(\mathbf{Z}) &:= \min_{\mathbf{\Theta} \succ \mathbf{0}} \quad \langle \overline{\mathbf{\Sigma}}, \mathbf{\Theta } \rangle - \log \det \mathbf{\Theta} + \Omega(\mathbf{\Theta}) \quad \mbox{  s.t.  }  {\Theta}_{ij} = 0 \mbox{  if  } {Z}_{ij} = 0 \: \forall (i,j), \\
 &= \max_{\mathbf{R}: \overline{\mathbf{\Sigma}} + \mathbf{R} \succ \mathbf{0}} \:  p +\log \det (\overline{\mathbf{\Sigma}} + \mathbf{R}) - \langle \mathbf{Z}, \mathbf{\Omega}^\star (\mathbf{R}) \rangle, 
\end{align*}
where $\mathbf{\Omega}^\star$ is some generalization of the Fenchel conjugate for $\Omega$ \cite[see][chap. ~3.3]{boyd2004convex}.
\end{theorem}
An explicit statement of the assumptions and proof of the theorem can be found in Appendix \ref{sec:A.dual}. Theorem \ref{thm:dualL0regularized} calls for a few observations:
\begin{enumerate}
\item $\tilde{h}(\mathbf{Z})$ is a point-wise maximum of linear, hence convex, functions of $\mathbf{Z}$. As a result, $\tilde{h}$ is a convex function.
\item With the dual reformulation, it is easy to see that $\tilde{h}(\mathbf{Z})$ remains bounded.
\item For the big-$M$ regularization, Theorem \ref{thm:dualL0regularized} reduces to 
\begin{align*}
\tilde{h}(\mathbf{Z}) &= \min_{\mathbf{\Theta}\succeq \mathbf{0}} \: \langle \overline{\mathbf{\Sigma}}, \mathbf{\Theta} \rangle - \log \det \mathbf{\Theta} \: \mbox{  s.t.  } |\Theta_{ij}| \leqslant M_{ij} Z_{ij}, \\
&= \max_{\mathbf{R}: \overline{\mathbf{\Sigma}} + \mathbf{R} \succ \mathbf{0}} \:  p +\log \det (\overline{\mathbf{\Sigma}} + \mathbf{R}) - \sum_{i,j} M_{ij} {Z}_{ij} | {R}_{ij} |. 
\end{align*}
\item For the $\ell_2^2$-regularization, Theorem \ref{thm:dualL0regularized} reduces to 
\begin{align*}
\tilde{h}(\mathbf{Z}) &= \min_{\mathbf{\Theta} \succ \mathbf{0}} \quad \langle \overline{\mathbf{\Sigma}}, \mathbf{\Theta } \rangle - \log \det \mathbf{\Theta} + \dfrac{1}{2\gamma} \| \mathbf{\Theta} \|_2^2 \quad \mbox{  s.t.  } {\Theta}_{ij} = 0 \mbox{  if  } {Z}_{ij} = 0, \\
&= \max_{\mathbf{R}: \overline{\mathbf{\Sigma}} + \mathbf{R} \succ \mathbf{0}} \:  p +\log \det (\overline{\mathbf{\Sigma}} + \mathbf{R}) - \dfrac{\gamma}{2} \sum_{i,j} {Z}_{ij} {R}^2_{ij} .
\end{align*}
\item Given a feasible support $\mathbf{Z}$, we denote by $\mathbf{R}^\star (\mathbf{Z})$ the associated dual variable, i.e., $\tilde{h}(\mathbf{Z}) = p +\log \det (\overline{\mathbf{\Sigma}}+ \mathbf{R}^\star (\mathbf{Z})) - \langle \mathbf{Z} , \mathbf{\Omega}^\star (\mathbf{R}^\star (\mathbf{Z})) \rangle$. Then for any feasible $\mathbf{Z}'$, we have
\begin{equation}
\label{eq:subgrad}
\tilde{h}(\mathbf{Z}') \geqslant \tilde{h}(\mathbf{Z}) + \langle \mathbf{Z}' - \mathbf{Z} , \mathbf{\Omega}^\star (\mathbf{R}^\star (\mathbf{Z})) \rangle.
\end{equation}
The inequality above provides a linear lower-approximation of $\tilde{h}$ which coincides with $\tilde{h}$ at $\mathbf{Z}$. In particular, it proves that $-\mathbf{\Omega}^\star (\mathbf{R}^\star (\mathbf{Z}))$ is a subgradient of $\tilde{h}$ at $\mathbf{Z}$. This observation plays a central role in devising a numerical strategy to solve \eqref{eq:BinForm}.
\end{enumerate}

\subsection{Cutting-plane algorithm} \label{sec:IO.algo}
Instead of solving the non-smooth integer optimization Problem \eqref{eq:BinForm}, we consider its regularized proxy
\begin{equation}
\label{eq:BinForm2} 
\min_{\mathbf{Z} \in \mathcal{S}_p^k } \quad \tilde{h}(\mathbf{Z}),
\end{equation}
with
\begin{eqnarray}
\tilde{h}(\mathbf{Z})   &=& \min_{\mathbf{\Theta} \succ \mathbf{0}} \quad \langle \overline{\mathbf{\Sigma}}, \mathbf{\Theta } \rangle - \log \det \mathbf{\Theta} + \Omega(\mathbf{\Theta}) \quad \mbox{  s.t.  }  {\Theta}_{ij} = 0 \mbox{  if  } {Z}_{ij} = 0 \: \forall (i,j), \label{eq:separationProblem}  \\
 & = & \max_{\mathbf{R}: \overline{\mathbf{\Sigma}} + \mathbf{R} \succ \mathbf{0}} \:  p +\log \det (\overline{\mathbf{\Sigma}} + \mathbf{R}) - \langle \mathbf{Z}, \mathbf{\Omega}^\star (\mathbf{R}) \rangle, \nonumber
\end{eqnarray}
as studied in the previous section. Our numerical approach substitutes $\tilde{h}$ in \eqref{eq:BinForm2} by a piece-wise linear lower-approximation and iteratively refines this approximation. This process is equivalent to constraint generation: Applying the inequality \eqref{eq:subgrad} at all feasible supports, $\tilde{h}$ can indeed be seen as a piece-wise linear convex function with an exponential number of pieces:
\begin{align*}
\tilde{h}(\mathbf{Z}') = \max \left\lbrace \tilde{h}(\mathbf{Z}) + \langle \mathbf{Z}' - \mathbf{Z} , \mathbf{\Omega}^\star (\mathbf{R}^\star (\mathbf{Z})) \rangle \: : \: \mathbf{Z} \in \mathcal{S}_p^k \right\rbrace, \quad \forall \mathbf{Z}' \in \mathcal{S}_p^k,
\end{align*}
and the algorithm iteratively includes new pieces. The method is referred to in the literature as outer-approximation \cite{duran1986outer} or generalized Benders decomposition (GBD) and described in pseudo-code in Algorithm \ref{CPAlgo}. 

\begin{algorithm*}
\caption{Cutting-plane algorithm}
\label{CPAlgo}
\begin{algorithmic}
\REQUIRE Initial point $\mathbf{Z}^{(1)} \in \mathcal{S}_p^k$, sample covariance matrix $\overline{\mathbf{\Sigma}}$, sparsity parameter $k$, and tolerance $\epsilon$.
\STATE $t \leftarrow 1 $
\REPEAT
\STATE Compute $\mathbf{Z}_{t+1},\eta_{t+1}$ solution of 
\begin{equation}
\label{eq:master}
\min_{\mathbf{Z} \in S_k^p ,\eta} \: \eta \quad \mbox{  s.t.  } \eta \geqslant \tilde{h}(\mathbf{Z}_i) + \langle \mathbf{Z}-\mathbf{Z}_i , \mathbf{\Omega}^\star(\mathbf{R}^\star(\mathbf{Z}_i)) \rangle,~ \forall  i =1,\dots , t.
\end{equation}
\STATE Compute $\mathbf{R}^\star(\mathbf{Z}_{t+1}), \tilde{h}(\mathbf{Z}_{t+1})$ by solving \eqref{eq:separationProblem}.
\STATE $t \leftarrow t+1 $
\UNTIL{$\eta_t < \tilde{h}(\mathbf{Z}_t) - \varepsilon$}
\RETURN $\mathbf{Z}_t$ 
\end{algorithmic}
\end{algorithm*}

We summarize some important observations, properties, and connections to the literature for the above algorithm.
\begin{enumerate} 
\item Generalized Benders decomposition is a method that can be used to solve convex mixed-integer optimization problems. In this context, Problem \eqref{eq:master} is often referred to as the master problem, and Problem \eqref{eq:separationProblem} is referred to as the (separation) subproblem. The GBD algorithm converges in this context in a finite number of steps because subproblems \eqref{eq:separationProblem} are convex and satisfy Slater's condition, and the set $\mathcal{S}_p^k$ is finite (see Theorem 2.4 in \cite{geoffrion1972generalized}). Thus, the above algorithm converges to an optimal solution for the cardinality constrained Problem \eqref{eq:BinForm2} in a finite number of steps.
\item Note that at each iteration the algorithm supplies a feasible solution $\mathbf{Z}_t$, an upper bound $\tilde{h}(\mathbf{Z}_t)$, and a lower bound $\eta_t$ on the optimal solution. Current heuristic approaches do not offer such a certificate of suboptimality. 
\item Algorithm \ref{CPAlgo} requires to solve a large mixed-integer linear optimization problem each time a new constraint is added. Thus, a branch and bound tree is built at each iteration of the algorithm. Lazy constraint callbacks provide an alternative to building a new branch and bound tree at each iteration of the algorithm. When a constraint is added, instead of resolving the problem, the constraint is added to all active nodes in the current branch-and-bound tree. This enables the same tree to be used for all iterations. This saves the rework of building a new tree every time a mixed-integer feasible solution is found. Lazy constraint callbacks are a relatively new type of callback. CPLEX 12.3 introduced lazy constraint callbacks in 2010 and Gurobi 5.0 introduced lazy constraint callbacks in 2012. To date, the only mixed-integer solvers which provide lazy callback functionality are CPLEX \cite{ilog2012cplex}, Gurobi \cite{gurobi2015gurobi}, and GLPK (see \url{http://gnu.org/software/glpk/}).  
{\item The algorithm can greatly benefit from the choice of a good initial solution $\mathbf{Z}^{(1)}$. In practice, we initialize the algorithm with the support returned by Glasso or Meinshausen and B{\"u}hlmann's  \cite{meinshausen2006high}  local neighborhood selection method.} 
\end{enumerate} 

\subsection{Implementation considerations and cross-validation} \label{sec:IO.cv}
In this section, we describe the grid-search procedure to tune the value of the sparsity level, $k$, and the regularization parameter, $M$ or $\gamma$. 

Two alternatives have been considered in the literature for parameter tuning. The first approach is cross-validation: Before any computation, the data is divided into a training and a validation set, typically with a ratio of $2:1$. Inverse covariance matrices are computed using the training data only and evaluated out-of-sample on the validation data. We pick the parameter values that lead to the best out-of-sample performance in terms of negative log-likelihood. Though simple, cross-validation does not generally have consistency properties for model selection \cite{shao1993linear}. Its``leave-one-out'' or ``multi-fold'' variants are computationally more expensive for they repeat this process on multiple training / validation splits. The second approach consists in using an in-sample information criterion, such as the extended information criterion from \cite{foygel2010extended}
\begin{align*}
BIC_{1/2}(\mathbf{\Theta}) = n \left[ \langle \overline{\mathbf{\Sigma}},\mathbf{\Theta}\rangle - \log \det \mathbf{\Theta} \right] + \| \mathbf{\Theta} \|_0 \log n + 2 \| \mathbf{\Theta} \|_0 \log p,
\end{align*}
which balances goodness of fit and complexity of the model. This criterion is satisfying for it can be computed in-sample and is asymptotically consistent. Consistency results, however, only hold asymptotically and under some assumptions on the data. We will compare those two approaches numerically in Section \ref{sec:exp}.
 
We test different values of $k$ in a grid search manner. Let us remark that the sparsity $k$ only impacts the feasible set of Problem \eqref{eq:BinForm2} and that all linear lower approximations of $\tilde h$ generated from solving a particular instance of Problem \eqref{eq:BinForm2} are valid for any value of $k$. Practically speaking, we solve a series of problems \eqref{eq:BinForm2} for decreasing values of $k$, where each new problem is constructed from the previous one by adding a tighter cardinality constraint. In such a way, each new problem benefits from the cuts generated for previous problems.  

Regarding the regularization parameter, we inspect values which are uniformly log-distributed, starting from $M_0 = p / \| \overline{\mathbf{\Sigma}} \|_1$ for the big-$M$ regularization and $\gamma_0 = 4 p /  \| \overline{\mathbf{\Sigma}} \|_2^2$ for the ridge regularization. Those values follow from bounds on the norm of $\mathbf{\Theta}^\star$, the optimal solution of Problem \eqref{eq:BinForm2}, which we prove in Appendix \ref{sec:A.dual.bounds}. For the big-$M$ formulation, we describe an optimization-based approach to find valid $M$ values from any feasible solution in Appendix \ref{sec:A.bigM}. 
%While such an approach has merit, for it avoids a computationally expensive grid-search,  it
%is still not computationally cheap either and the resulting  bounds were not tight enough to justify this extra burden in our opinion.

% -------------------------------------------------------------------------------------------------
% SECTION: COVARIANCE SELECTION
% -------------------------------------------------------------------------------------------------
\section{Covariance selection problem} \label{sec:covsel}
In this section, we investigate numerical strategies to efficiently solve separation subproblems of the form \eqref{eq:separationProblem}. We provided both primal and dual formulations for the separation Problem \eqref{eq:separationProblem}. In Section \ref{sec:covsel.overview}, we discuss the main advantages of solving the primal vs. the dual formulation. In Section  \ref{sec:covsel.chordal} and \ref{sec:covsel.cd} we describe two families of numerical algorithms. In Section \ref{sec:covsel.exp}, we compare empirically those algorithms. 

\subsection{Comparisons between primal and dual approaches} \label{sec:covsel.overview}
The overall cutting-plane algorithm \ref{CPAlgo} requires at each iteration not only the optimal value $h(\mathbf{Z})$ but also the associated dual variables $\mathbf{R}^\star(\mathbf{Z})$, which are eventually needed to obtain the subgradients $-\mathbf{\Omega}^\star(\mathbf{R}^\star(\mathbf{Z}))$. For that matter, solving the dual formulation in \eqref{eq:separationProblem} appears attractive. 

In the end, the variables of interest are the primal ones, i.e., the sparse precision matrix. Optimal primal and dual variables satify the KKT conditions $\overline{\mathbf{\Sigma}} + \mathbf{R}^\star - (\mathbf{\Theta}^\star)^{-1} = \mathbf{0}$ (see proof of Theorem \ref{thm:dualL0regularized} in Appendix \ref{sec:A.dual.proof}). So, primal variables can be reconstructed from the dual variables at the cost of a $p \times p$ matrix inversion. Due to numerical errors however, inverting $\mathbf{R}^\star(\mathbf{Z})$ might not lead to a sparse matrix. To that extent, it might be favorable to solve the primal formulation in \eqref{eq:separationProblem}, and obtain dual variables by inverting $\mathbf{\Theta}^\star(\mathbf{Z})$. This computation might be  computationally expensive  $(O(p^3))$ , but $\mathbf{\Theta}^\star$ is sparse, it involves at most $p+ 2 k$ nonzero coefficients, a pattern which numerical algorithms could exploit.

All in all, the primal and dual formulations seem equally attractive. Moreover, both objective functions involve the log-determinant. As a result, any gradient-based method will require updating the decision variable, as well as its inverse. Matrix inversion is thus the computational bottleneck for both primal and dual methods. Based on these observations, we identified two streams of relevant numerical strategies:
\begin{enumerate}
\item The first stream of algorithms implements standard first- or second-order methods to solve the primal problem, leveraging the structure of the sparsity pattern defined by $\mathbf{Z}$ to efficiently compute and update the inverse of $\mathbf{\Theta}$ \cite{dahl2008covariance}.
\item The second stream consists in coordinate descent methods for either the primal \cite{scheinberg2009sinco} or the dual formulation \cite{krishnamurthy201118}, where each iteration leads to low-rank update of the matrix and its inverse.
\end{enumerate}

\subsection{Gradient-based methods for the primal formulation}\label{sec:covsel.chordal}
\cite{dahl2008covariance} proposed an efficient gradient-based algorithm for solving the unregularized covariance selection Problem \eqref{eq:h(z)}. The gradient of the objective function is 
\begin{align*}
\overline{\mathbf{\Sigma}} - \mathbf{\Theta}^{-1}.
\end{align*}
However, thanks to the constraints that ${\Theta}_{ij}=0$ if ${Z}_{ij}=0$, only the $p + 2 k$ coordinates ${\Theta}_{ij}$ with $(i,j)$ such that $ Z_{ij} = 1$ are to be updated. In this context, \cite{dahl2008covariance} showed how a particular kind of sparsity patterns - patterns whose clique graph is chordal \cite[see][Section 3 for a definition]{dahl2008covariance} - could enable smart block structure decomposition of both $\mathbf{\Theta}$ and its inverse and fast computations of $\Theta_{ij}$ and $\Theta^{-1}_{ij}$ for the coordinates $(i,j)$ of interest. They also generalize their approach to sparsity patterns which are not chordal, through the use of so-called chordal embeddings. For large and sparse matrices, \cite{dahl2008covariance} report speedups in runtime of two to three orders of magnitude for computing the inverse, and hence the gradient of the objective function. In a similar fashion, their method can accelerate Hessian updates as well. They publicly released \verb|CHOMPACK|, a library which implements sparse matrix computations leveraging chordal sparsity patterns \cite{vandenberghe2015chordal}.

Lastly, \cite{dahl2008covariance} report that a limited-memory Broyden-Fletcher-Goldfarb-Reeves (BFGS) method significantly outperforms other first order methods, such as conjugate gradient, for the covariance selection Problem \eqref{eq:h(z)}. Surprisingly, the authors mention but do not numerically compare with coordinate descent methods, which will be the topic of the next section.

In the case of the regularized covariance selection Problem \eqref{eq:separationProblem}, their approach can easily be adapted:
\begin{itemize}
\item For big-$M$ regularization, one simply needs to project the iterates to ensure the constraints $| \Theta_{ij} | \leqslant M_{ij}$ are satisfied throughout the algorithm. 
\item Ridge regularization adds a $\tfrac{1}{\gamma}\mathbf{\Theta}$ term to the gradient, which raises no additional computational difficulty.
\end{itemize} 

\subsection{Coordinate descent methods}\label{sec:covsel.cd}
Coordinate descent methods are one of the most widely used and highly scalable methods in statistical learning problems. Indeed, as previously mentioned, the most successful methods for $\ell_1$-regularized inverse covariance estimation \eqref{eq:L1MLE} all involve a block coordinate descent strategy for the dual formulation and differ only in the algorithm used to solve the subproblem associated with each block. The caveat in coordinate descent methods often resides in an efficient update step, combined with a good rule for picking the coordinate to update. As noted by many authors in similar contexts \cite{dahl2008covariance,scheinberg2009sinco,krishnamurthy201118}, the update step can be computed in closed-form in our case, which makes coordinate descent methods very attractive.  

For clarity, we illustrate the main ingredients of these methods on the primal formulation with $\ell_2^2$-regularization only, but the same ideas can be applied to the dual formulation and to big-$M$ regularization as well. For a given feasible support $\mathbf{Z}$, we solve
\begin{align*}
\min_{\mathbf{\Theta} \succ \mathbf{0}} \quad \langle \overline{\mathbf{\Sigma}}, \mathbf{\Theta } \rangle - \log \det \mathbf{\Theta} + \dfrac{1}{2\gamma} \| \mathbf{\Theta} \|_2^2 \quad \mbox{  s.t.  } {\Theta}_{ij} = 0 \mbox{  if  } {Z}_{ij} = 0.
\end{align*}

\subsubsection{Coefficient updates}
Given $\mathbf{\Theta} \succ 0$, we first consider the update of the $(i,j)$th coefficient with $i \neq j$, that is, $\Theta_{ij} \leftarrow {\Theta}_{ij} + t$ for some $t \in \mathbb{R}$. In matrix form, this can be written as $\mathbf{\Theta} \leftarrow \mathbf{\Theta} + t (e_i e_j^T + e_j e_i^T)$. Denoting $\mathbf{W}:= \mathbf{\Theta}^{-1}$ the inverse of $\mathbf{\Theta}$, we have 
\begin{align*}
\log \det (\mathbf{\Theta} + t e_i e_j^T + t e_j e_i^T) = \log \det \mathbf{\Theta} + \log \left( 1 + 2 W_{ij} t + ( W_{ij}^2 - W_{ii} W_{jj} ) t^2 \right), 
\end{align*}
so that the best update is obtained by minimizing
\begin{align*}
2\overline \Sigma_{ij} t - \log \left( 1 + 2 W_{ij} t + ( W_{ij}^2 - W_{ii} W_{jj} ) t^2 \right) + \tfrac{1}{\gamma} (\Theta_{ij} + t)^2.
\end{align*}
Setting the derivative to zero, we find the best update $t^\star$ as the unique solution of the equation
\begin{align*}
2\overline \Sigma_{ij} - \dfrac{ 2 W_{ij} + 2 ( W_{ij}^2 - W_{ii} W_{jj} ) t }{1 + 2 W_{ij} t + ( W_{ij}^2 - W_{ii} W_{jj} ) t^2} + \tfrac{2}{\gamma} (\Theta_{ij} + t) = 0,
\end{align*}
which satisfies $1 + 2 W_{ij} t + ( W_{ij}^2 - W_{ii} W_{jj} ) t^2  > 0$. The above equation can be reduced into a cubic equation in $t$.

Regarding diagonal coefficients, the best update for the $(i,i)$th coefficient, $\Theta_{ii} \leftarrow {\Theta}_{ii} + 2 t$, can similarly be found by minimizing 
\begin{align*}
2 \overline \Sigma_{ii} t - \log \left( 1 + 2 W_{ii} t \right) + \tfrac{1}{2 \gamma} (\Theta_{ii} + 2 t)^2,
\end{align*}
over $t$ such that $1 + 2 W_{ii} t > 0$, which boils down to solving a quadratic equation.

In both cases, the value $t^\star$ for the best update $\mathbf{\Theta} \leftarrow \mathbf{\Theta} + t^\star  (e_i e_j^T + e_j e_i^T )$ can fortunately be computed in closed-form, i.e., constant time. After updating $\mathbf{\Theta}$, $\mathbf{W}$ can be update in $O(p^2)$ steps only, using Woodbury-Sherman–Morrison formula. 

{Observe that using these one-coordinate updates, the matrix  $\mathbf{\Theta}$ remains positive definite throughout the algorithm. Indeed, using Shur complements \cite{zhang2006schur}, $\mathbf{\Theta} + t^\star  (e_i e_j^T + e_j e_i^T ) \succ 0$ if $\mathbf{\Theta} \succ 0$ and $1 + 2 W_{ij} t^\star + ( W_{ij}^2 - W_{ii} W_{jj} ) > 0$. If the algorithm is properly initialized by a positive definite matrix, positive definiteness of the subsequent iterates then follows by induction.}

%\paragraph{Remark: } For $\ell_2^2$-regularization, the primal and dual formulations have a very similar structure: the objective is the sum of a log-determinant with a quadratic term. Consequently, similar update rules can be derived for the dual. For big-$M$ regularization,  
\subsubsection{Update rule and computational complexity:} 
In the case of Glasso, \cite{scheinberg2009sinco} successfully suggested a greedy rule: at each iteration, the algorithm scans through all the coefficients of $\mathbf{\Theta}$ and compute the objective decrease resulting from their update. Then, only the coefficient leading to the largest improvement is updated, as described in Algorithm \ref{CDPrimal}. All together, one iteration of the algorithm updates one coefficient and requires $O(p^2)$ operations, with the update of $\mathbf{W}$ as the computational bottleneck. Note that this strategy is particularly efficient on the primal formulation, since there are only  $p + 2k$ potentially nonzero coefficients, compared with $p \times (p+1) / 2$ in the dual. 
\begin{algorithm*}
\caption{Greedy coordinate descent algorithm}
\label{CDPrimal}
\begin{algorithmic}
\REQUIRE Support $\mathbf{Z} \in \mathcal{S}_p^k$, sample covariance matrix $\overline{\mathbf{\Sigma}}$, regularization parameter $\gamma$.
\REPEAT
\STATE For all $(i,j)$ such that $Z_{ij}=1$, compute the objective decrease resulting from the update of the $(i,j)$th coefficient.
\STATE Update $\mathbf{\Theta} \leftarrow \mathbf{\Theta} + t^\star e_i e_j^T + t^\star e_j e_i^T$  for $(i,j)$ which leads to the biggest improvement.
\STATE Update $\mathbf{W}$ accordingly
\UNTIL Stopping criterion
\RETURN $\mathbf{\Theta}$ 
\end{algorithmic}
\end{algorithm*}

Since updating the inverse of $\mathbf{\Theta}$ remains the challenging part, \cite{krishnamurthy201118} suggested a block coordinate approach for solving the dual formulation of the Lasso estimator \eqref{eq:L1MLE}. We can adapt their approach to our regularized covariance selection problem, both in primal and dual formulation.
From a high level perspective, at each iteration, a whole row is updated instead of a single coefficient. The computational cost remains $O(p^2)$ steps per iteration, but one might expect fewer iterations in total. We refer to \cite{krishnamurthy201118} for a detailed presentation of the updates and the overall algorithm.  

We terminate the algorithm as soon as the duality gap or the objective decrease is sufficiently small.

\subsection{Empirical performance and comparisons}\label{sec:covsel.exp}
In this section, we compare the computational time required to solve the covariance selection problem by each method and {see} how they scale with the problem size $p$ and the sparsity $k$. We also investigated how the conditioning of the problem, through the number of samples $n$ used to compute the empirical covariance matrix $\overline{\mathbf{\Sigma}}$ and the regularization parameter $M$ or $\gamma$, impacted computational time. However, we observed little effect and decided not to report those experiments. 

\subsubsection{Instance generation}
As in \cite{yuan2007model,friedman2008sparse}, we consider a full precision matrix $\mathbf{\Theta}_0$ with $\Theta_{ii}=2$ and $\Theta_{ij}=1$ for $i \neq j$, in short $\mathbf{\Theta}_0 = \textbf{I}_p + e e^T$. We then generate $n$ random samples from the normal distribution $\mathcal{N}(0, \mathbf{\Theta}_0^{-1})$ and compute the empirical covariance matrix $\overline{\mathbf{\Sigma}}$. We randomly sample a feasible support $\mathbf{Z}$ from $\mathcal{S}_p^k$ and solve Problem \eqref{eq:separationProblem}.

The degrees of freedom in our simulations are the dimension $p$ and the sparsity level $t$. Based on those quantities, $k$ and $n$ are fixed to
\begin{align*}
k & = \left \lfloor t \: \dfrac{p (p-1)}{2} \right \rfloor, \\
n &= p.
\end{align*} 

\subsubsection{Methods implementation}
For both the big-$M$ and the $\ell_2^2$ regularization problem, we implement and compare five methods:
\begin{itemize}
\item a BFGS method on the primal formulation (\verb|BFGS_primal|), using the library \verb|CHOMPACK| for sparse matrix computations \cite{vandenberghe2015chordal},
\item four (block) coordinate descent strategies, denoted \verb|CD_primal|, \verb|CD_dual|, \\
\verb|BCD_primal|, and \verb|CD_dual|.
\end{itemize}
All code is written in \verb|Julia 0.6.0| \cite{lubin2015computing}, with the exception of the BFGS algorithm, which is implemented in \verb|Python 3.5.3| and integrated into the main \verb|Julia| script using the \verb|PyCall| package. We terminate the algorithms when the duality gap falls below $10^{-4}$ or the objective improvement after one iteration is less than $10^{-12}$.
\subsubsection{Empirical results}
Figures \ref{fig:covsel.bigm} and  \ref{fig:covsel.ridge} report computational time as a $p$ and $t$ increase for the big-$M$ and ridge regularization respectively. From these experiments, we can make the following observations:
\begin{enumerate}
\item  For (block) coordinate descent methods, solving the primal formulation is more effective than solving the dual problem.  
\item Coordinate descent methods compete with block coordinate descent schemes when the sparsity level $t$ is very low (less than $1\%$) but do not scale as well as $t$ increases. 
\item As a result, \verb|BCD_primal| is often the best method for solving Problem \eqref{eq:separationProblem}.
\item The \verb|BFGS_primal| algorithm generally takes $50-100$ times longer than \verb|BCD_primal|. For $p > 1000$, the algorithm did not terminate after a $12$-hour time limit.
\end{enumerate}

\begin{figure}
\centering
\begin{subfigure}[t]{.45\linewidth}
	\centering
	\includegraphics[width=\linewidth]{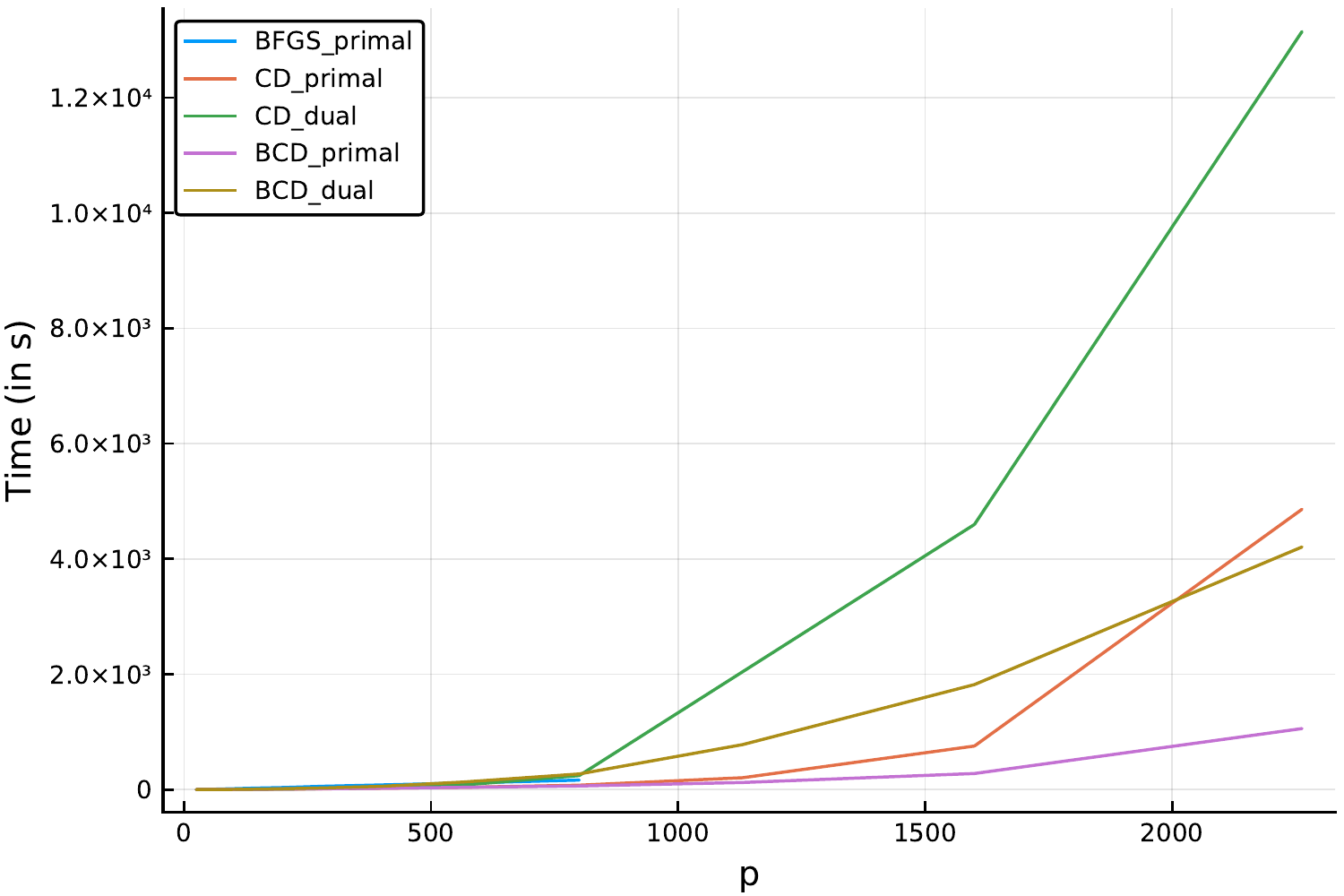}
	\caption{$p$, with $t=1\%$.}
\end{subfigure} %
~
\begin{subfigure}[t]{.45\linewidth}
	\centering
	\includegraphics[width=\linewidth]{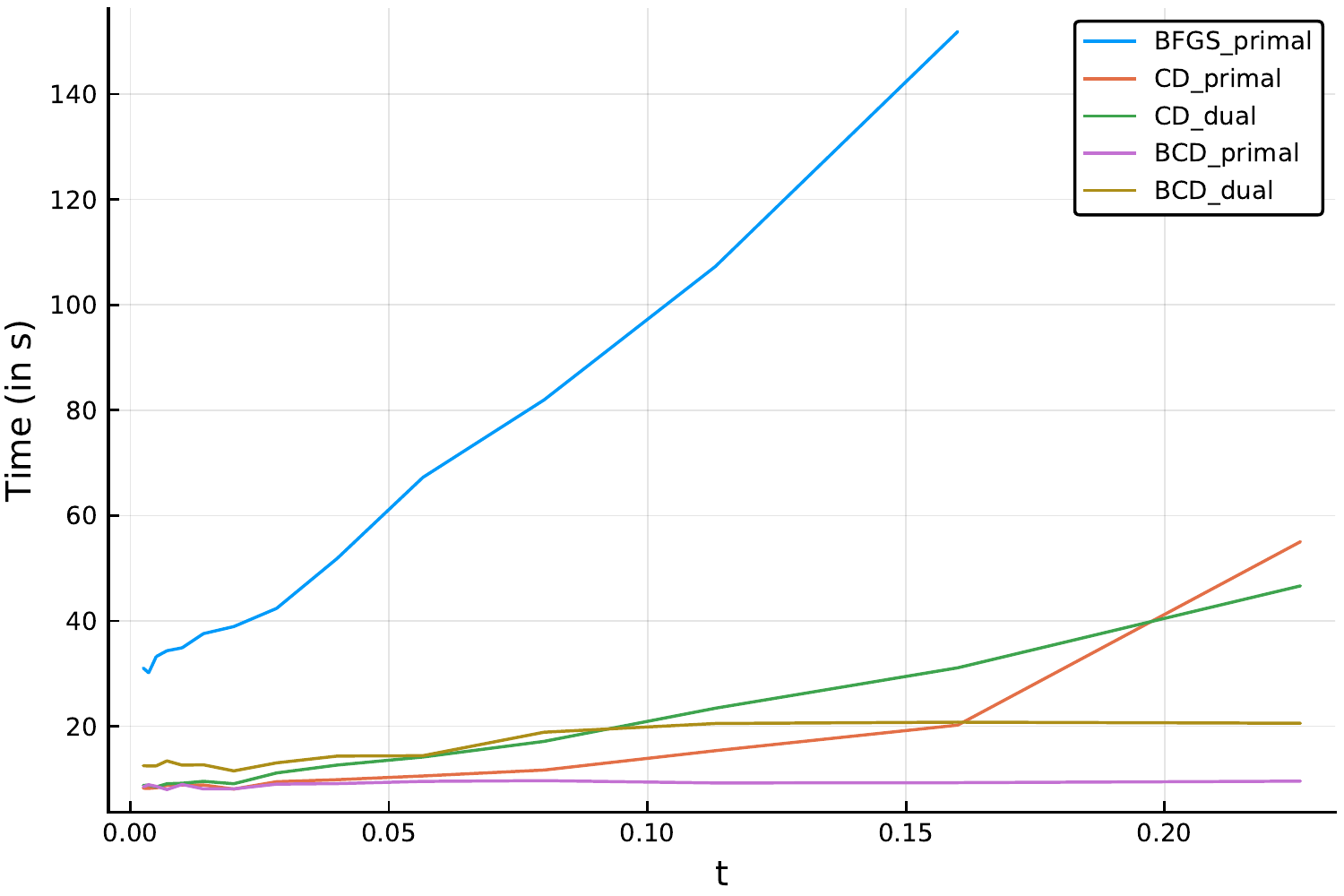}
	\caption{$t$, with $p=200$.}
\end{subfigure}
\caption{Impact of dimension size $p$ and sparsity level $t$ on computational time, for the big-$M$ regularization with $M =M_0 = p / \|\overline{\mathbf{\Sigma}}\|_1$. }
\label{fig:covsel.bigm}
\end{figure}
\begin{figure}
\centering
\begin{subfigure}[t]{.45\linewidth}
	\centering
	\includegraphics[width=\linewidth]{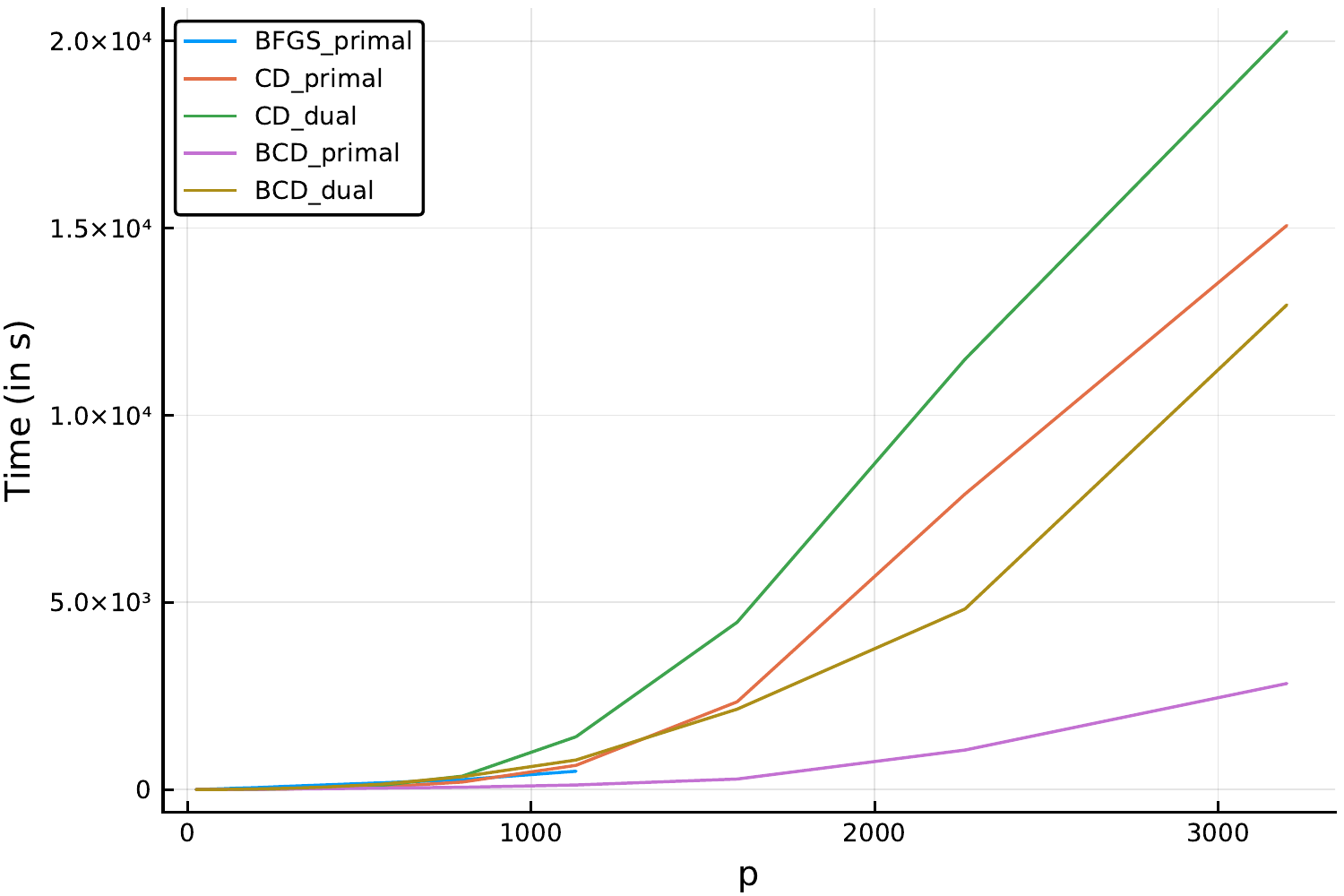}
	\caption{$p$, with $t=1\%$.}
\end{subfigure} %
~
\begin{subfigure}[t]{.45\linewidth}
	\centering
	\includegraphics[width=\linewidth]{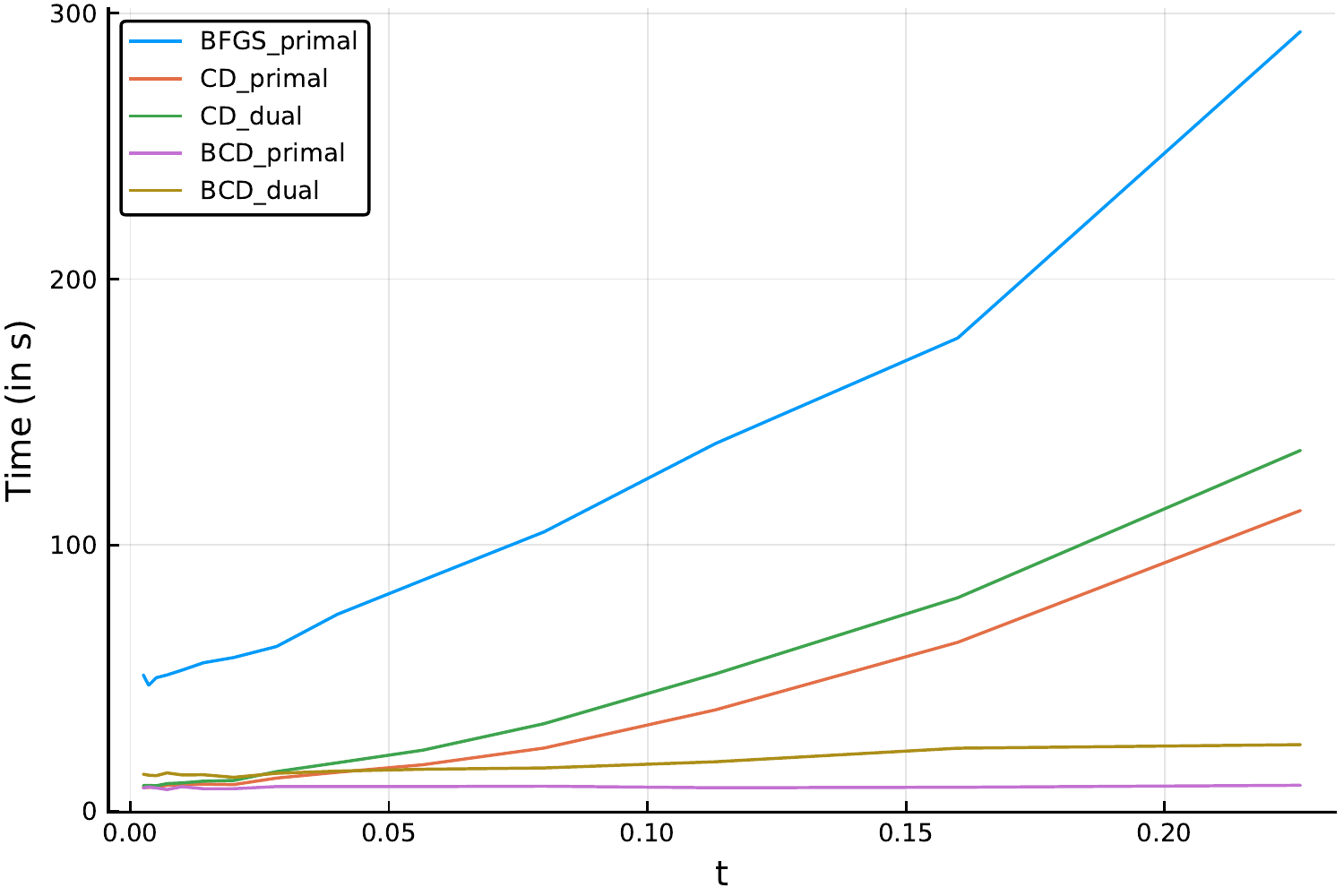}
	\caption{$t$, with $p=200$.}
\end{subfigure}
\caption{Impact of dimension size $p$ and sparsity level $t$ on computational time, for the ridge regularization with $\gamma = \gamma_0 = 4 p / \| \overline{\mathbf{\Sigma}} \|_2^2$. }
\label{fig:covsel.ridge}
\end{figure}

% -------------------------------------------------------------------------------------------------
% SECTION 5: COMPUTATIONAL RESULTS
% -------------------------------------------------------------------------------------------------
\section{Computational Results} \label{sec:exp}
In this section, we present numerical results on both synthetic (Section \ref{sec:exp.synth}) and real data (Section \ref{sec:exp.cancer}). 

\subsection{Synthetic experiments} \label{sec:exp.synth}
We follow the methodology described in \cite{banerjee2008model}. We sample precision matrices of the form $\mathbf{\Theta}_0 = \delta \textbf{I}_p + 0.5 \mathbf{Z}_0$, where $\mathbf{Z}_0 \in \mathcal{S}^p_{k_{true}}$ and $\delta$ is chosen so that the condition number is equal to $p$. We then randomly sample $n$ vectors from a multivariate normal distribution $\mathcal{N}(0, \mathbf{\Theta}_0^{-1})$, compute the empirical covariance matrix $\overline{\mathbf{\Sigma}}$ and standardize it. To evaluate the output of the algorithms out-of-sample, we generate similarly $n /2$ (resp. $5 n$) data points for the validation (resp. test) set. 

In this setting, we can assess the feature selection ability of a method in terms of accuracy $A$, i.e., the fraction of the $k_{true}$ nonzero upper-diagonal coefficients of $\mathbf{\Theta}_0$ correctly recovered, and false detection rate $FDR$, defined as the proportion of coefficients in the support of the solution which are not in the support of $\mathbf{\Theta}_0$. We also compute the negative log-likelihood ($-LL$) of the returned precision matrix on the test set.

All discrete optimization problems are terminated once the tolerance gap falls below $10^{-4}$, where the tolerance gap is the percentage difference between the final lower and upper bounds, or after a $5$-minute time limit.

\subsubsection{Impact of regularization and sparsity $k$} \label{sec:exp.synth.regk}
First, we consider one problem instance with $p=200$, $n / p =1$, and sparsity level $t_{true}=1\%$.  The discrete formulation \eqref{eq:BinForm2} involves two hyper-parameters, the sparsity $k$ and the regularization parameter $M$ or $\gamma$, which needs to be tuned using grid-search as described in Section \ref{sec:IO.cv}. 

The value of the regularization parameter has a crucial impact on the overall computational time of the cutting-plane algorithm. {Figure \ref{fig:regularization.bigm} shows a steep increase in computational time (top) and in the number of cuts (middle) as the regularization parameter, for both big-$M$ and ridge regularization, increases.}  Unfortunately, for applications of interest in our experiments, we needed to use high values of $M$ and $\gamma$ and had to stop the algorithm after a $5$-minute time limit. Yet, this early stopping strategy did not harm the overall performance of our approach. Indeed, the algorithm is able to find optimal or near-optimal solutions in a short amount of time but spends most of the time proving optimality. {For moderate values of $M/\gamma$, the optimality gap (Figure \ref{fig:regularization.bigm}(c)) after five minute is indeed relatively small, and the algorithm spents a lot of time closing that gap. For large regularization parameter value, on the other hand, the gap increases significantly (over $100\%$) and becomes uninformative. This corresponds to the regime of most of our subsequent experiments for which we will not report optimality gaps.} We provide extensive computational time experiments on smaller-size problems {as $n$, $p$ and $k$ vary} in Appendix \ref{sec:A.comptime}. 

\begin{figure}
\centering
\begin{subfigure}[t]{\linewidth}
	\centering
	\includegraphics[width=.45\linewidth]{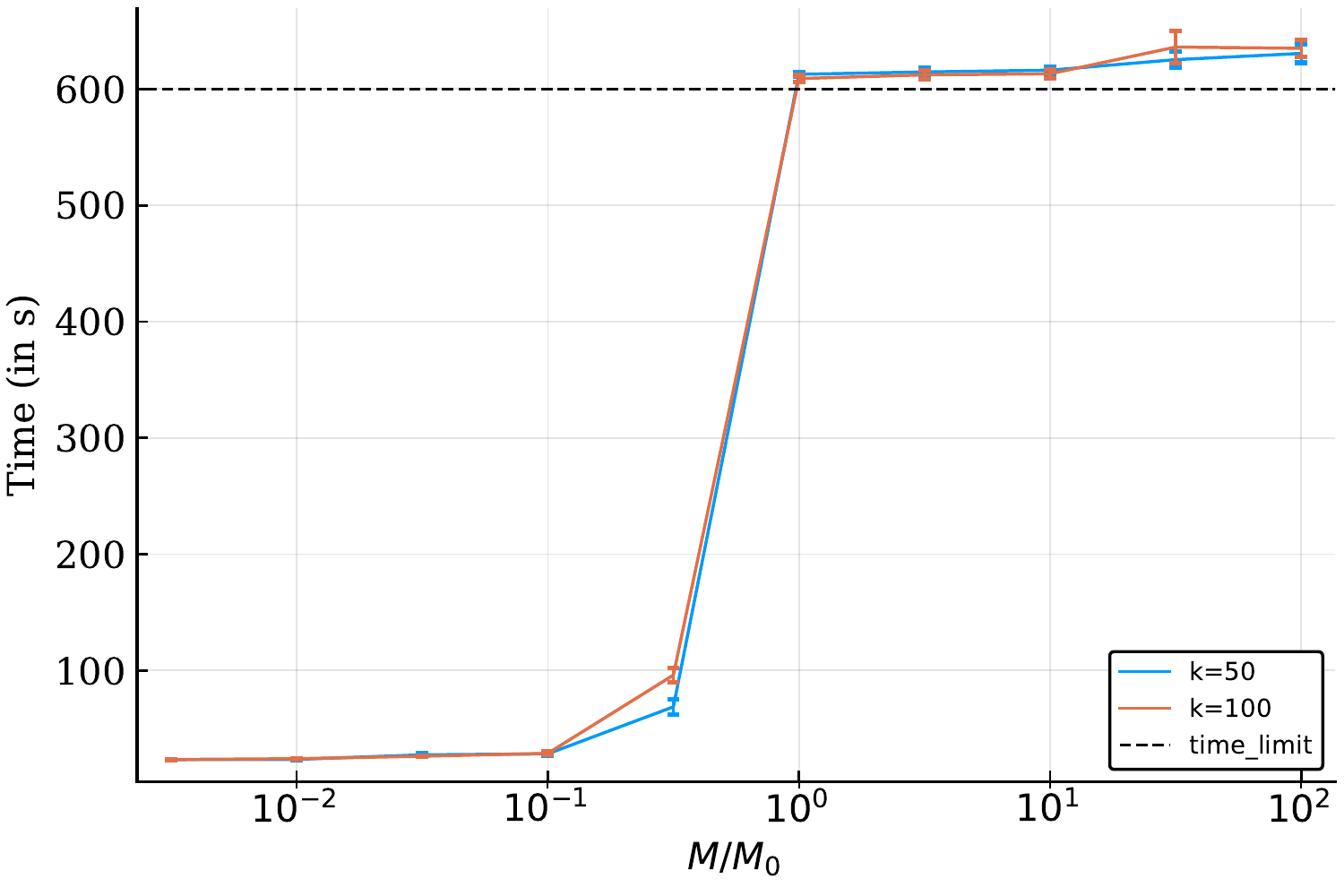}
	\includegraphics[width=.45\linewidth]{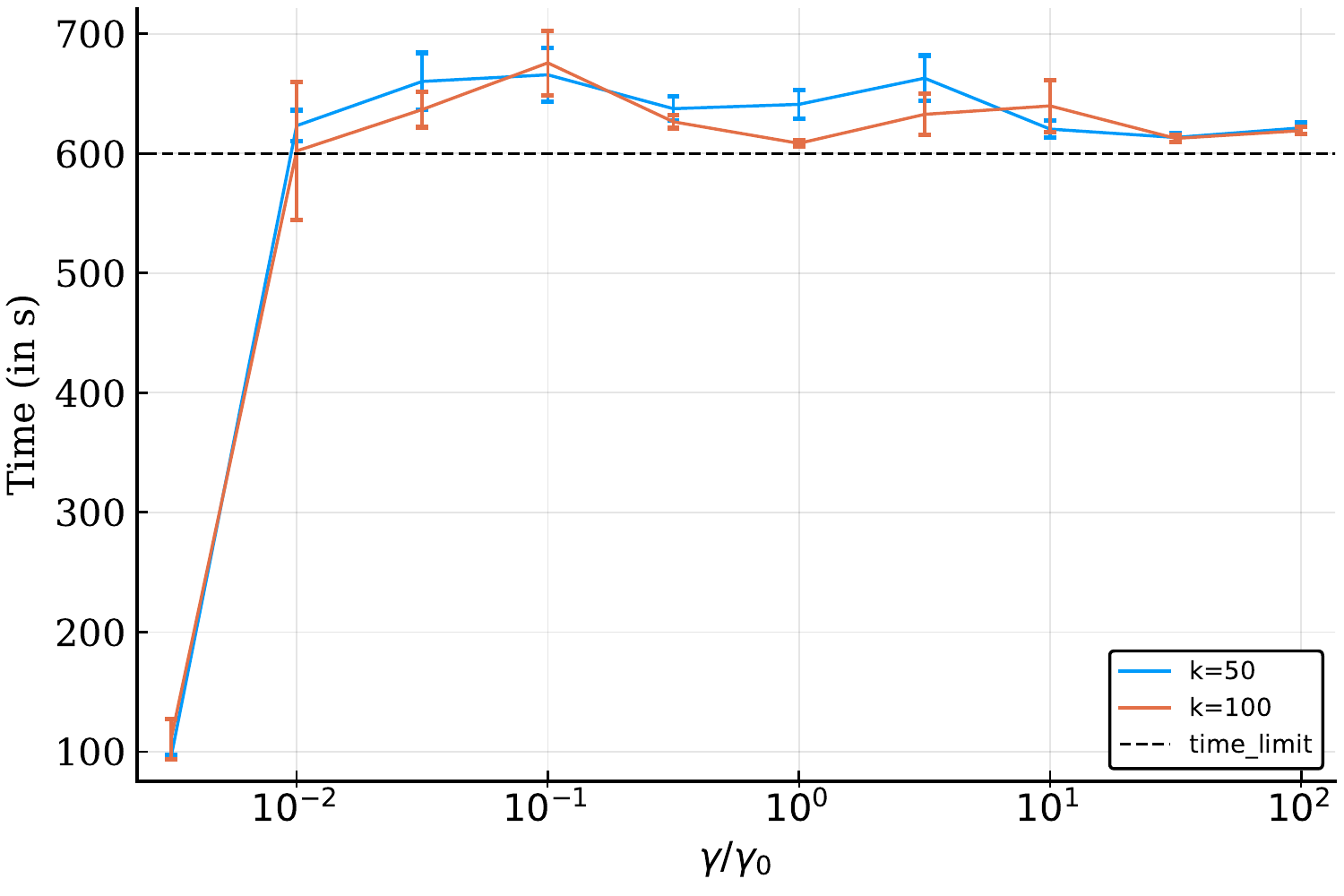}
	\caption{Computational time (in seconds).}
\end{subfigure} %
~
\begin{subfigure}[t]{\linewidth}
	\centering
	\includegraphics[width=.45\linewidth]{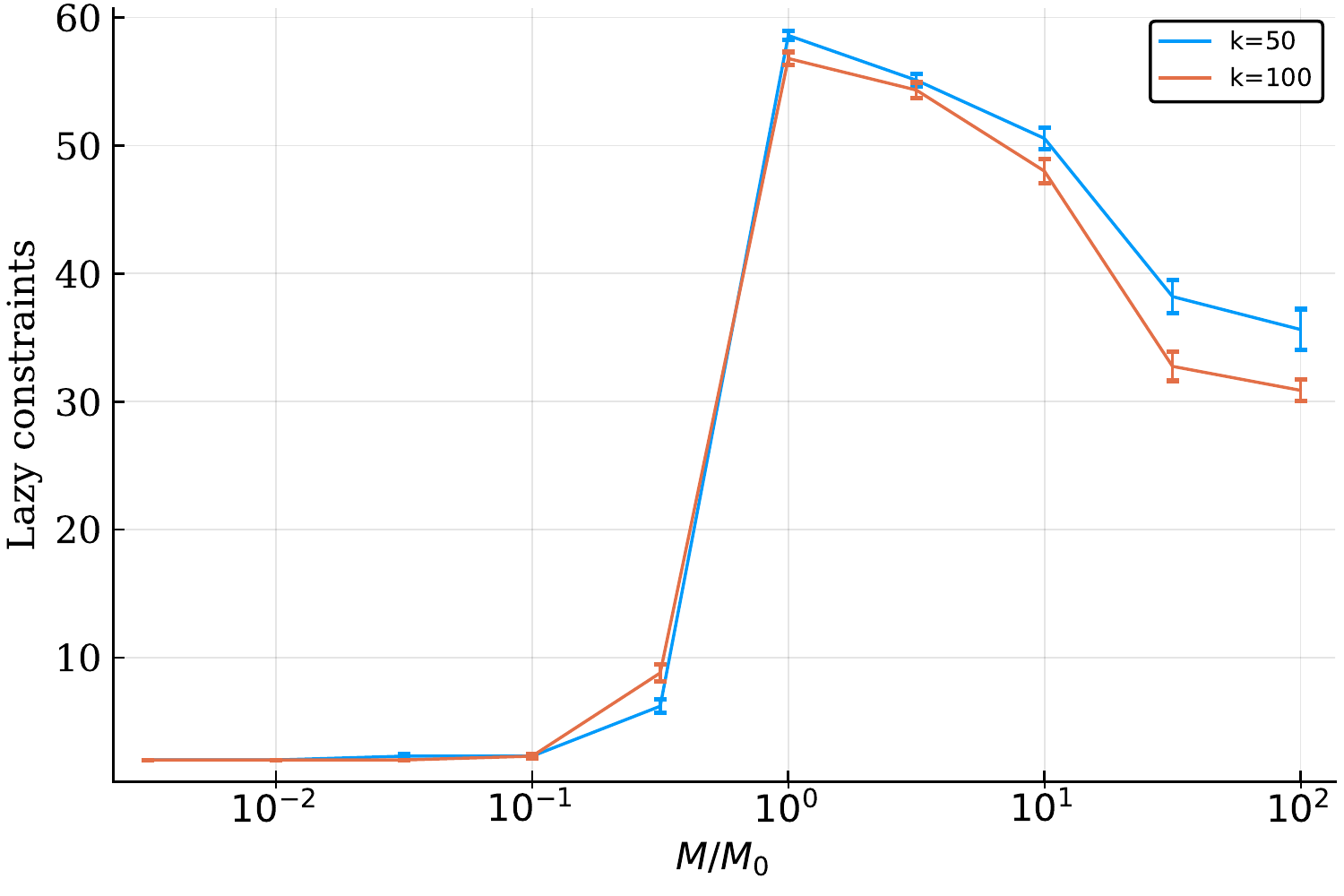}
	\includegraphics[width=.45\linewidth]{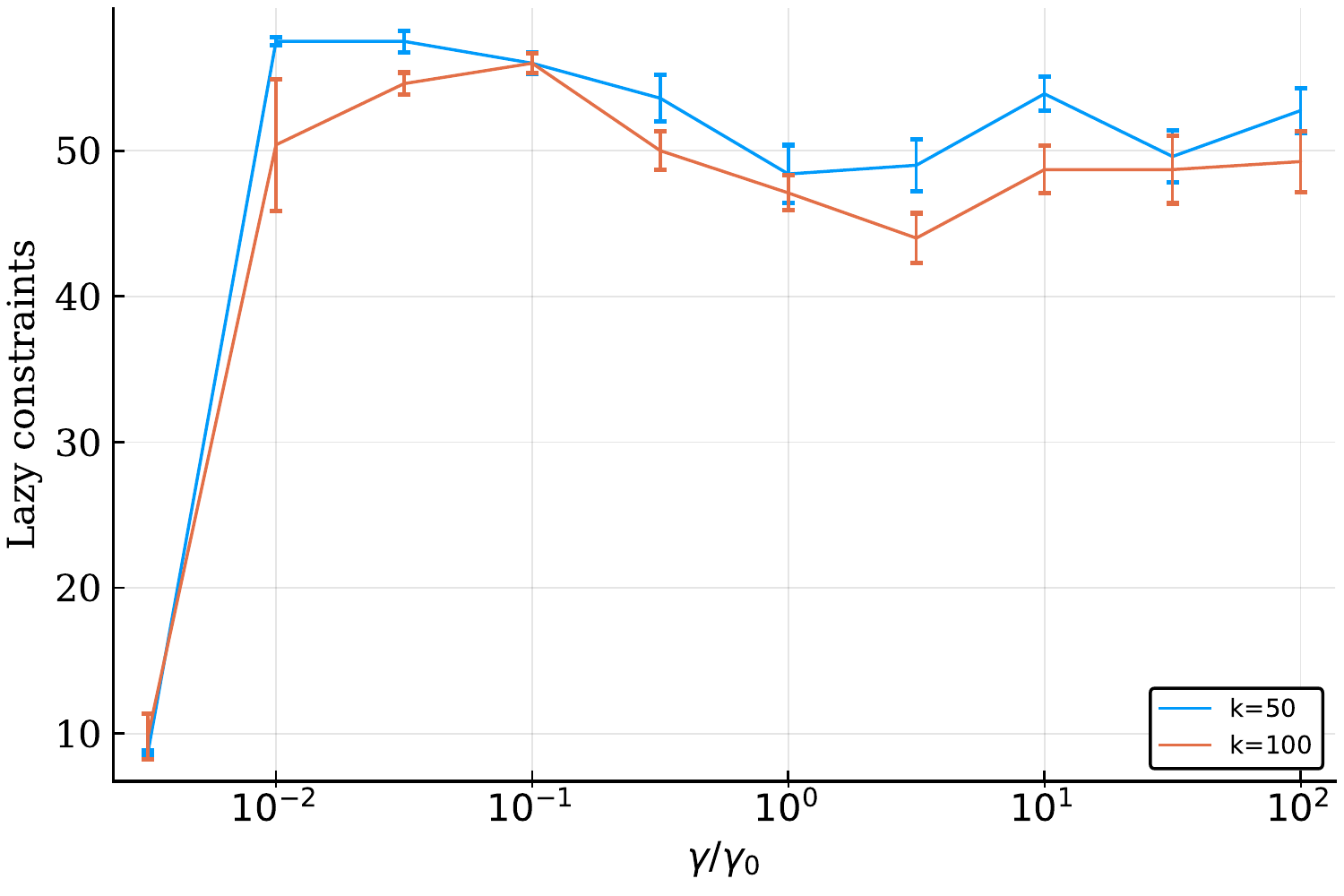}
	\caption{Number of cuts.}
\end{subfigure} %
~
\begin{subfigure}[t]{\linewidth}
	\centering
	\includegraphics[width=.45\linewidth]{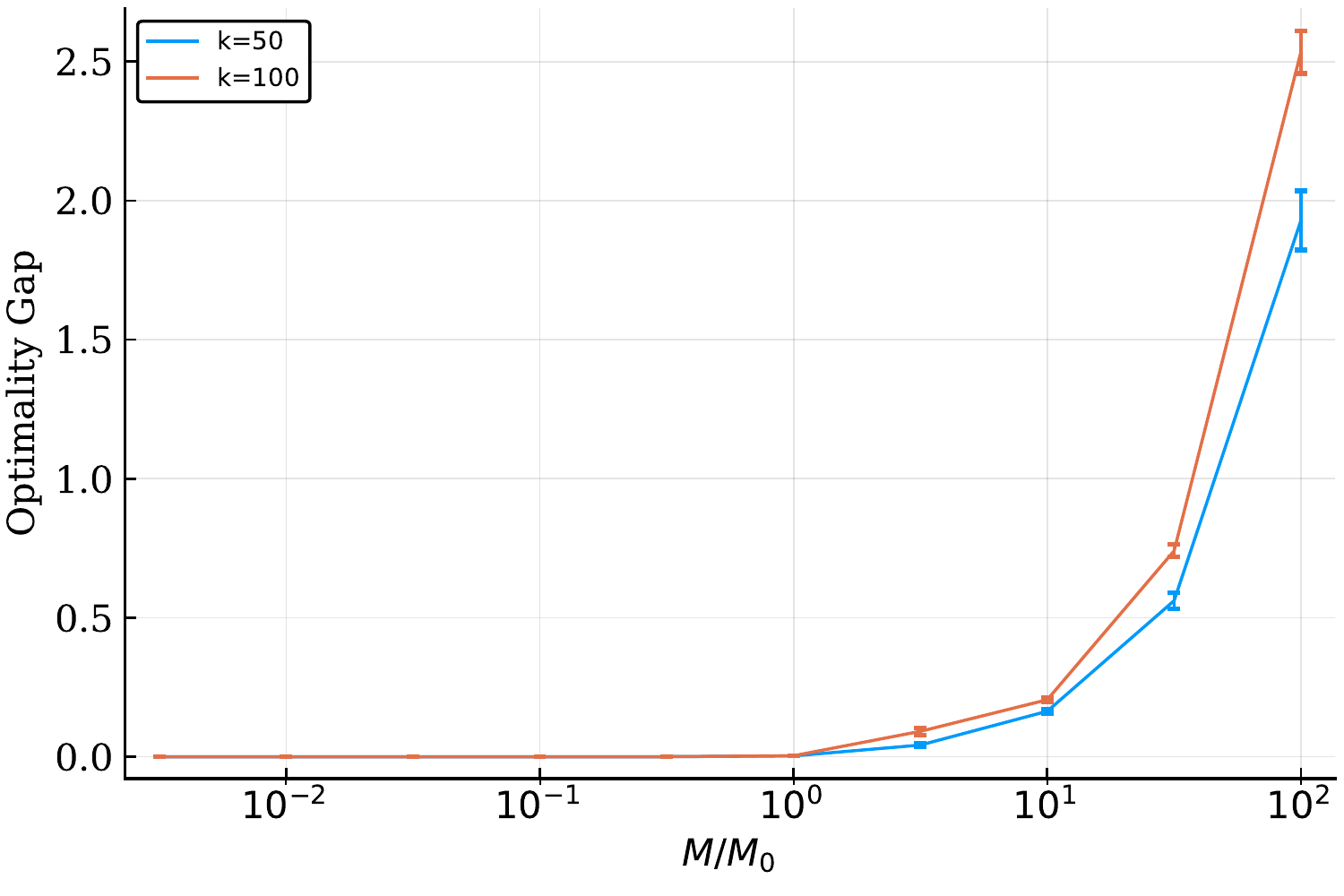}
	\includegraphics[width=.45\linewidth]{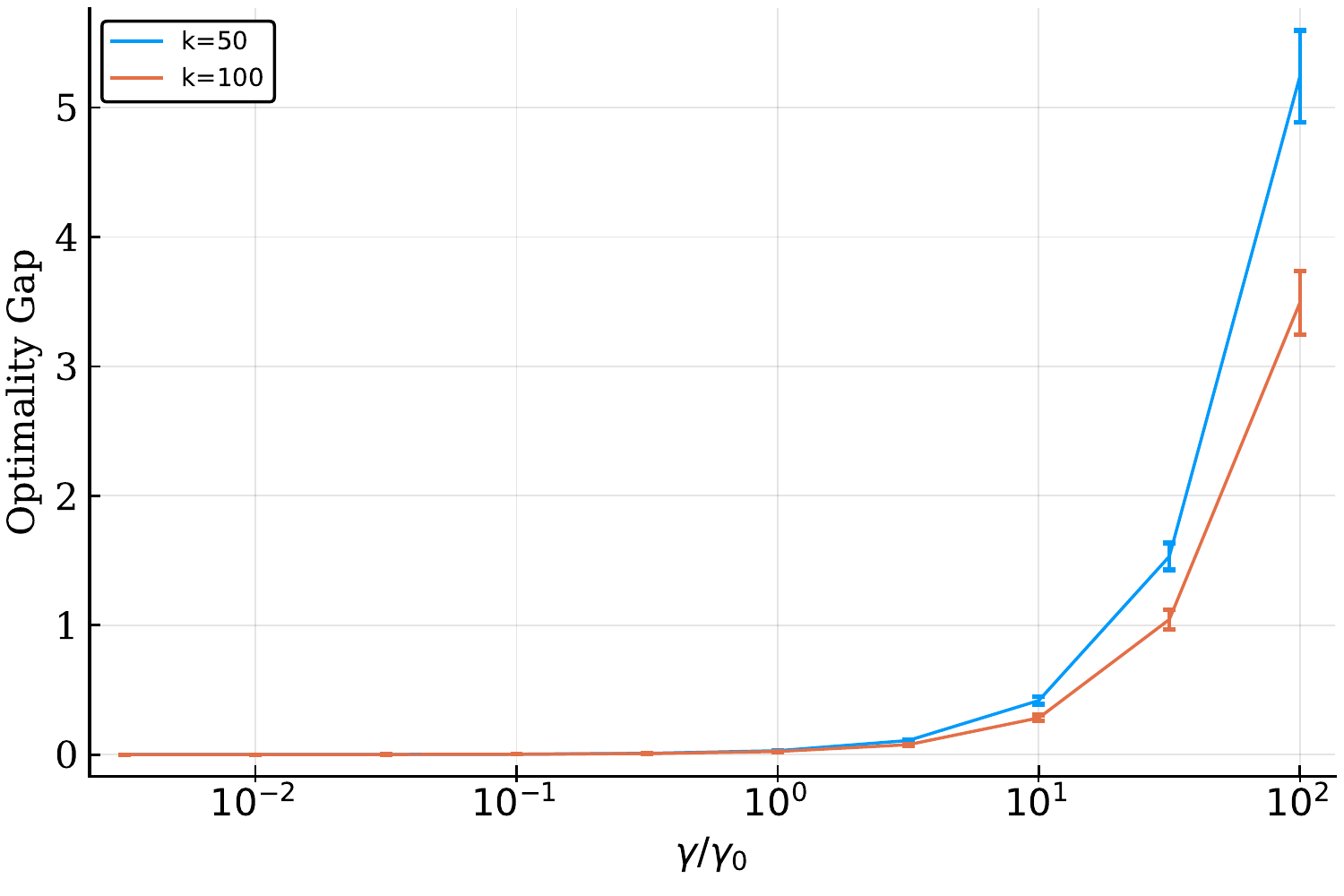}
	\caption{Relative optimality gap.}
\end{subfigure}
\caption{Impact of the regularization parameter $M / M_0$ for big-M (left), $\gamma / \gamma_0$ for ridge (right) on computational time (top), number of cuts (middle) and relative optimality gap (bottom). For the big-$M$ regularization, $M_0 = p / \|\overline{\mathbf{\Sigma}}\|_1$. For ridge regularization, $\gamma_0 = 4p / \| \overline{\mathbf{\Sigma}} \|_2^2$.}
\label{fig:regularization.bigm}
\end{figure}

At the end of the grid search, we select the best pair of parameters and compare the quality of the solution in terms of sparsity, accuracy, false detection and out-of-sample log-likelihood with solutions returned by Glasso \cite{friedman2008sparse} and Meinshausen and B{\"u}hlmann's approximation scheme \cite{meinshausen2006high}, implemented in the \verb|R| package \verb|glasso|\footnote{available at \url{https://cran.r-project.org/web/packages/glasso/}}. We tuned the hyper-parameter $\rho$ in those formulations through a grid search, testing values which led to similar sparsity level $k$ as the discrete formulations. Table \ref{tab:E1.cv} (resp. Table \ref{tab:E1.bic}) reports the results when the hyper-parameters are tuned using the negative log-likelihood on a test set (resp. the information criterion from \cite{foygel2010extended}). 

In both cases, we observe that discrete formulations outperform the other two methods in terms of resulting sparsity (by at least $40\%$), false detection rate (by a factor $4$-$12$) and out-of-sample likelihood (by $11$-$18\%$). On the other hand, Meinshausen and B{\"u}hlmann's approximation (MB in short) is always the fastest and most accurate method. Actually, we use its solution as a warm-start to our discrete optimization method. Let us remark that the big-$M$ and the ridge formulation perform almost identically and that their performance is barely not impacted by the choice of the criterion. On the contrary, the model selected with Glasso and MB highly depends on the cross-validation criterion: with negative log-likelihood, both methods tend to select the less sparse model, whereas much sparser models are selected with $BIC_{1/2}$.

\begin{table}[h!] 
\centering
\begin{tabular}{l|c|c|c|c}
\toprule
Method & big-$M$ & Ridge & MB & Glasso \\
\midrule
$k^\star$ & $199$ ($0$) & $199$ ($0$) & $796$ ($0$) & $796$ ($0$) \\
$A$ & $0.9508$ ($0.0080$) & $0.9508$ ($0.0080$) & $0.9960$ ($0.0020$) & $0.9945$ ($0.0023$) \\
$FDR$ & $0.0492$ ($0.0080$) & $0.0492$ ($0.0080$) & $0.6791$ ($0.0030$) & $0.7514$ ($0.0006$) \\
$-LL$ & $141.39$ ($3.05$) & $141.37$ ($3.05$) & $157.11$ ($2.47$) & $162.05$ ($1.89$) \\
Time (in s) & $352.87$ ($11.12$) & $203.36$ ($39.00$)  & $1.10$ ($0.04$) & $3.97$ ($0.31$) \\
\bottomrule
\end{tabular}
\caption{Average performance on synthetic data with $p=200$, $n/p = 1$, $t=1\%$ (leading to $k_{true}=199$), where the hyper-parameters of each formulation is chosen using the best negative log-likelihood over a validation set. We report the average performance over $10$ instances {(and their standard deviation)}.} 
\label{tab:E1.cv}
\end{table}

\begin{table}[h!] 
\centering
\begin{tabular}{l|c|c|c|c}
\toprule
Method & big-$M$ & Ridge & MB & Glasso \\
\midrule
$k^\star$ & $194$ ($5$) & $194$ ($5$) & $276$ ($8$) & $542$ ($26$) \\
$A$ & $0.9317$ ($0.0081$) & $0.9317$ ($0.0081$) & $0.9890$ ($0.0037$) & $0.9814$ ($0.0047$) \\
$FDR$ & $0.0444$ ($0.0062$) & $0.0444$ ($0.0062$) & $0.2634$ ($0.0213$) & $0.6329$ ($0.0167$) \\
$-LL_{test}$ & $141.78$ ($3.24$) & $141.78$ ($3.24$) & $167.16$ ($2.48$) & $170.22$ ($2.42$) \\
Time (in s) & $349.5$ ($14.5$) & $225.2$ ($43.00$)  & $0.90$ ($0.05$) & $2.77$ ($0.19$) \\
\bottomrule
\end{tabular}
\caption{Average performance on synthetic data with $p=200$, $n/p = 1$, $t=1\%$ (leading to $k_{true}=199$), where the hyper-parameters of each formulation are chosen using the best in-sample extended Bayesian information criterion $BIC_{1/2}$. We report the average performance over $10$ instances {(and their standard deviation)}.}
\label{tab:E1.bic}
\end{table}

\subsubsection{Impact of problem size}
We now pursue the same comparison for problems with varying characteristics $n/p$, $t$ and $p$. 
\paragraph{Number of samples $n$}
Information-theoretic intuition suggests that the problem becomes easier as $n$ increases. For $n<p$, the empirical covariance matrix is always singular so its inverse cannot be properly defined without sparsity assumptions. On the other side of the spectrum, theoretical guarantees exists for many algorithms \cite{meinshausen2006high,santhanam2012information} in the limit $n\rightarrow\infty$. As shown on Figure \ref{fig:scale.n.ll}, this intuition is confirmed experimentally with accuracy (resp. false detection rate) increasing (resp. decreasing) as $n/p$ increases. In addition, we observe that the conclusions drawn from the previous section hold consistently for various values of $n$: the discrete optimization formulations lead to reduced false detection rate, while being of comparable accuracy with the most accurate benchmark. They also demonstrate better out-of-sample negative log-likelihood (Figure \ref{fig:scale.n.loglik} in Appendix \ref{sec:A.statperf}) and their performance is robust to the cross-validation criterion used (Figure \ref{fig:scale.n.ebic} in Appendix \ref{sec:A.statperf}). Note that the other two methods, MB and Glasso, do not exhibit a decreasing false detection rate when cross-validated using the $BIC_{1/2}$ criterion. 

 \begin{figure}
\centering
\begin{subfigure}[t]{.45\linewidth}
	\centering
	\includegraphics[width=\linewidth]{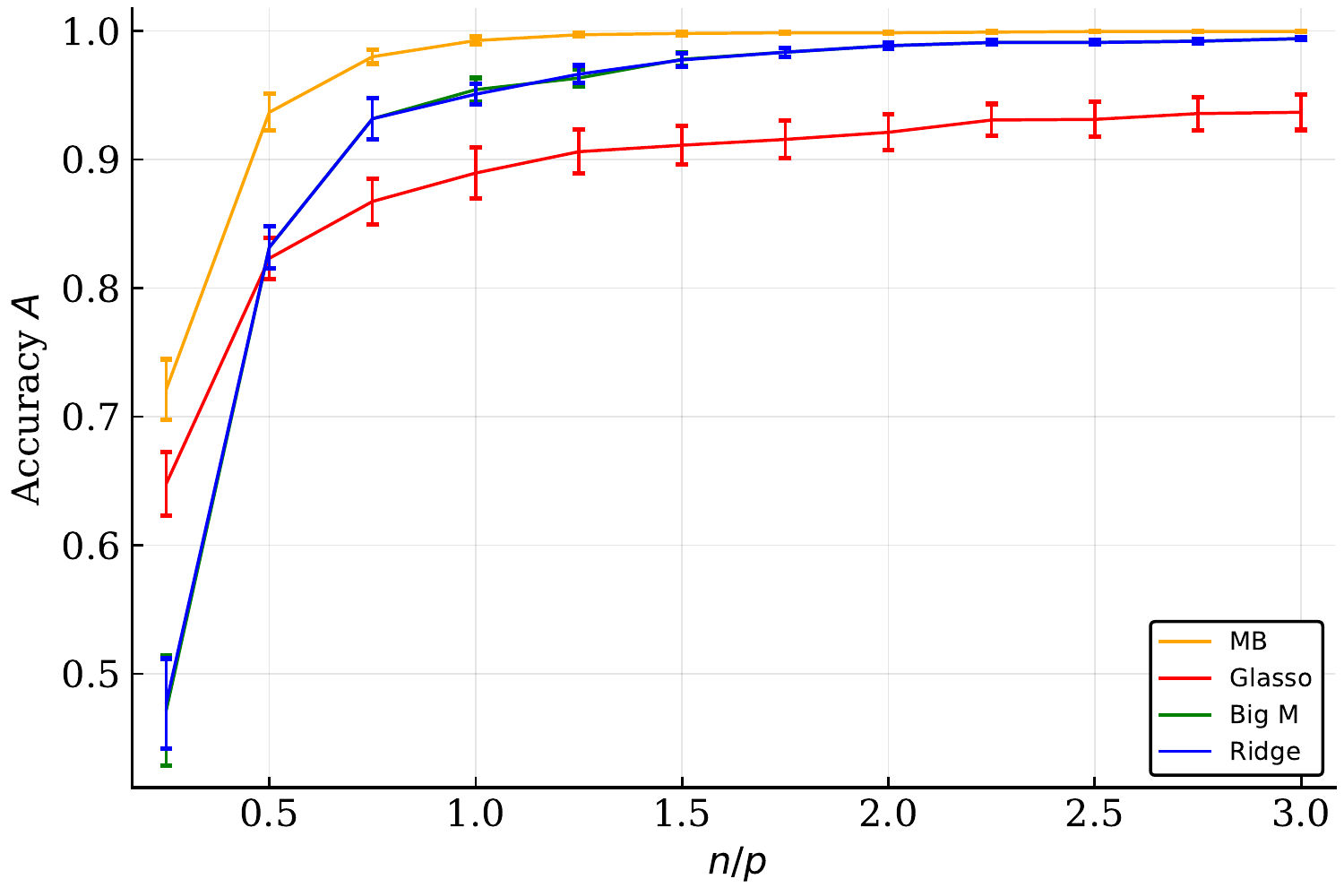}
	\caption{Accuracy $A$ vs. $n/p$.}
\end{subfigure} %
~
\begin{subfigure}[t]{.45\linewidth}
	\centering
	\includegraphics[width=\linewidth]{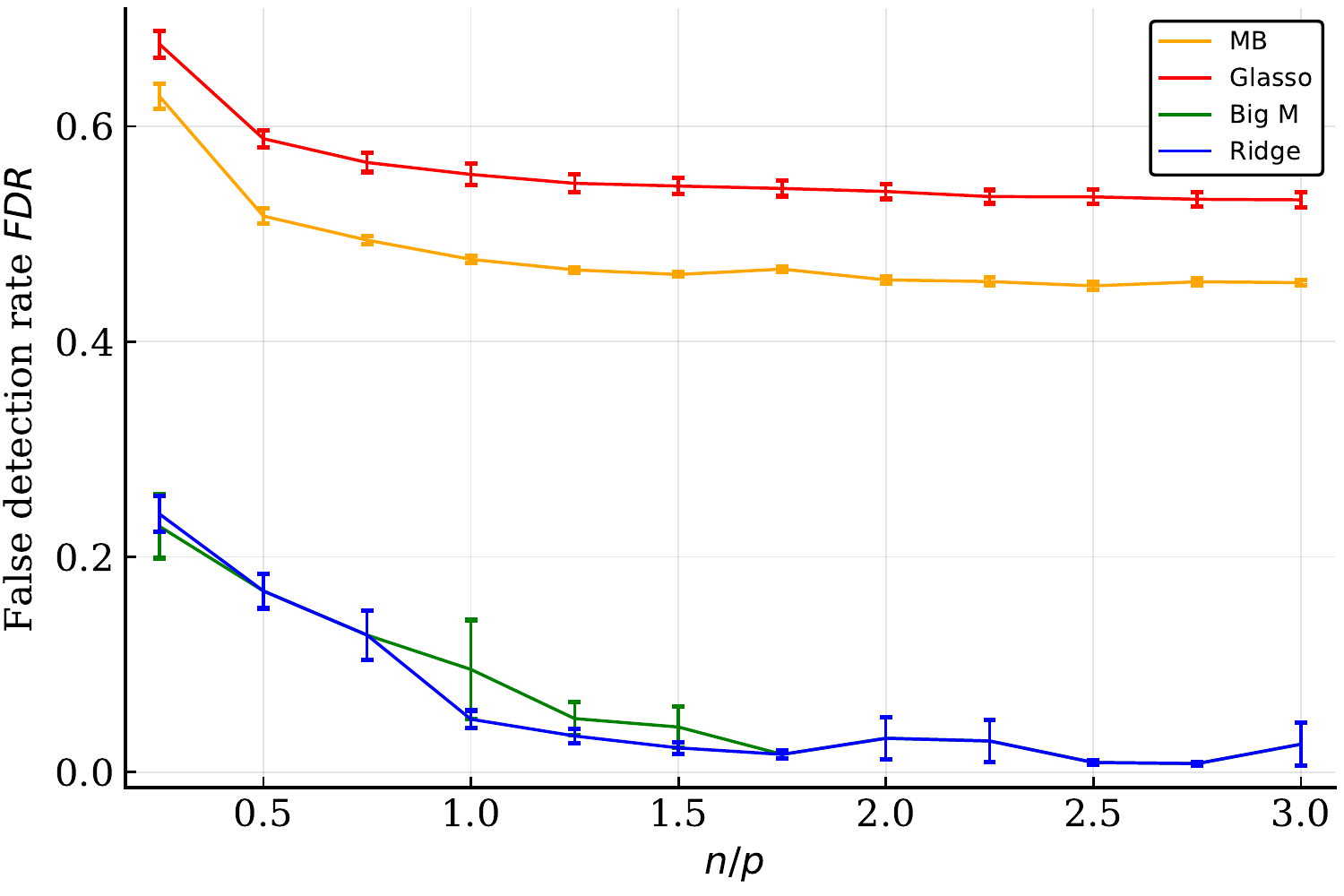}
	\caption{False detection rate $FDR$ vs. $n/p$.}
\end{subfigure}
\caption{Impact of the number of samples $n/p$ on support recovery. Results are averaged over $10$ instances with $p=200$, $t=1\%$. Hyper-parameters are tuned using out-of-sample negative log-likelihood.}
\label{fig:scale.n.ll}
\end{figure}
 
\paragraph{Sparsity level $t$} Recall that the sparsity level $t$ relates to the number of nonzero upper-diagonal coefficients of $\mathbf{\Theta}_0$ through the relationship 
$$k_{true} = \left \lfloor t \: \dfrac{p (p-1)}{2} \right \rfloor.$$
From Section \ref{sec:covsel.exp}, we observed that the separation Problem \eqref{eq:separationProblem} is increasingly harder to solve as $t$ increases. In addition, the combinatorics of the master Problem \eqref{eq:BinForm2} also increases with $t$, since the size of the feasible set $\mathcal{S}^{k_{true}}_p$ grows exponentially with $k_{true}$ as long as $k_{true} \leqslant \tfrac{p(p-1)}{4}$ (i.e., $t\leqslant 0.5$). Figure \ref{fig:scale.t.ll} represents accuracy and false detection rate as $t$ increases, for all methods, using negative log-likelihood as a cross-validation criterion. We report negative log-likelihood and results with $BIC_{1/2}$ as the cross-validation criterion in Appendix \ref{sec:A.statperf} (Figures \ref{fig:scale.t.loglik} and \ref{fig:scale.t.ebic} respectively).

\begin{figure}[H]
\centering
\begin{subfigure}[t]{.45\linewidth}
	\centering
	\includegraphics[width=\linewidth]{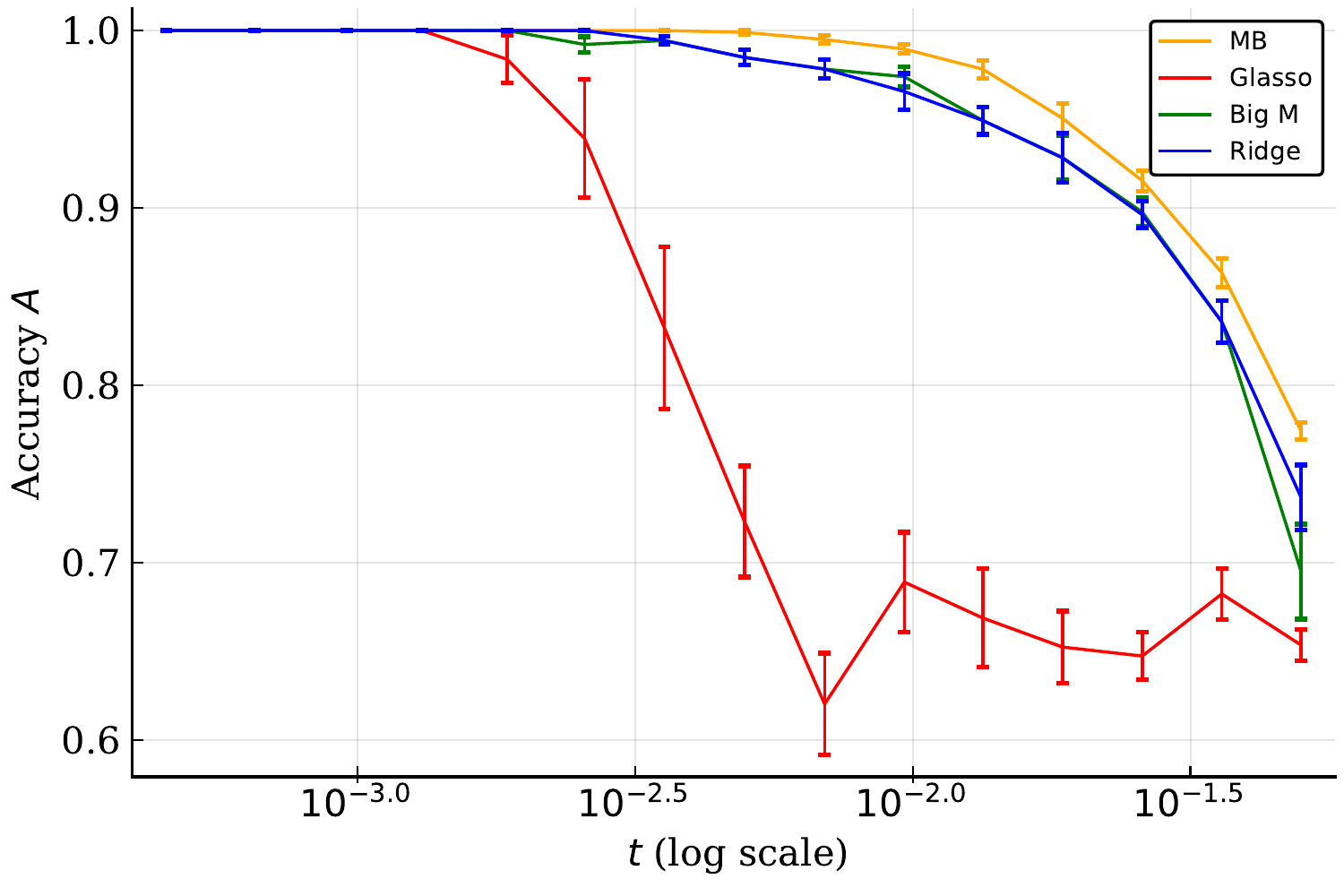}
	\caption{Accuracy $A$ vs. $t$.}
\end{subfigure} %
~
\begin{subfigure}[t]{.45\linewidth}
	\centering
	\includegraphics[width=\linewidth]{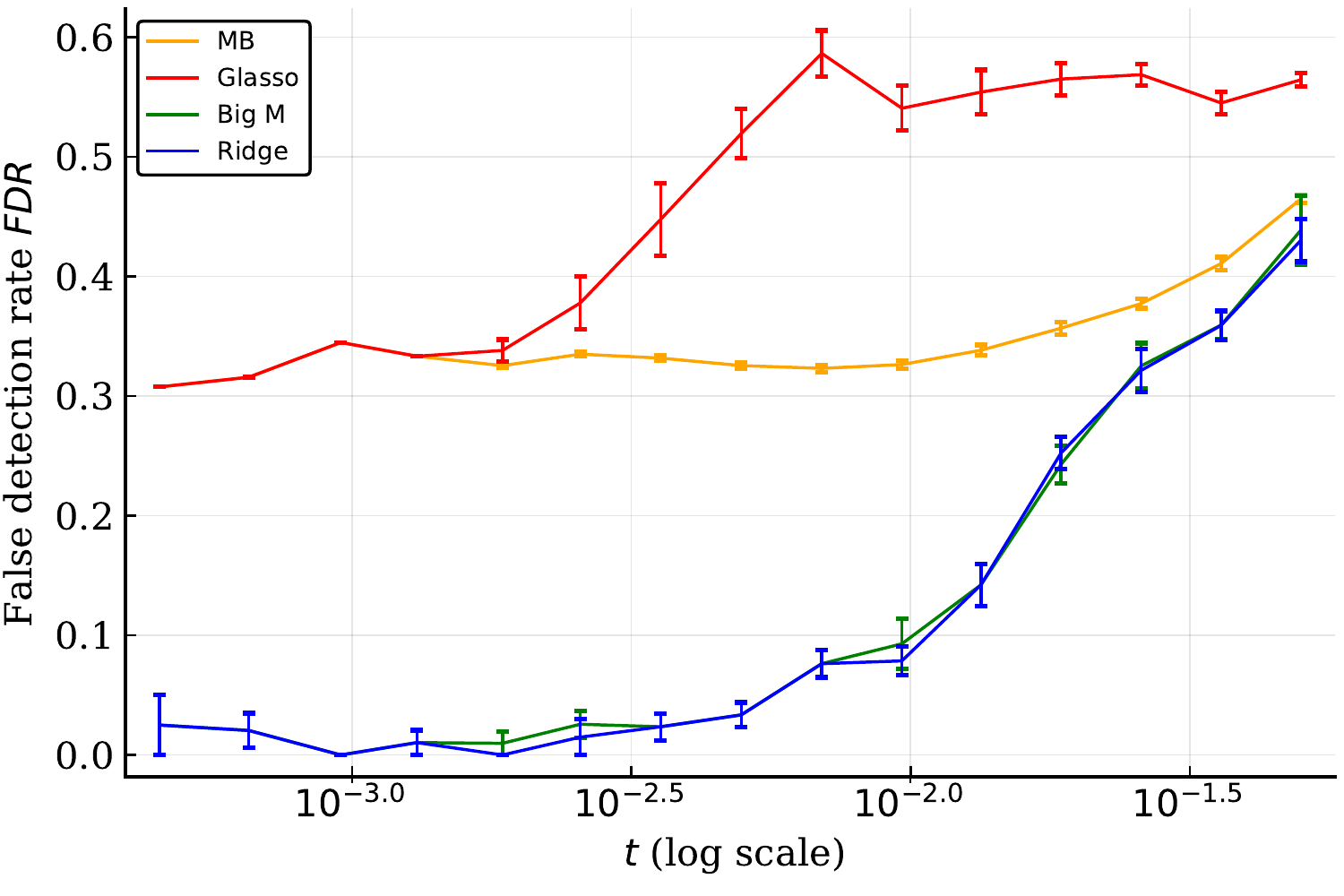}
	\caption{False detection rate $FDR$ vs. $t$.}
\end{subfigure}
\caption{Impact of the  sparsity level $t$ on support recovery. Results are averaged over $10$ instances with $p=200$, $n=p$. Hyper-parameters are tuned using the out-of-sample negative log-likelihood.}
\label{fig:scale.t.ll}
\end{figure}

\paragraph{Dimension $p$} For $n/p$ and $t$ fixed, the sparse precision matrix estimation problem should not be statistically more difficult as $p$ increases, but computationally more expensive. We report results in Appendix \ref{sec:A.statperf}. Figures \ref{fig:scale.p.ll} and \ref{fig:scale.p.ebic} report resulting accuracy and false detection rate as $p$ increases, using negative log-likelihood and $BIC_{1/2}$ respectively as a cross-validation criterion. Figure \ref{fig:scale.p.loglik} reports the impact of $p$ on out-of-sample negative log-likelihood, Figure \ref{fig:scale.p.time} the impact on time. Interestingly, the big-$M$ formulation is harder to scale than the ridge regularization, due to the additional constraints. As a result, fewer cuts were generated within the 5-minute time limit and the resulting precision matrix shows a different accuracy/false detection trade-off with relatively poorer out-of-sample log-likelihood as $p$ increases.

% -------------------------------------------------------------------------------------------------
% SECTION 5.1: ANALYSIS OF BREAST CANCER DATASET
% -------------------------------------------------------------------------------------------------

\subsection{Analysis of a Breast Cancer Dataset} \label{sec:exp.cancer}
We apply our method on a real breast cancer dataset analyzed in \cite{hess2006pharmacogenomic}. The dataset can be found at \url{http:// bioinformatics.mdanderson.org/}. The dataset consists of 22,283 gene expression levels for 133 patients, including 34 with pathological complete response (pCR) and 99 with residual disease (RD). The pCR subjects are considered to have a high chance of cancer-free survival in the long term, and thus it is of interest to study the response states of the patients (pCR or RD) to preoperative chemotherapy. The main objective of this analysis is to estimate the inverse covariance matrix of the gene expression levels and then apply linear discriminant analysis (LDA) to predict whether or not a subject can achieve the pCR state. 

The dataset has been studied in \cite{fan2009network} using Glasso, revised Glasso, and SCAD. Later the same analysis was performed with the CLIME estimator \cite{cai2011constrained}. For the sake of consistency, we perform the same analysis, but use our method to estimate inverse covariance matrices when needed. We first briefly describe how the data is prepared and analyzed. We then present our results and compare with known results in \cite{fan2009network,cai2011constrained}. 

The data is first randomly divided into testing and training sets using stratified sampling. 5 pCR subjects and 16 RD subjects are randomly chosen to constitute the testing data. The remaining 112 subjects are chosen to constitute the training data. This process is repeated 100 times and the following data preparation techniques are used on each of the 100 instances of the training and testing data. A two-sample t-test is performed between the two groups in the training dataset to determine the most significant genes; we retain the $113$ genes with the smallest $p$-values as the variables for prediction and the rest are discarded. The data for each variable (gene) is then standardized by dividing the data with the corresponding standard deviation, estimated from the training dataset. 

We next perform the linear discriminant analysis. We assume the normalized gene expression data are normally distributed as $\mathcal{N}(\mathbf{\mu}_k,\mathbf{\Sigma})$, where the two groups have the same covariance $\mathbf{\Sigma}$, but different means, $\mathbf{\mu}_k$ ($k=1$ for pCR and $k=2$ for RD). The linear discriminant scores are as follows:
$$\delta_k(\mathbf{x}) = \mathbf{x}^{\top} \hat{\mathbf{\Sigma}}^{-1} \hat{\mathbf{\mu}}_k - \frac{1}{2} \hat{\mathbf{\mu}}_k^{\top} \hat{\mathbf{\Sigma}}^{-1} \hat{\mathbf{\mu}}_k + \log \pi_k,$$
where $\pi_k = n_k/n$ is the proportion of the number of observations in the training data belonging to class $k$, and the classification rule is given by $\argmax_k \delta_k(\mathbf{x})$. Based on each training dataset, we estimate the mean $\hat{\mathbf{\mu}}_k$ as,
$$\hat{\mathbf{\mu}}_k = \frac{1}{n_k} \sum_{i \in class-k} \mathbf{x}_i \quad \text{for } k = 1,2,$$ 
and the precision matrix $\hat{\mathbf{\Sigma}}^{-1}$ using the cardinality constrained problem. Since the sample size is less than the dimension of the matrix, the empirical covariance is not invertible and can not  be used in LDA. 

\begin{table}[h!] 
\centering 
\begin{tabular}{| c | c |}
\toprule 
Comparison Metrics & Description \\
\midrule 
\multirow{2}{4.5em}{\centering Specificity} & \multirow{2}{10em}{\centering $\frac{TN}{TN + FP}$} \\
& \\
\multirow{2}{4.5em}{Sensitivity} & \multirow{2}{10em}{\centering $\frac{TP}{TP + FN}$} \\ 
& \\
\multirow{2}{2.5em}{MCC} & \multirow{2}{20em}{ \centering $\frac{TP \times TN - FP \times FN}{\sqrt{(TP + FP)(TP + FN)(TN + FP)(TN + FN)}}$} \\ 
& \\
\bottomrule
\end{tabular}
\caption{Metrics used for prediction performance comparison for the breast cancer dataset. TP, TN, FP, and FN are the number of true positives, true negatives, false positives and false negatives, respectively. Positives correspond to pCR subjects and negatives correspond to RD subjects.}
\label{tab: metrics} 
\end{table}

The classification performance of $\delta_k$ is clearly associated with the estimation performance of $\hat{\mathbf{\Sigma}}^{-1}$. Let true positive (TP) be the number of pCR subjects $\delta_k$ identifies as pCR subjects and let true negative (TN) be the number of RD subjects $\delta_k$ identifies as RD Subjects. To compare prediction performance, we use comparison metrics: specificity, sensitivity, and also Matthews Correlation Coefficient (MCC). They are each defined in Table \ref{tab: metrics}. MCC is widely used in machine learning for assessing the quality of a binary classifier; it takes true and false, positives and negatives, into account and is generally regarded as a balanced measure. A larger MCC value indicates a better classifier \cite{fan2009network}.

\begin{table}[h!] 
\centering 
\begin{tabular}{| c | c c c c |}
\toprule
Method & Specificity & Sensitivity & MCC & NNZ \\
\midrule
Glasso & $0.768 \text{ } (0.009)$ & $0.630  \text{ } (0.021)$ & $0.366  \text{ } (0.018)$ & $3923  \text{ } (2)$ \\ 
Adaptive Lasso & $ 0.787  \text{ } (0.009)$ & $0.622  \text{ } (0.022)$ & $0.381  \text{ } (0.018)$ & $1233  \text{ } (1)$ \\
SCAD & $0.794  \text{ } (0.009)$ & $0.634  \text{ } (0.022)$ & $0.402  \text{ } (0.020)$ & $674  \text{ } (1)$ \\ 
CLIME & $0.749  \text{ } (0.009)$ & $0.806  \text{ } (0.017)$ & $0.506  \text{ } (0.020)$ & $492  \text{ } (7)$ \\ 
big-$M$ & $0.779  \text{ } (0.011)$ & $0.717  \text{ } (0.019)$ & $0.460  \text{ } (0.019)$ & $436  \text{ } (3)$ \\
Ridge & $ 0.775 \text{ } (0.011)$ & $0.716 \text{ } (0.020)$ & $ 0.453 \text{ } (0.021)$ & $427 \text{ } (3)$ \\
\bottomrule
\end{tabular}
\caption{Comparison of estimators on the breast cancer dataset. Data for Glasso, revised Glasso and SCAD is from \cite{fan2009network} and data for CLIME is from \cite{cai2011constrained}. Average performance is reported on 100 instances of training and testing data; standard deviations are included in parentheses. NNZ refers to the number of nonzero entries in the estimate.}
\label{tab:bcancercomp} 
\end{table}

We perform the LDA for each of the 100 instances and report a summary of average performance in Table \ref{tab:bcancercomp}. {For each experiment, we calibrate the parameters $k$ and $M$ / $\gamma$ using the extended Bayesian information criterion on the training data. We observe that our proposed methods outperform Lasso-based methods on all aspects. Our discrete optimization formulations are comparable to SCAD and Clime, yet not dominated nor dominating by either of the two. Big-$M$ and ridge formulations improve over SCAD in terms of sensitivity and MCC, and over Clime in terms of specificity. On the contrary, SCAD ranks first on specificity and Clime on sensitivity and MCC. However, the biggest advantage of discrete formulations over the others is that they produce sparser estimates. } This is especially desirable in the context of graphical models, when it is desirable to induce sparsity for explanatory and predictive power. 

% -------------------------------------------------------------------------------------------------
% SECTION : EXTENSION
% -------------------------------------------------------------------------------------------------

\section{Extension to graphical model estimation with structural information} \label{sec:structural} 
In this section, we illustrate the modeling power of our mixed-integer formulation. In graphical models estimation,
it is not unusual to have some information or intuition about the correlation structure between variables \cite{drton2017structure}, information which can easily be encoded in our framework by additional constraints on the binary variables $\mathbf{Z}$.  

\paragraph{Sparsity} In this paper, we focused on imposing sparsity on the precision matrix $\mathbf{\Theta}$. This requirement translates into the linear constraint
\begin{align*}
\sum_{i>j} Z_{ij} \leqslant k.
\end{align*} 
\paragraph{Partial knowledge of the support} In some settings, the modeler has some partial knowledge of the correlation structure and can inform the optimization problem through the additional constraints
\begin{align*}
Z_{ij} = 0, &\mbox{ if } (i,j) \in \mathcal{S}_0, \\
Z_{ij} = 1, &\mbox{ if } (i,j) \in \mathcal{S}_1, 
\end{align*} 
where $\mathcal{S}_0$ (resp. $\mathcal{S}_1$) is a set of indices for which $\Theta_{ij}$s are known to be $0$ (resp. $\neq 0$). 
\paragraph{Degree} Information about the degree of each variable in the underlying structure (or graph) might also be relevant \cite{ma2015structure}. In a protein contact graph for example, the degree of each node is upper bounded by some constant. With our framework, the degree of any variable $i$ is given by $d_i := \sum_{j>i} Z_{ij}$, so that adding the linear constraints 
\begin{align*}
\ell_i \leqslant d_i \leqslant u_i,\ \forall i
\end{align*}
would enforce lower ($\ell_i$) and upper ($u_i$) bounds on the node degrees. In a more flexible fashion,
\begin{align*}
\left| \, \dfrac{1}{p} \sum_i d_i - \overline{d} \, \right| \leqslant \epsilon,
\end{align*}
requires the average node degree to be within $\epsilon$ from a given target $\overline{d}$. Similarly, quadratic constraints could be added in order to match second moments. Finally, many real-world networks, including the network of webpages or some gene regulatory networks, involve nodes which have a lot more edges than the others \cite{tan2014learning}. Our framework can account for such hubs by introducing additional binary variables $y_i,\, i=1,\dots, p$ and adding the following constraints
\begin{align*}
d_i \leqslant d_{low} + (d_{high}-d_{low}) y_i, &\ \forall i, \\
\sum_i y_i \leqslant m,
\end{align*}   
where $d_{high}$ (resp. $d_{low}$) is the maximum degree of a hub (resp. non-hub) node and $m$ is an upper-bound on the total number of hubs in the network. 

\paragraph{Tree structure} Finally, tree-structured graphical models have been extensively studied in the literature \cite{chow1968approximating} for they are sparse and allow efficient inference. Introducing additional binary variables $y_{i,j}^k$ for all ordered triples $(i,j,k)$ of pairwise different nodes, \cite{martin1991using} provided an extended formulation for a spanning tree:
\begin{align*}
\| Z \|_0 &= p-1, \\ 
y_{ij}^k + y_{ji}^k &= Z_{ij}, &\; \forall i,j = 1,\dots, p, \, i<p, \;\forall k = 1,\dots, p, \\
\sum_{j: {j \notin \{i,k\}}} y_{ij}^k  &= 1 - Z_{ik} &\; \forall i,k =1,\dots,p, \, i < k,
\end{align*}
where $y_{ij}^k=1$ if and only if the edge $(i,j)$ is contained in the tree and $k$ is in the component of $j$ when removing $(i,j)$ from the tree.

% -------------------------------------------------------------------------------------------------
% SECTION : SUMMARY
% -------------------------------------------------------------------------------------------------

\section{Summary} \label{sec:conclusion}
In this work, we use a variety of modern optimization methods to provide the first provably exact algorithm for solving the cardinality-constrained negative log-likelihood Problem \eqref{eq:L0MLE}. {Through the unifying lens of regularization, we show that the well known big-$M$ constraints are not only a formulation technique but more importantly a smoothing procedure. On that matter, ridge regularization can be considered as a fruitful alternative. Our cutting-plane approach has the additional benefit of treating separately the combinatorial aspect of the problem from the SDP component of it.} The method provides provably optimal solutions, and delivers near optimal solutions in minutes for {$p$ in the $1, 000$s and sparsity level of the order of $1\%$}. Computational experiments on both synthetic and real data show that such discrete formulations {deliver solutions with increased out-of-sample predictive power and lower false detection rate than existing methods, while being as accurate}. 

% -------------------------------------------------------------------------------------------------
% APPENDICES
% -------------------------------------------------------------------------------------------------
\appendix

\section{Proofs of Theorem \ref{thm:dualL0regularized} and corollaries} \label{sec:A.dual} 
In this section, we detail the proof of Theorem \ref{thm:dualL0regularized}. We first specify the assumptions required on the regularizer $\Omega$, prove Theorem \ref{thm:dualL0regularized} and finally investigate some special cases of interest.
\subsection{Assumptions} \label{sec:A.dual.assumptions} 
We first assume that the function $\Omega$ is decomposable, i.e., there exist scalar functions $\Omega_{ij}$ such that 
\begin{equation}
\label{eq:A1} \tag{A1}
\forall \: \mathbf{\Phi}, \quad \Omega(\mathbf{\Phi}) = \sum_{i,j} \Omega_{ij}(\Phi_{ij}).
\end{equation}
In addition, we assume that for all $(i,j)$, $\Omega_{ij}$ is convex and tends to regularize towards zero. Formally, 
\begin{equation}
\label{eq:A2} \tag{A2}
\forall \: (i,j), \quad \min_x \: \Omega_{ij}(x) = \Omega_{ij}(0).
\end{equation}
Those first two assumptions are not highly restrictive and are satisfied by $\ell_\infty$-norm constraint (big-$M$), $\ell_1$-norm regularization (LASSO) or $\| \cdot \|_2^2$-regularization, among others.

For any function $f$, we denote with a superscript $\star$ its Fenchel conjugate \cite[see][chap. ~3.3]{boyd2004convex} defined as
\begin{align*}
 f^\star(y) := \sup_x \langle x, y \rangle - f(x).
\end{align*}
In particular, the Fenchel conjugate of any function $f$ is convex. Given Assumption \eqref{eq:A1}, 
\begin{align*}
\Omega^\star (\mathbf{R}) &= \sup_{\mathbf{\Phi}} \langle \mathbf{\Phi}, \mathbf{R} \rangle - \Omega(\mathbf{\Phi}), \\
&= \sum_{i,j} \sup_{\Phi_{ij}} \Phi_{ij} R_{ij} - \Omega_{ij}(\Phi_{ij}), \\
&= \sum_{i,j} \Omega_{ij}^\star (R_{ij}).
\end{align*}
As a result, it is easy to see that if $\Omega$ satisfies \eqref{eq:A1} and \eqref{eq:A2}, so does its Fenchel conjugate. 

Let us denote $\mathbf{A} \circ \mathbf{B}$ the Hadamard or component-wise product between matrices $\mathbf{A} $ and $\mathbf{B}$. Consider a matrix $\mathbf{R}$ and a support matrix $\mathbf{Z} \in \{0,1\}^{p \times p}$. The function $\mathbf{Z} \mapsto \Omega^\star( \mathbf{Z} \circ \mathbf{R} )$ is convex in $\mathbf{Z}$, by convexity of $\Omega^\star$. We now assume that it is linear in $\mathbf{Z}$, that is, there exists a function $\mathbf{\Omega}^\star: \: \mathbb{R}^{p \times p} \rightarrow \mathbb{R}^{p \times p}$ satisfying:
\begin{equation}
\label{eq:A3} \tag{A3}
\forall \: \mathbf{Z} \in \{0,1\}^{p \times p}, \forall \: \mathbf{R} \in \mathbb{R}^{p \times p}, \: \Omega^\star(  \mathbf{Z} \circ \mathbf{R} ) = \langle \mathbf{Z}, \mathbf{\Omega}^\star(\mathbf{R}) \rangle.
\end{equation}
%Note that assumption \eqref{eq:A3} is far from trivial. In particular, it is not satisfied by the $\ell_1$-regularization $\Omega(\mathbf{\Phi}) = \| \mathbf{\Phi} \|_1$ for which 
%\begin{align*}
%\Omega^\star(\mathbf{R}) = 
%\begin{cases} 0 & \mbox{ if  } \| \mathbf{R} \|_\infty \leqslant 1,\\ +\infty &\mbox{ otherwise.} \end{cases}
%\end{align*}
\subsection{Proof of Theorem \ref{thm:dualL0regularized}} \label{sec:A.dual.proof} 
Given $\mathbf{Z} \in \{0,1\}^{p \times p}$ such that ${Z}_{ii} = 1$ for all $i = 1,\dots,p$, we first prove that under assumptions \eqref{eq:A1} and \eqref{eq:A2}:
\begin{align*}
\tilde{h}(\mathbf{Z}) &:= \quad \min_{\mathbf{\Theta} \succ \mathbf{0}} \quad \langle \overline{\mathbf{\Sigma}}, \mathbf{\Theta } \rangle - \log \det \mathbf{\Theta} + \Omega(\mathbf{\Theta}) \quad \mbox{  s.t.  } {\Theta}_{ij} = 0 \mbox{  if  } {Z}_{ij} = 0 \: \forall (i,j), \\
&= \quad \max_{\mathbf{R}:\overline{\mathbf{\Sigma}} + \mathbf{R} \succ \mathbf{0}} \:  p +\log \det (\overline{\mathbf{\Sigma}} + \mathbf{R}) - \Omega^\star ( \mathbf{Z} \circ \mathbf{R}). 
\end{align*}
Then, Assumption \eqref{eq:A3} will conclude the proof. 

\begin{proof}
We decompose the minimization problem \textit{\`a la Fenchel}.
\begin{align*}
\tilde{h}(\mathbf{Z}) &= \min_{\mathbf{\Theta} \succ \mathbf{0}} \: \langle \overline{\mathbf{\Sigma}}, \mathbf{\Theta} \rangle - \log \det \mathbf{\Theta}  + \Omega(\mathbf{\Theta}) \: \mbox{  s.t.  } {\Theta}_{ij} = 0 \mbox{  if  } {Z}_{ij} = 0, \\
&= \min_{\mathbf{\Theta}\succ \mathbf{0}, \mathbf{\Phi}} \: \langle \overline{\mathbf{\Sigma}}, \mathbf{\Theta} \rangle - \log \det \mathbf{\Theta} + \Omega(\mathbf{Z} \circ \mathbf{\Phi}) \: \mbox{  s.t.  } {\Theta}_{ij} = {Z}_{ij} {\Phi}_{ij}, \\
&= \min_{\mathbf{\Theta} \succeq \mathbf{0}, \mathbf{\Phi}} \: \langle \overline{\mathbf{\Sigma}}, \mathbf{\Theta} \rangle - \log \det \mathbf{\Theta} + \Omega(\mathbf{Z} \circ \mathbf{\Phi}) \: \mbox{  s.t.  } \mathbf{\Theta} = \mathbf{Z} \circ \mathbf{\Phi}.
\end{align*}
In the last equality, we omitted the constraint $\mathbf{\Theta} \succ \mathbf{0}$, which is implied by the domain of $\log \det$. Assuming \eqref{eq:A1} and \eqref{eq:A2} hold, the regularization term $\Omega(\mathbf{Z} \circ \mathbf{\Phi})$ can be replaced by $\Omega(\mathbf{\Phi})$ and 
\begin{align*}
\tilde{h}(\mathbf{Z}) &= \min_{\mathbf{\Theta} \succeq \mathbf{0}, \mathbf{\Phi}} \: \langle \overline{\mathbf{\Sigma}}, \mathbf{\Theta} \rangle - \log \det \mathbf{\Theta} + \Omega(\mathbf{\Phi}) \: \mbox{  s.t.  } \mathbf{\Theta} = \mathbf{Z} \circ \mathbf{\Phi}.
\end{align*}
The above objective function is convex in $(\mathbf{\Theta}, \mathbf{\Phi})$, the feasible set is a non-empty - $\mathbf{\Theta} = \mathbf{\Phi} = \mathbf{I}_p$ is feasible - convex set, and Slater's conditions are satisfied. Hence, strong duality holds.  
\begin{align*}
\tilde{h}(\mathbf{Z}) &= \min_{\mathbf{\Theta} \succeq \mathbf{0}, \mathbf{\Phi}} \: \langle \overline{\mathbf{\Sigma}}, \mathbf{\Theta} \rangle - \log \det \mathbf{\Theta} + \Omega(\mathbf{\Phi}) \: \mbox{  s.t.  } \mathbf{\Theta} = \mathbf{Z} \circ \mathbf{\Phi}, \\
&= \min_{\mathbf{\Theta} \succeq \mathbf{0}, \mathbf{\Phi}} \: \max_{\mathbf{R}} \: \langle \overline{\mathbf{\Sigma}}, \mathbf{\Theta} \rangle - \log \det \mathbf{\Theta}  + \Omega(\mathbf{\Phi})  \: + \langle \mathbf{\Theta} - \mathbf{Z} \circ \mathbf{\Phi}, \mathbf{R} \rangle, \\
&= \max_{\mathbf{R}} \: \min_{\mathbf{\Theta} \succeq \mathbf{0}} \: \left[ \langle \overline{\mathbf{\Sigma}} + \mathbf{R}, \mathbf{\Theta} \rangle - \log \det \mathbf{\Theta} \right] + \min_{\mathbf{\Phi}} \left[ \Omega( \mathbf{\Phi}) - \langle \mathbf{Z} \circ \mathbf{\Phi}, \mathbf{R} \rangle \right].
\end{align*}
For the first inner-minimization problem, first-order conditions $\overline{\mathbf{\Sigma}} + \mathbf{R} - \mathbf{\Theta}^{-1} = \mathbf{0}$ lead to the constraint $\overline{\mathbf{\Sigma}}+ \mathbf{R} \succ 0$ and the objective value is $p + \log \det (\overline{\mathbf{\Sigma}} + \mathbf{R})$. The second inner-minimization problem is almost the definition of the Fenchel conjugate:
\begin{align*}
\min_{\mathbf{\Phi}} \Omega( \mathbf{\Phi}) - \langle \mathbf{Z} \circ \mathbf{\Phi}, \mathbf{R} \rangle  &= - \max_{\mathbf{\Phi}} \langle \mathbf{\Phi}, \mathbf{Z} \circ \mathbf{R} \rangle - \Omega( \mathbf{\Phi}), \\
&= - \Omega^\star( \mathbf{Z} \circ \mathbf{R} )
\end{align*}
Hence, 
\begin{align*}
h(\mathbf{Z}) &= \max_{\mathbf{R}: \overline{\mathbf{\Sigma}}+ \mathbf{R} \succ \mathbf{0}} \: p + \log \det (\overline{\mathbf{\Sigma}} + \mathbf{R}) - \Omega^\star( \mathbf{Z} \circ \mathbf{R} ).
\end{align*}
\end{proof}
\paragraph{Remark:} Notice that we proved that {$\Tilde h(\mathbf{Z})$} could be written as point-wise maximum of \emph{concave} functions of $\mathbf{Z}$. Assumption \eqref{eq:A3} is needed to ensure that the function in the maximization is convex in $\mathbf{Z}$ at the same time. 
\subsection{Special Cases and Corollaries} \label{sec:A.dual.bounds} 
\subsubsection{No regularization} We first consider the unregularized case of \eqref{eq:h(z)} where $\forall \: \mathbf{\Phi}, \: \Omega(\mathbf{\Phi}) = 0$. Assumptions \eqref{eq:A1} and \eqref{eq:A2} are obviously satisfied. Moreover, for any $\mathbf{R}$,
\begin{align*}
\Omega^\star(\mathbf{R}) &= \sup_\mathbf{\Phi} \langle \mathbf{\Phi}, \mathbf{R} \rangle
= \begin{cases} 0 & \mbox{if } \mathbf{R} = \mathbf{0}, \\ +\infty &\mbox{otherwise.}
\end{cases}
\end{align*}
With the convention that $0 \times \infty = 0$, Assumption \eqref{eq:A3} is satisfied and Theorem \ref{thm:dualL0regularized} holds: 
\begin{align*}
{h}(\mathbf{Z})
&= \quad \max_{\mathbf{R}: \overline{\mathbf{\Sigma}}+ \mathbf{R} \succ \mathbf{0}} \:  p +\log \det (\overline{\mathbf{\Sigma}} + \mathbf{R}) - \langle \mathbf{Z} , \mathbf{\Omega}^\star ( \mathbf{R}) \rangle, \\
&=  \max_{\mathbf{R}: \overline{\mathbf{\Sigma}} + \mathbf{R} \succ \mathbf{0}} \:  p +\log \det (\overline{\mathbf{\Sigma}}+ \mathbf{R}) \quad \mbox{  s.t.  }  {Z}_{ij} {R}_{ij} = 0, \, \forall (i,j).
\end{align*}
In particular, this reformulation proves that ${h}(\mathbf{Z})$ is convex\footnote{Convexity of ${h}(\mathbf{Z})$ can also be proved from the primal formulation \eqref{eq:h(z)} directly. Take two matrices  $\mathbf{Z}_1$ and $ \mathbf{Z}_2$, $\lambda \in (0,1)$, $\mathbf{Z} := \lambda \mathbf{Z}_1 + (1-\lambda) \mathbf{Z}_2$, then it follows from the definition  \eqref{eq:h(z)} that $ h(\mathbf{Z}) \leqslant \lambda h(\mathbf{Z}_1) + (1-\lambda) h(\mathbf{Z}_2)$.}, but that the coordinates of its sub-gradient $-\mathbf{\Omega}^\star(\mathbf{R}^\star(\mathbf{Z}))$ are either $0$ or $-\infty$, hence uninformative. Note that the same conclusion is true for $\ell_1$-regularization.

From the proof of Theorem \ref{thm:dualL0regularized}, one can derive a lower bound on $ \| \mathbf{\Theta}^\star \|_\infty$ which will be useful for big-$M$ regularization. 
\begin{theorem}
The solution of \eqref{eq:BinForm2} satisfies $\| \mathbf{\Theta}^\star \|_\infty \geqslant \frac{p}{\| \overline{\mathbf{\Sigma}} \|_1}$
\end{theorem}
\begin{proof}
For a feasible support $\mathbf{Z}$, denote the optimal primal and dual variables $\mathbf{\Theta}^\star(\mathbf{Z})$ and $\mathbf{R}^\star(\mathbf{Z})$ respectively. There is no duality gap and KKT condition $\mathbf{\Theta}^\star(\mathbf{Z})^{-1}= \overline{\mathbf{\Sigma}} + \mathbf{R}^\star(\mathbf{Z})$ holds, so that $ \langle \overline{\mathbf{\Sigma}}, \mathbf{\Theta}^\star(\mathbf{Z}) \rangle = p$. From H\"older's inequality, we obtain the desired lower bound.
\end{proof}

\subsubsection{Big-$M$ regularization} For the big-$M$ regularization, 
\begin{align*}
\Omega(\mathbf{\Theta}) = \begin{cases} 0 & \mbox{ if  } |\Theta_{ij}| \leqslant M_{ij}, \\ +\infty & \mbox{ otherwise} \end{cases},
\end{align*}
is decomposable with $\Omega_{i,j}(\Theta_{ij}) = 0$ if $| \Theta_{ij} | \leqslant M_{ij}$, $+\infty$ otherwise. Assumptions \eqref{eq:A1} and \eqref{eq:A2} are satisfied. Moreover, for any $\mathbf{R}$,
\begin{align*}
\Omega^\star(\mathbf{R}) &= \sup_\mathbf{\Phi \, : \, \| \mathbf{\Phi}\|_\infty \leqslant \mathbf{M}} \langle \mathbf{\Phi}, \mathbf{R} \rangle
= \| \mathbf{M} \circ \mathbf{R} \|_1.
\end{align*}
In particular, for any binary matrix $\mathbf{Z}$, 
\begin{align*}
\Omega^\star(\mathbf{Z} \circ \mathbf{R}) = \sum_{i,j} | M_{ij} Z_{ij} R_{ij} | = \sum_{i,j} M_{ij} Z_{ij} | R_{ij} |, 
\end{align*}
so that Assumption \eqref{eq:A3} is satisfied with $\mathbf{\Omega}^\star(\mathbf{R}) = \left( M_{ij} |R_{ij}| \right)_{ij}$.

\subsubsection{Ridge regularization} For the $\ell_2^2$-regularization, 
\begin{align*}
\Omega(\mathbf{\Theta}) = \dfrac{1}{2\gamma} \| \mathbf{\Theta} \|_2^2,
\end{align*}
is decomposable with $\Omega_{i,j}(\Theta_{ij}) = \tfrac{1}{2\gamma} \Theta_{ij}^2$. Assumptions \eqref{eq:A1} and \eqref{eq:A2} are satisfied. Moreover, for any $\mathbf{R}$,
\begin{align*}
\Omega^\star(\mathbf{R}) &= \sup_\mathbf{\Phi} \langle \mathbf{\Phi}, \mathbf{R} \rangle - \dfrac{1}{2\gamma} \| \mathbf{\Phi} \|_2^2
= \dfrac{\gamma}{2} \| \mathbf{R} \|_2^2
\end{align*}
In particular, for any binary matrix $\mathbf{Z}$, 
\begin{align*}
\Omega^\star(\mathbf{Z} \circ \mathbf{R}) = \dfrac{\gamma}{2} \sum_{i,j}  ( Z_{ij} R_{ij} )^2 = \dfrac{\gamma}{2}\sum_{i,j} Z_{ij} R_{ij}^2, 
\end{align*}
since $Z_{ij}^2 = Z_{ij}$, so that Assumption \eqref{eq:A3} is satisfied with $\mathbf{\Omega}^\star(\mathbf{R}) = \left( \tfrac{\gamma}{2} R_{ij}^2 \right)_{ij}.$

Moreover, from the proof of Theorem \ref{thm:dualL0regularized}, one can connect the norm of $\mathbf{\Theta}^\star(\mathbf{Z})$ and $\gamma$.
\begin{theorem}
\label{thm:ridge.bound}
For any support $\mathbf{Z}$, the norm of the optimal precision matrix $\mathbf{\Theta}^\star(\mathbf{Z})$ is bounded by 
\begin{align*}
\dfrac{\gamma}{2} \| \overline{\mathbf{\Sigma}} \|_2 \left( \sqrt{1 + \dfrac{4 p}{\gamma \| \overline{\mathbf{\Sigma}}\|_2^2}} -1\right)\leqslant \| \mathbf{\Theta}^\star(\mathbf{Z}) \|_2  \leqslant \sqrt{p \gamma}. 
\end{align*}
\end{theorem}
\begin{proof}
There is no duality gap:
\begin{align*}
\langle \overline{\mathbf{\Sigma}}, \mathbf{\Theta}^\star(\mathbf{Z}) \rangle - \log \det \mathbf{\Theta}^\star(\mathbf{Z})  + \dfrac{1}{2\gamma} \| \mathbf{\Phi}^\star(\mathbf{Z})  \|_2^2 = p + \log \det ( \overline{\mathbf{\Sigma}} + \mathbf{R}^\star(\mathbf{Z}) ) +\dfrac{\gamma}{2} \| \mathbf{Z} \circ \mathbf{R}^\star(\mathbf{Z}) \|_2^2.
\end{align*}
In addition, the following KKT conditions hold
\begin{align*}
\mathbf{\Theta}^\star ( \mathbf{Z})^{-1} &= \overline{\mathbf{\Sigma}} + \mathbf{R}^\star(\mathbf{Z}), \\
\mathbf{\Phi}^\star ( \mathbf{Z}) &= \gamma \mathbf{Z} \circ \mathbf{R}^\star ( \mathbf{Z}),
\end{align*}
where the second condition follows from the inner minimization problem defining $\Omega^\star$. All in all, we have
\begin{align*}
\langle \overline{\mathbf{\Sigma}}, \mathbf{\Theta}^\star(\mathbf{Z}) \rangle + \dfrac{1}{\gamma} \| \mathbf{\Phi}^\star(\mathbf{Z})  \|_2^2 = p.
\end{align*}
Since $\mathbf{\Sigma}$ and $\mathbf{\Theta}^\star(\mathbf{Z})$ are semi-definite positive matrices, $\langle \overline{\mathbf{\Sigma}}, \mathbf{\Theta}^\star(\mathbf{Z}) \rangle \geqslant 0$. Hence, $$ \| \mathbf{\Phi}^\star(\mathbf{Z})  \|_2 \leqslant \sqrt{p \gamma}.$$
To obtain the lower bound, we apply Cauchy-Schwartz inequality $\langle \overline{\mathbf{\Sigma}}, \mathbf{\Theta}^\star(\mathbf{Z}) \rangle \leqslant \| \overline{\mathbf{\Sigma}} \|_2 \| \mathbf{\Theta}^\star(\mathbf{Z}) \|_2$ and solve the quadratic equation 
$$ \dfrac{1}{\gamma} \| \mathbf{\Phi}^\star(\mathbf{Z})  \|_2^2 + \| \overline{\mathbf{\Sigma}} \|_2 \| \mathbf{\Theta}^\star(\mathbf{Z}) \|_2 - p \geqslant 0.$$
\end{proof}
In particular, the lower bound in Theorem \ref{thm:ridge.bound} is controlled by the factor $\tfrac{4 p}{\gamma \| \overline{\mathbf{\Sigma}} \|_2^2}$, suggesting an appropriate scaling of $\gamma$ to start a grid search with. 

\section{An optimization approach for finding big-$M$ values} \label{sec:A.bigM}
In this section, we present a method for obtaining suitable constants $\mathbf{M}$. The approach involves solving two optimization problems for each off-diagonal entry of the matrix being estimated. The problems provide lower and upper bounds for each entry of the optimal solution. First we present the problems, then we discuss how they are solved.

\subsection{Bound Optimization Problems} 
Let $\hat{\mathbf{\Theta}}$ be a feasible solution for \eqref{eq:L0MLE} and define, 
$$u := \langle \hat{\mathbf{\Theta}}, \overline{\mathbf{\Sigma}} \rangle - \log \det \hat{\mathbf{\Theta}}.$$
A simple way to obtain lower bounds for the $ij$th entry of the optimal solution is to solve
\begin{alignat}{2} 
\begin{aligned} 
& \min_{\mathbf{\Theta} \succ \mathbf{0}} \quad && \Theta_{ij}  \label{eq:bigm.lb}  \\
&\text{ s.t. } && \langle \overline{\mathbf{\Sigma}}, \mathbf{ \Theta } \rangle - \log \det \mathbf{\Theta} \leqslant u. 
\end{aligned} 
\end{alignat}
Likewise, to obtain upper bounds we solve
\begin{alignat}{2} 
\begin{aligned} 
& \max_{\mathbf{\Theta} \succ \mathbf{0}} \quad && \Theta_{ij}  \label{eq:bigm.ub}  \\
&\text{ s.t. } && \langle \overline{\mathbf{\Sigma}}, \mathbf{ \Theta } \rangle - \log \det \mathbf{\Theta} \leqslant u. 
\end{aligned} 
\end{alignat}

Note that it is sufficient to find a feasible solution $\hat{\mathbf{\Theta}}$ to formulate \eqref{eq:bigm.lb} and \eqref{eq:bigm.ub}, and a feasible solution with a smaller value leads to better bounds. \\

\subsection{Solution Approach} 
We describe the approach for the lower bound Problem \eqref{eq:bigm.lb} only, the upper bound Problem \eqref{eq:bigm.ub} being similar. 

{First, we make the additional assumption that $\overline{\mathbf{\Sigma}}$ is invertible. We know this assumption cannot hold in the high dimensional setting where $p > n$. Numerically, one can always argue that the lowest eigenvalues of $\overline{\mathbf{\Sigma}}$ {are} never exactly equal to zero but should be strictly positive. In this case however, these eigenvalues should be small and close to machine precision, making matrix inversion very unstable. Note that this extra assumption is required for problems \eqref{eq:bigm.lb} and \eqref{eq:bigm.ub} to be bounded.}

Problem \eqref{eq:bigm.lb} is a semidefinite optimization problem and there are $\sfrac{p(p+1)}{2}$ entries to bound so it is necessary to efficiently solve \eqref{eq:bigm.lb} and avoid solving so many SDPs. Instead, one can solve the dual of \eqref{eq:bigm.lb} very efficiently. Note an advantage for considering the dual is we do not need to solve the problem to optimality to obtain a valid bound. Using basic arguments from convex duality theory similar to the ones invoked in Section \ref{sec:A.dual.proof}, the dual problem for \eqref{eq:bigm.lb} writes 
\begin{equation} \label{eq:bigm.lb.dual} 
\max_{\lambda > 0} \left \{ \lambda \left( p - u + \log \det \left(\frac{1}{2 \lambda} (e_i e_j^T + e_j e_i^T) + \overline{\mathbf{\Sigma}} \right) \right) \right \}
\end{equation}
Computationally, problem (\ref{eq:bigm.lb.dual}) is easier to solve because it is a convex optimization problem with a scalar decision variable $\lambda$.

{Denote $g(\lambda)$} the objective function in the dual Problem \eqref{eq:bigm.lb.dual}. Algebraic manipulations yield
\begin{align*}
g(\lambda) &:= \lambda \left[ p - u + \log \det \left(\frac{1}{2 \lambda} (e_i e_j^T + e_j e_i^T) + \overline{\mathbf{\Sigma}} \right) \right], \\
&= \lambda \left[ p - u + \log \det (\overline{\mathbf{\Sigma}}) + \log  \left(1 + \dfrac{\Theta_{ij} }{\lambda} + \dfrac{\Theta_{ij}^2 - \Theta_{ii} \Theta_{jj}  }{4 \lambda^2} \right) \right],
\end{align*}
where $\mathbf{\Theta} = \overline{\mathbf{\Sigma}}^{-1}$. We can then easily derive the first and second derivatives of $g$ and apply Newton's method to solve Problem \eqref{eq:bigm.lb.dual}.

\section{Additional material on computational performance of the cutting-plane algorithm} \label{sec:A.comptime}
In this section, we consider the runtime of the cutting-plane algorithm on synthetic problems as in Section \ref{sec:exp.synth}.{ In Section \ref{sec:exp.synth.regk}, we illustrated how the regularization parameter $M$ or $\gamma$ can impact the convergence of the cutting-plane algorithm, so we focus in this section on the impact of the problem sizes $n$, $p$ and $k$.}

In particular, we study the time needed by the algorithm to find the optimal solution (opt-time) and to verify the solution's optimality (ver-time), as well as the number of cuts required (laz-cons). We carry out all experiments by generating 10 instances of synthetic data\footnote{For each instance, we generate a sparse precision matrix $\mathbf{\Theta}_0$ as in Section \ref{sec:exp.synth} and $n$ samples from the corresponding multivariate normal distribution} for $(p,k_{true}) \in \{30,50,80,120,200\} \times \{5,10\}$ and different values of $n$. We solve each instance of \eqref{eq:BinForm2} with big-$M$ regularization for $k = k_{true}$, $M=0.5$ and report average performance in Table \ref{tab:A.runtimes}. These computations are performed on 4 Intel E5-2690 v4 2.6 GHz CPUs (14 cores per CPU, no hyper threading) with 16GB of RAM in total. {We chose to fix the value of $M=0.5$ in order to isolate the impact of $p$, $k$ and $n$ on computational time, the specific value $0.5$ being informed by the knowledge of the ground truth}.

In general the algorithm provides an optimal solution in a matter of seconds, and a certificate of optimality in seconds or minutes even for $p$ in the $100$s. Optimal verification occurs significantly quicker when the sample size $n$ is larger because the sparsity pattern of the underlying matrix is easier to recover. However, we note that finding the optimal solution is not as {affected} by the sample size $n$. As $p$ or $k$ increase, optimal detection also does not significantly change, but optimal verification generally becomes significantly harder. Similar observations have been made for mixed-integer formulations of the best subset selection problem in linear and logistic regression \cite{bertsimas2016best}. We also observe that changes in $k$ have a more substantial impact on the runtime than changes in $n$ or $p$, especially when $p$ is large. {Finally, Meinshausen and B{\"u}hlmann's approximation is used as a warm-start and we observe that is often optimal, especially when $n/p$ is large.}

Thus, the cutting-plane algorithm in general provides an optimal or near-optimal solution fast, but optimal verification strongly depends on $p$, $k$, and $n$. Nonetheless, we observe that optimality of solutions can be verified for $p$ in the $100$s and $k$ in the $10$s in a matter of minutes. 

\begin{table}[h] 
\centering
\begin{tabular}{| c | c | c | c c c c |}
\toprule
$p$ & $k_{true}$ & $n$ & ver-time & opt-time & cut-time & laz-cons   \\
\midrule
\multirow{3}{2em}{\centering 30} & \multirow{3}{2em}{\centering 5} & $200$ & $2.37$ ($2.13$) & $0.0$ ($0.0$) & $1.95$ ($1.74$) & $28$ ($17.9$) \\
& & $150$ & $6.33$ ($7.34$) & $0.0$ ($0.0$) & $2.71$ ($3.14$) & $55$ ($55.8$) \\
& & $100$ & $30.7$ ($47.96$) & $0.0$ ($0.0$) & $14.46$ ($28.55$) & $258$ ($472.6$) \\
\hline 
\multirow{3}{2em}{\centering 30} & \multirow{3}{2em}{\centering 10} & $300$ & $31.11$ ($23.31$) & $5.05$ ($10.69$) & $14.32$ ($9.91$) & $265$ ($176.6$) \\
& & $250$ & $35.13$ ($28.89$) & $11.2$ ($13.13$) & $19.93$ ($14.91$) & $296$ ($204.8$) \\
& & $200$ & $33.7$ ($24.23$) & $7.75$ ($12.34$) & $15.35$ ($11.15$) & $290$ ($196.5$) \\
\midrule 
\multirow{3}{2em}{\centering 50} & \multirow{3}{2em}{\centering 5} & $200$ & $9.59$ ($9.06$) & $0.0$ ($0.0$) & $5.23$ ($3.66$) & $42$ ($25.2$) \\
& & $150$ & $29.43$ ($20.28$) & $0.0$ ($0.0$) & $18.49$ ($12.98$) & $153$ ($107.0$) \\
& & $100$ & $183.7$ ($243.73$) & $0.0$ ($0.0$) & $99.36$ ($118.0$) & $788$ ($937.8$)\\
\hline 
\multirow{3}{2em}{\centering 50} & \multirow{3}{2em}{\centering 10} & $300$ & $24.19$ ($20.29$) & $0.0$ ($0.0$) & $12.57$ ($10.37$) & $98$ ($80.8$) \\
& & $250$ & $31.37$ ($18.48$) & $0.0$ ($0.0$) & $15.2$ ($9.46$) & $122$ ($77.8$) \\
& & $200$ & $40.38$ ($29.27$) & $0.55$ ($1.73$) & $26.14$ ($19.14$) & $210$ ($149.1$) \\
\midrule 
\multirow{3}{2em}{\centering 80} & \multirow{3}{2em}{\centering 5} & $200$ & $70.12$ ($106.16$) & $0.0$ ($0.0$) & $51.56$ ($80.18$) & $154$ ($212.2$) \\
& & $150$ & $179.76$ ($175.22$) & $0.0$ ($0.0$) & $127.19$ ($110.85$) & $404$ ($348.3$) \\
& & $100$ & $988.9$ ($763.05$) & $0.0$ ($0.0$) & $482.83$ ($277.33$) & $1581$ ($990.9$) \\
\hline 
\multirow{3}{2em}{\centering 80} & \multirow{3}{2em}{\centering 10} & $300$ & $37.83$ ($9.17$) & $0.0$ ($0.0$) & $30.33$ ($10.11$) & $85$ ($25.2$) \\
& & $250$ & $71.4$ ($24.51$) & $0.0$ ($0.0$) & $47.06$ ($13.24$) & $139$ ($36.3$) \\
& & $200$ & $161.8$ ($74.35$) & $9.87$ ($31.2$) & $105.48$ ($41.14$) & $309$ ($121.6$)\\
\midrule 
\multirow{3}{2em}{\centering 120} & \multirow{3}{2em}{\centering 5} & $200$ & $152.54$ ($113.42$) & $34.89$ ($110.34$) & $119.24$ ($99.43$) & $170$ ($108.9$)  \\
& & $150$ & $713.45$ ($712.74$) & $251.25$ ($543.17$) & $480.18$ ($407.96$) & $740$ ($648.4$) \\
& & $100$ & $1793.67$ ($445.58$) & $646.84$ ($827.53$) & $1135.33$ ($320.83$) & $1671$ ($412.7$) \\
\hline 
\multirow{3}{2em}{\centering 120} & \multirow{3}{2em}{\centering 10} & $300$ & $238.7$ ($150.61$) & $0.0$ ($0.0$) & $172.75$ ($99.92$) & $224$ ($116.4$) \\
& & $250$ & $704.43$ ($568.93$) & $0.0$ ($0.0$) & $396.44$ ($238.16$) & $560$ ($348.5$) \\
& & $200$ & $1379.58$ ($666.52$) & $0.0$ ($0.0$) & $675.81$ ($248.96$) & $909$ ($393.1$)\\
\midrule
\multirow{3}{2em}{\centering 200} & \multirow{3}{2em}{\centering 5} & $200$ & $858.4$ ($770.03$) & $418.1$ ($496.15$) & $662.22$ ($567.77$) & $398$ ($335.0$) \\
& & $150$ & $1453.51$ ($614.68$) & $515.58$ ($548.82$) & $1023.24$ ($380.82$) & $723$ ($271.4$) \\
& & $100$ & $2000.28$ ($0.42$) & $917.42$ ($596.49$) & $1427.69$ ($139.69$) & $1024$ ($90.6$) \\
\hline 
\multirow{3}{2em}{\centering 200} & \multirow{3}{2em}{\centering 10} & $300$ & $934.55$ ($428.66$) & $337.16$ ($442.36$) & $646.12$ ($255.69$) & $368$ ($141.1$) \\
& & $250$ & $1792.1$ ($353.35$) & $354.84$ ($362.0$) & $1062.81$ ($205.64$) & $657$ ($167.6$) \\
& & $200$ & $2000.47$ ($0.9$) & $571.71$ ($571.04$) & $1198.26$ ($109.66$) & $763$ ($104.5$)
 \\
\bottomrule 
\end{tabular}
\caption{Average performance on instances of synthetic data with $k = k_{true}$. All problems are solved to a tolerance gap of $10^{-4}$, where the tolerance gap is the percentage difference between the final lower and upper bounds. Title ver-time and opt-time refer to the time (in seconds) it takes to verify optimality and to find the optimal solution respectively, cut-time refers to the amount of time spent solving the separation problems, and laz-cons refers to the number of lazy constraints generated. We report average time over $10$ random instances (and standard deviation).}
\label{tab:A.runtimes}
\end{table}

\section{Additional comparisons on statistical performance} \label{sec:A.statperf}
We report here additional results from the experiments conducted in Section \ref{sec:exp.synth}. 
\subsection{Comparisons for varying sample sizes $n/p$}

 \begin{figure}[H]
\centering
\begin{subfigure}[t]{.4\linewidth}
	\centering
	\includegraphics[width=\linewidth]{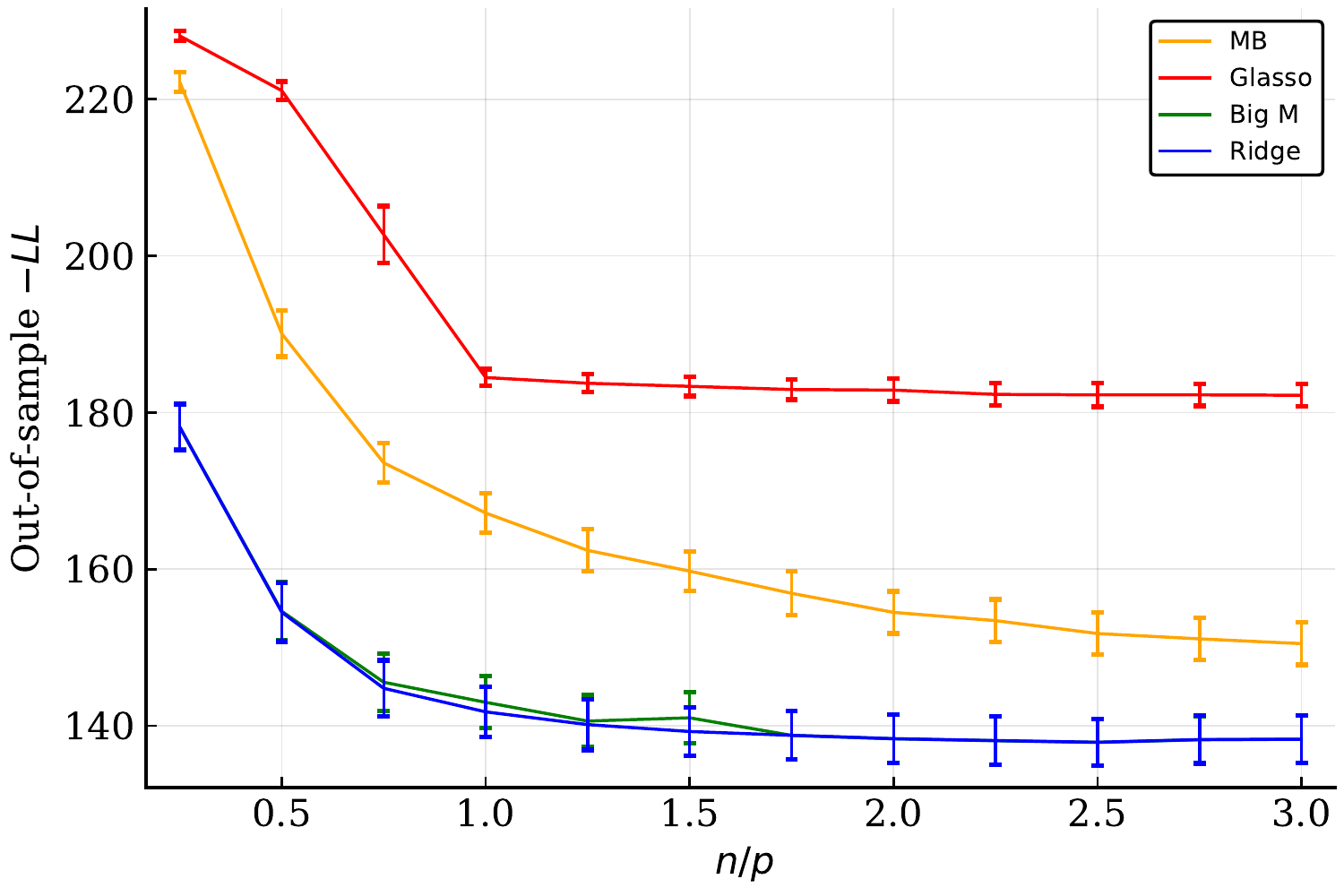}
	\caption{$BIC_{1/2}$ as a CV criterion.}
\end{subfigure} %
~
\begin{subfigure}[t]{.4\linewidth}
	\centering
	\includegraphics[width=\linewidth]{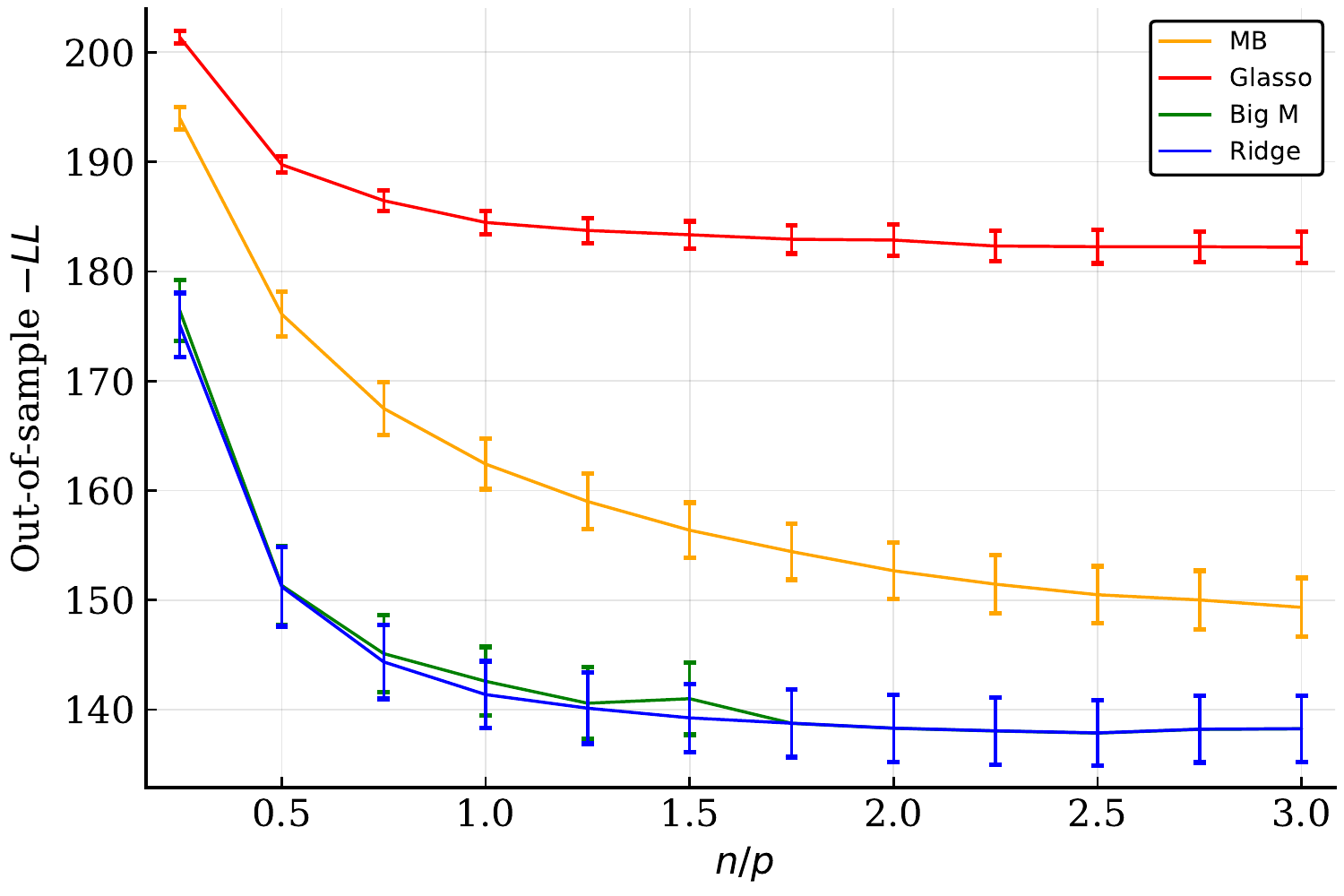}
	\caption{$-LL$ as a CV criterion.}
\end{subfigure}
\caption{Impact of the number of samples $n/p$ on out-of-sample negative log-likelihood. Results are averaged over $10$ instances with $p=200$, $t=1\%$.}
\label{fig:scale.n.loglik}
\end{figure}
\begin{figure}[H]
\centering
\begin{subfigure}[t]{.4\linewidth}
	\centering
	\includegraphics[width=\linewidth]{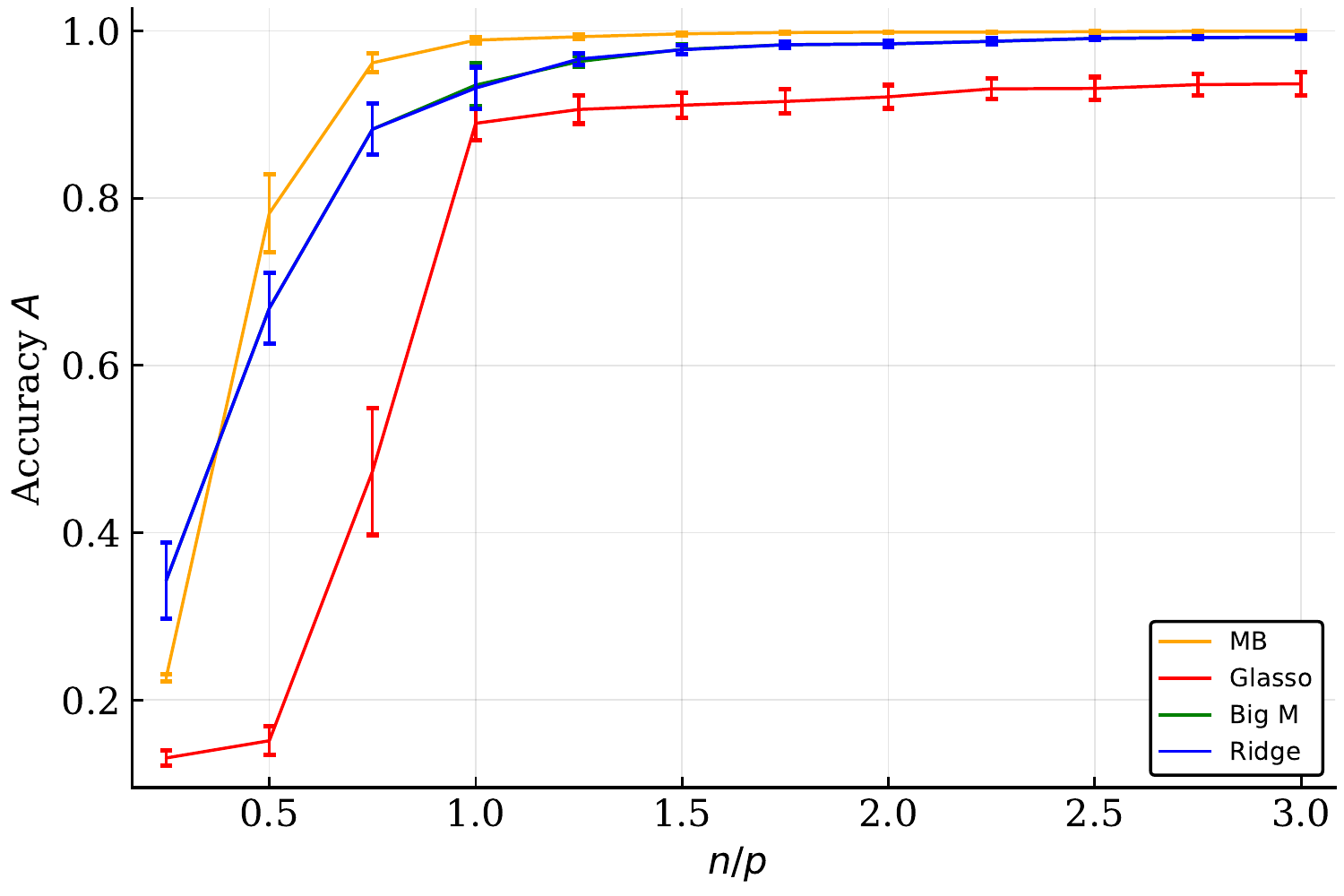}
	\caption{Accuracy $A$ vs. $n/p$.}
\end{subfigure} %
~
\begin{subfigure}[t]{.4\linewidth}
	\centering
	\includegraphics[width=\linewidth]{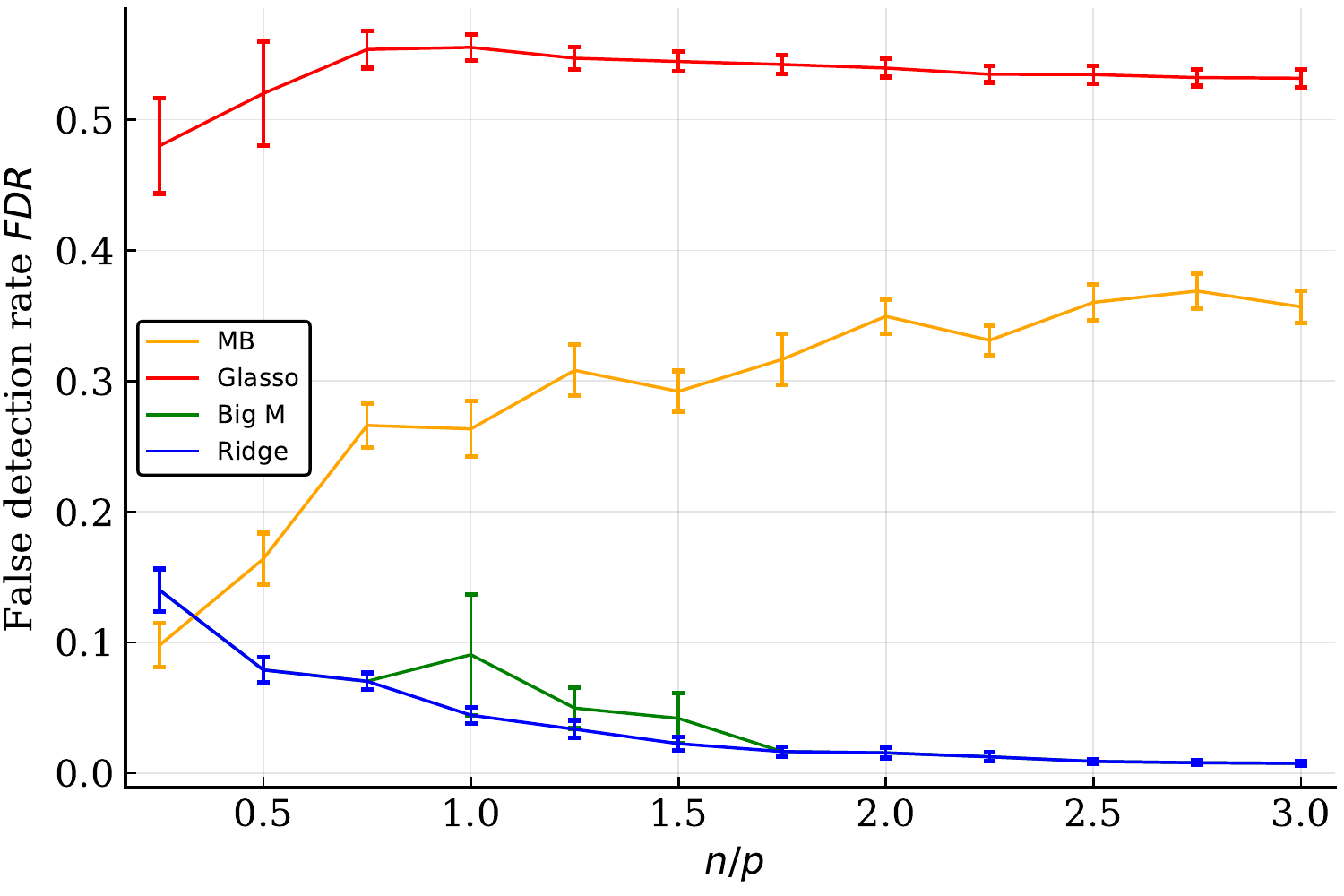}
	\caption{False detection rate $FDR$ vs. $n/p$.}
\end{subfigure}
\caption{Impact of the number of samples $n/p$ on support recovery. Results are averaged over $10$ instances with $p=200$, $t=1\%$. Hyper-parameters are tuned using the $BIC_{1/2}$ criterion.}
\label{fig:scale.n.ebic}
\end{figure}

\subsection{Comparisons for varying sparsity levels $t$}

 \begin{figure}[H]
\centering
\begin{subfigure}[t]{.4\linewidth}
	\centering
	\includegraphics[width=\linewidth]{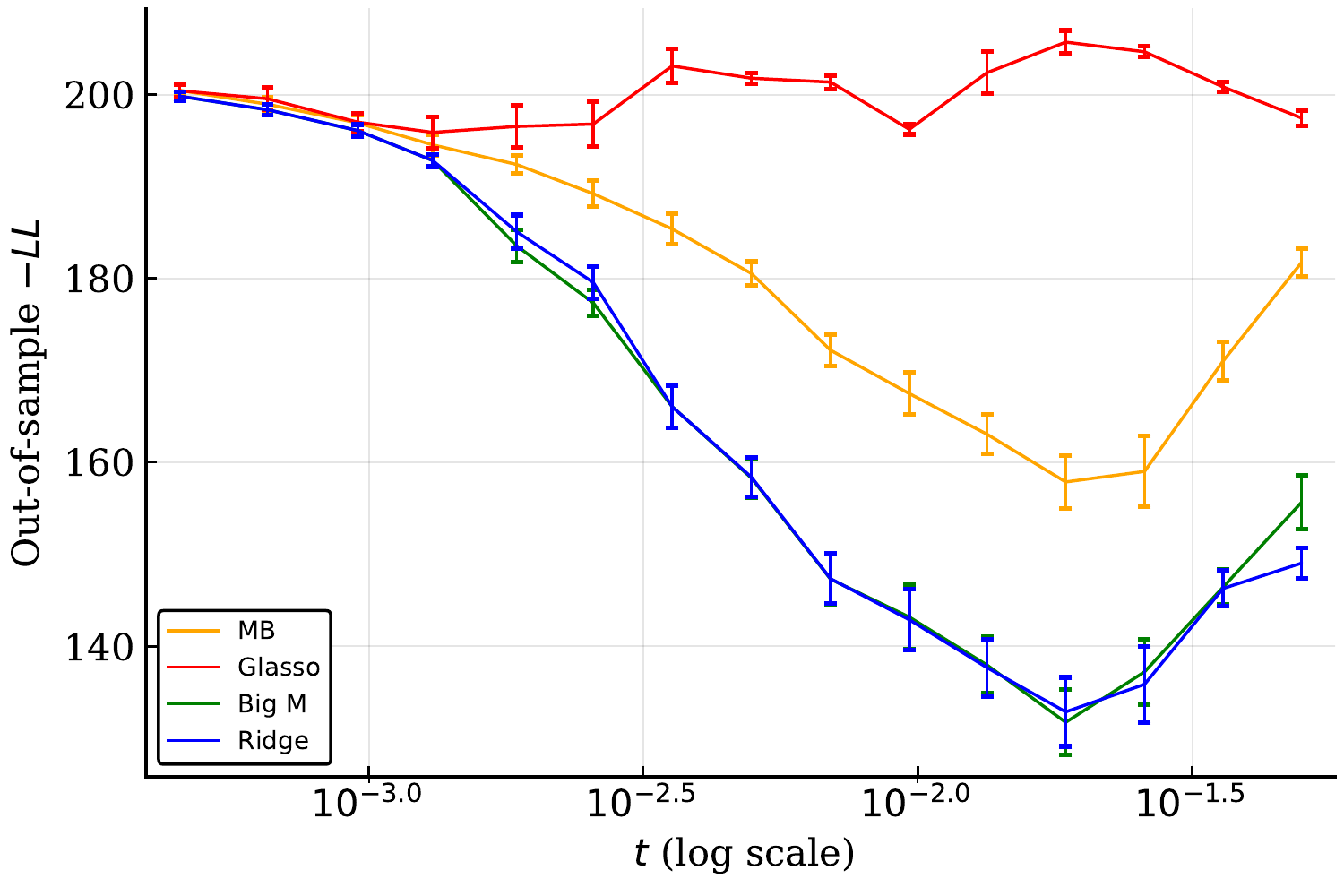}
	\caption{$BIC_{1/2}$ as a CV criterion.}
\end{subfigure} %
~
\begin{subfigure}[t]{.4\linewidth}
	\centering
	\includegraphics[width=\linewidth]{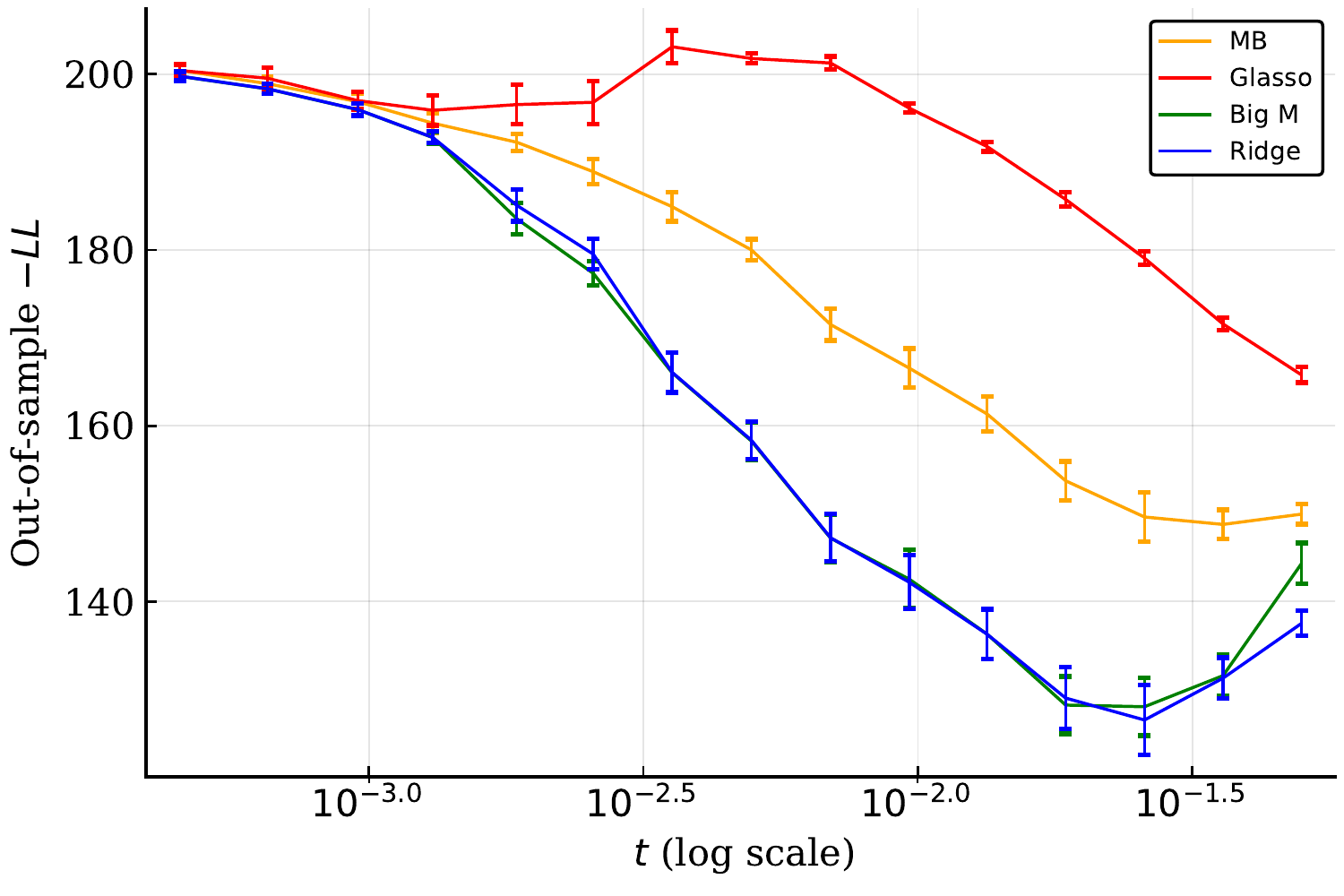}
	\caption{$-LL$ as a CV criterion.}
\end{subfigure}
\caption{Impact of the  sparsity level $t$  on out-of-sample negative log-likelihood. Results are averaged over $10$ instances with $p=200$, $n=p$.}
\label{fig:scale.t.loglik}
\end{figure}
\begin{figure}[H]
\centering
\begin{subfigure}[t]{.4\linewidth}
	\centering
	\includegraphics[width=\linewidth]{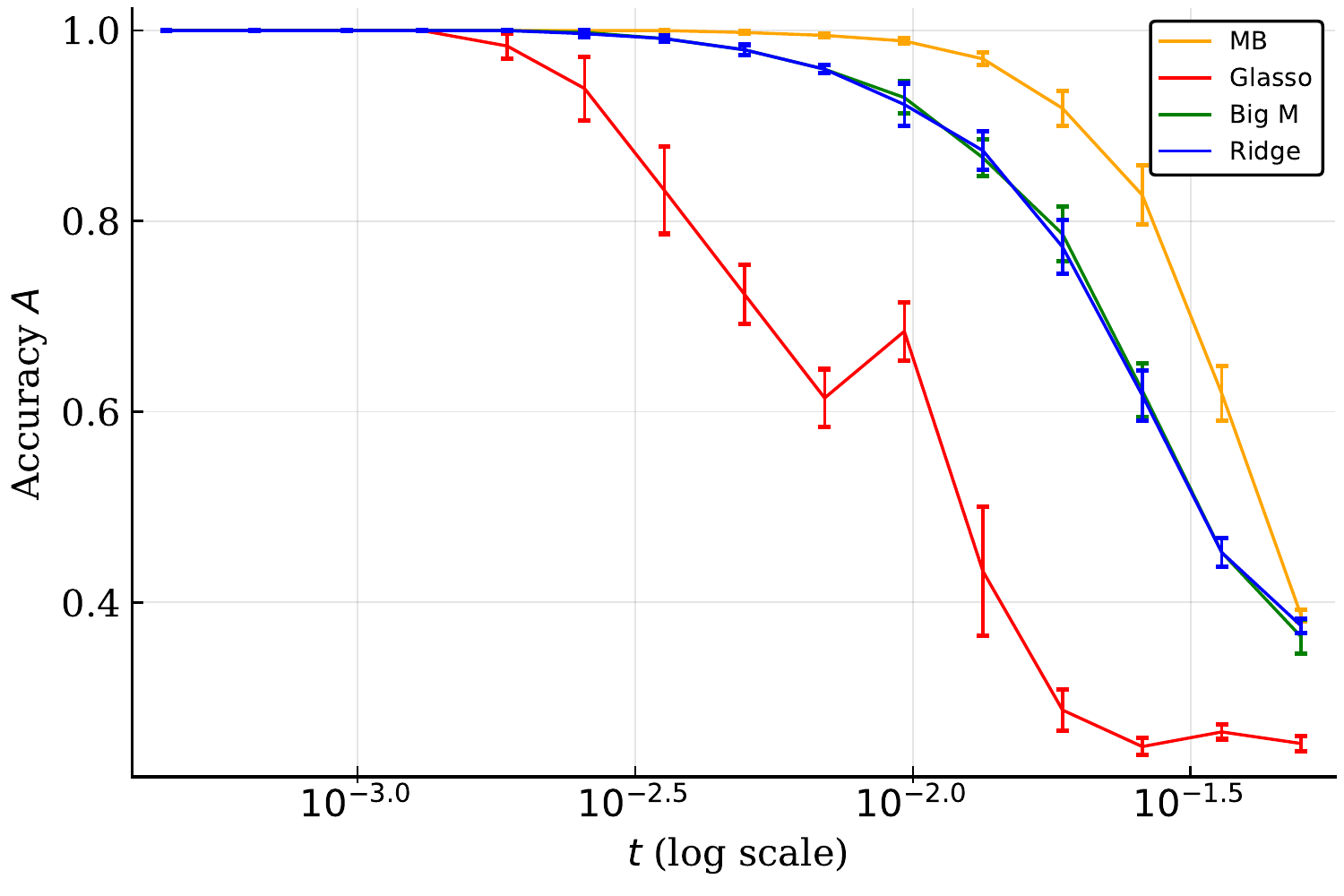}
	\caption{Accuracy $A$ vs. $t$.}
\end{subfigure} %
~
\begin{subfigure}[t]{.4\linewidth}
	\centering
	\includegraphics[width=\linewidth]{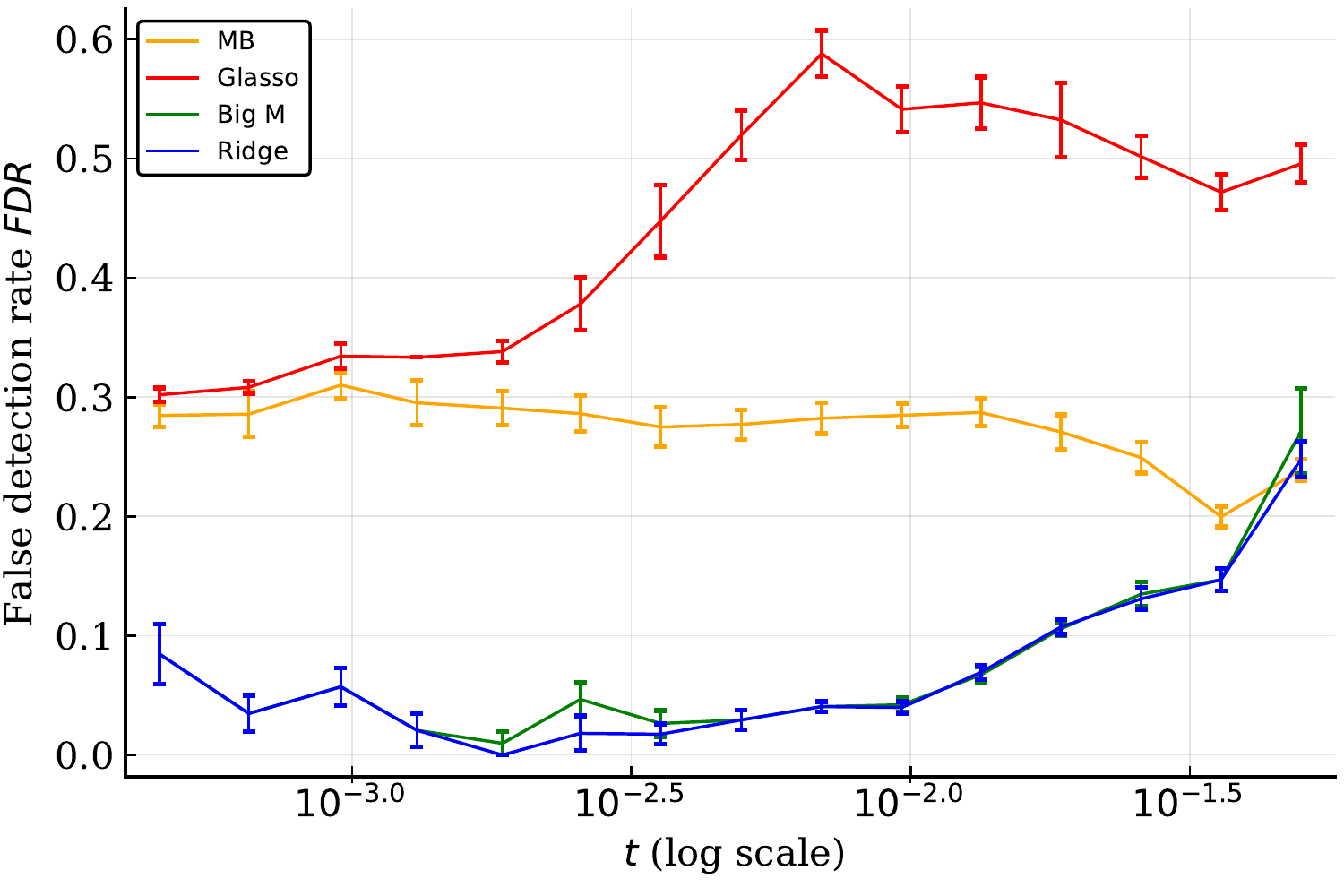}
	\caption{False detection rate $FDR$ vs. $t$.}
\end{subfigure}
\caption{Impact of the sparsity level $t$ on support recovery. Results are averaged over $10$ instances with $p=200$, $n=p$. Hyper-parameters are tuned using the $BIC_{1/2}$ criterion.}
\label{fig:scale.t.ebic}
\end{figure}

\subsection{Comparisons for varying dimensions $p$}

\begin{figure}[H]
\centering
\begin{subfigure}[t]{.4\linewidth}
	\centering
	\includegraphics[width=\linewidth]{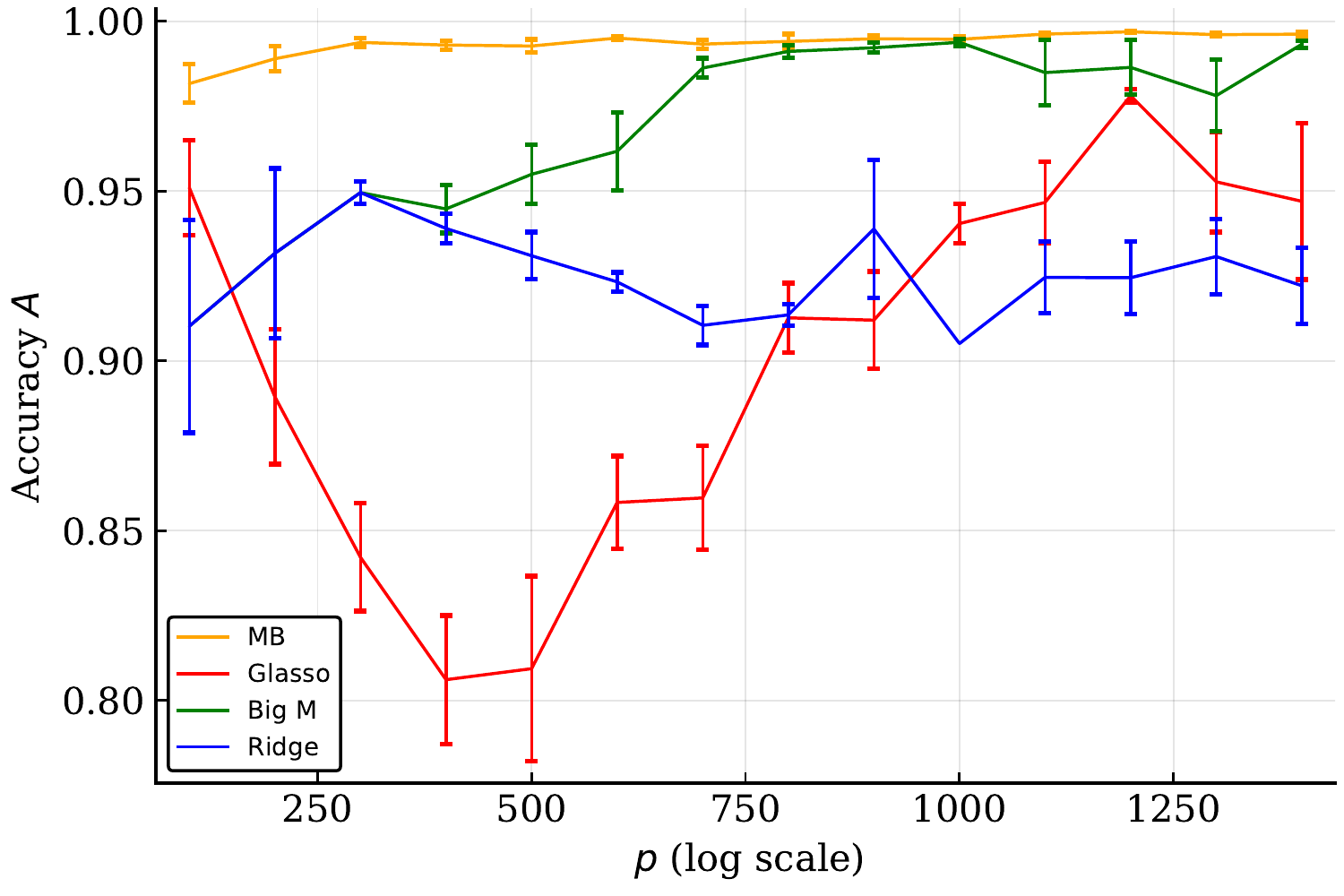}
	\caption{Accuracy $A$ vs. $p$.}
\end{subfigure} %
~
\begin{subfigure}[t]{.4\linewidth}
	\centering
	\includegraphics[width=\linewidth]{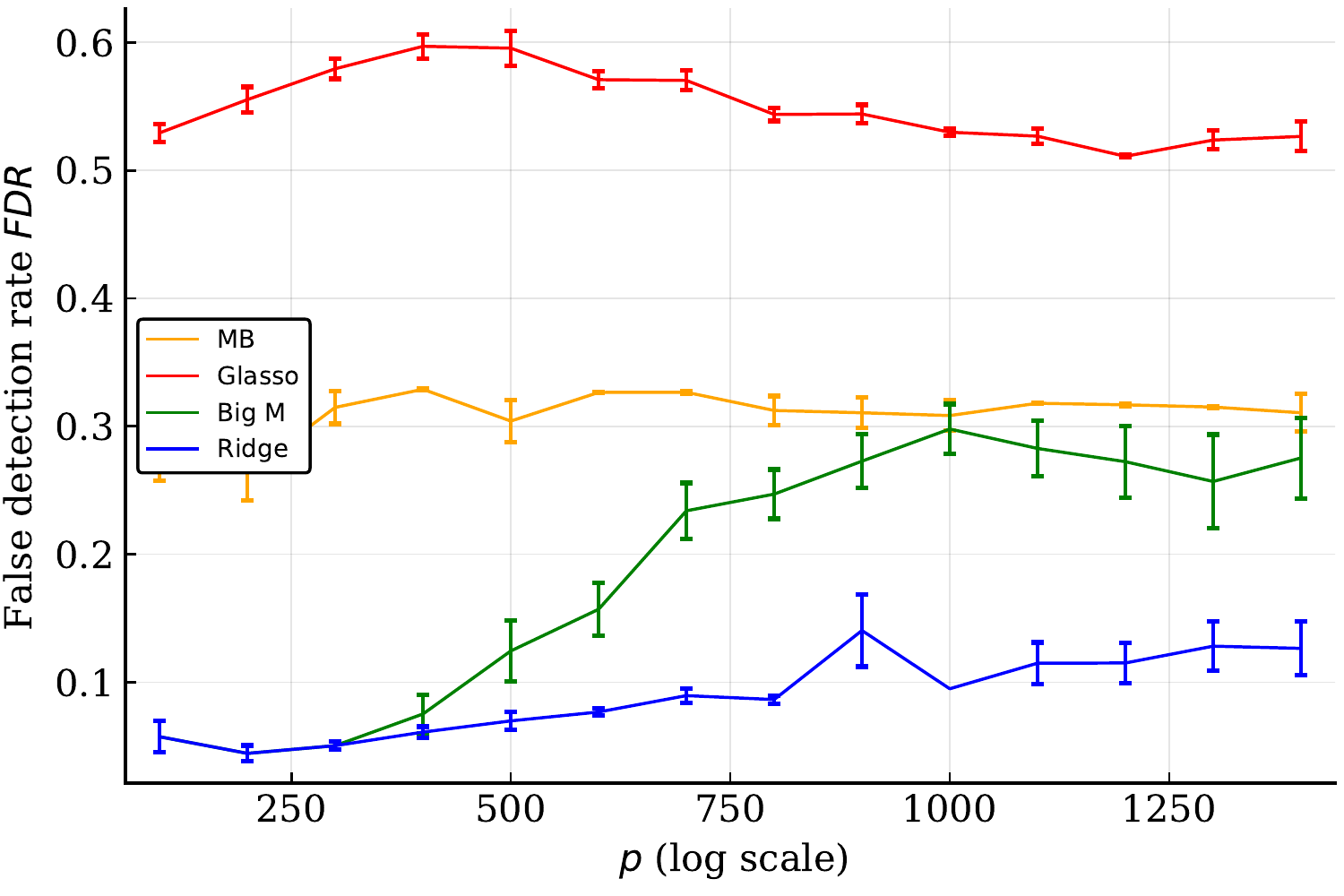}
	\caption{False detection rate $FDR$ vs. $p$.}
\end{subfigure}
\caption{Impact of the dimension $p$ on support recovery. Results are averaged over $10$ instances with $n=p$, $t=1\%$. Hyper-parameters are tuned using $-LL$.}
\label{fig:scale.p.ll}
\end{figure}
\begin{figure}[H]
\centering
\begin{subfigure}[t]{.4\linewidth}
	\centering
	\includegraphics[width=\linewidth]{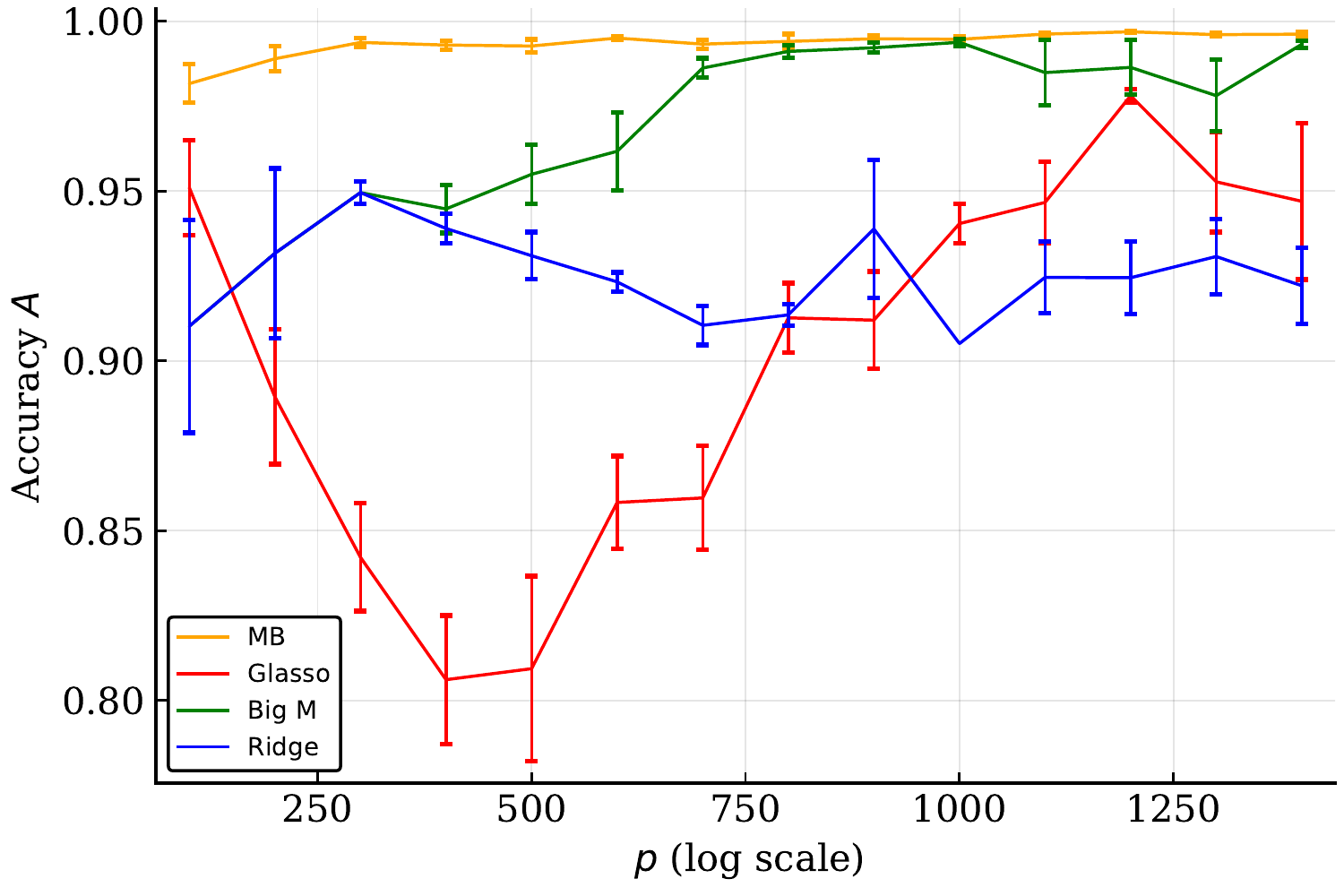}
	\caption{Accuracy $A$ vs. $p$.}
\end{subfigure} %
~
\begin{subfigure}[t]{.4\linewidth}
	\centering
	\includegraphics[width=\linewidth]{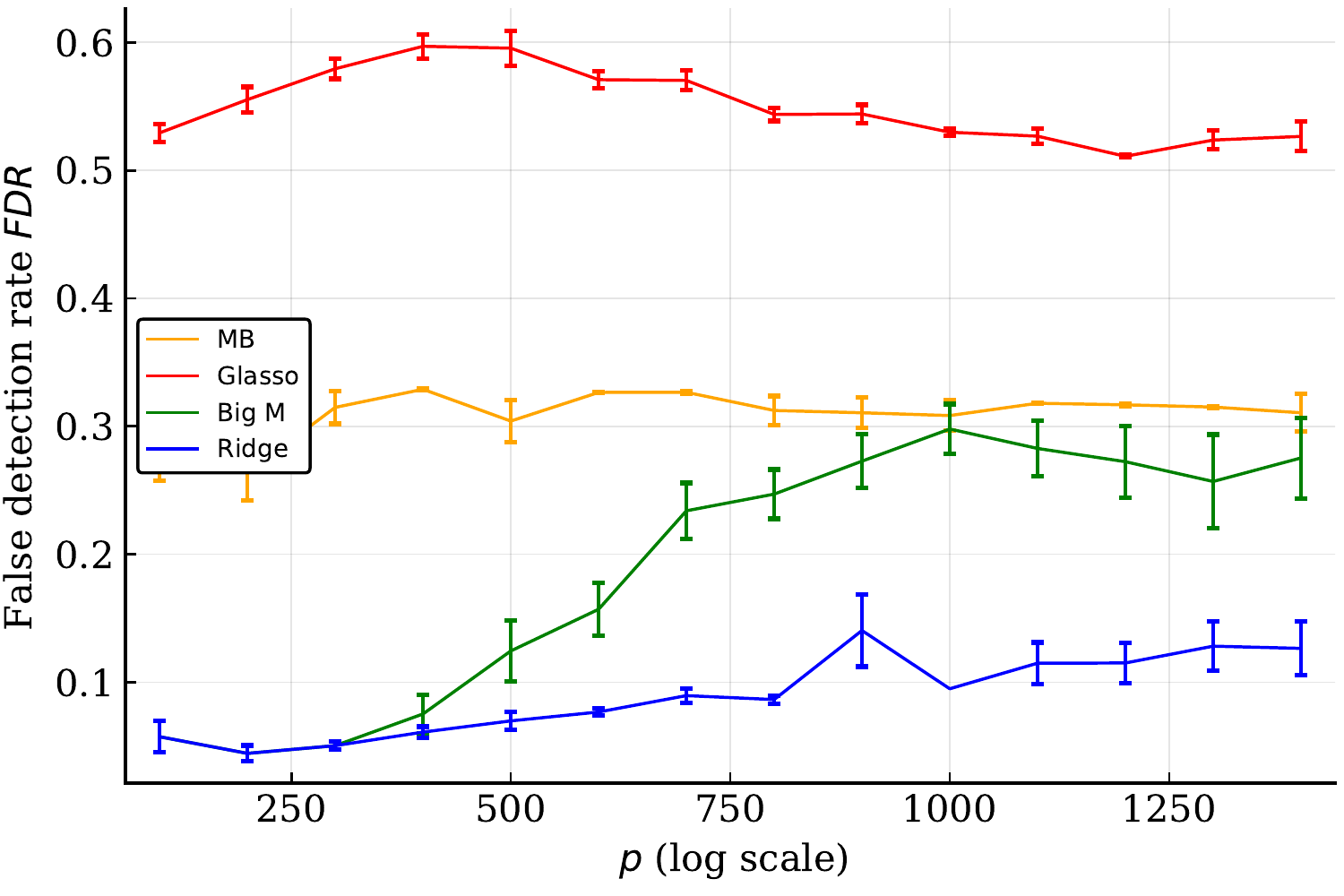}
	\caption{False detection rate $FDR$ vs. $p$.}
\end{subfigure}
\caption{Impact of the dimension $p$ on support recovery. Results are averaged over $10$ instances with $n=p$, $t=1\%$. Hyper-parameters are tuned using the $BIC_{1/2}$ criterion.}
\label{fig:scale.p.ebic}
\end{figure}
 \begin{figure}[H]
\centering
\begin{subfigure}[t]{.4\linewidth}
	\centering
	\includegraphics[width=\linewidth]{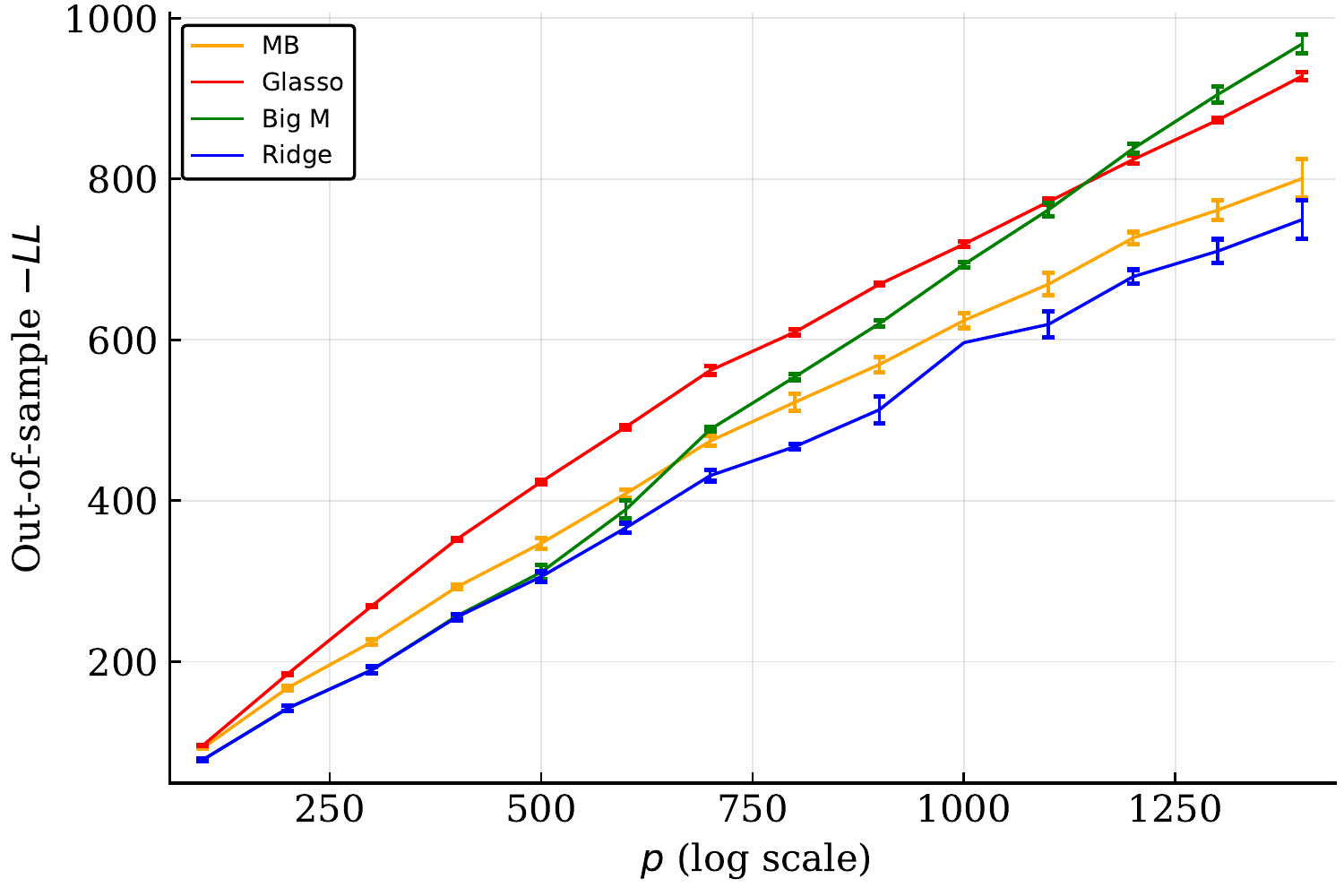}
	\caption{$BIC_{1/2}$ as a CV criterion.}
\end{subfigure} %
~
\begin{subfigure}[t]{.4\linewidth}
	\centering
	\includegraphics[width=\linewidth]{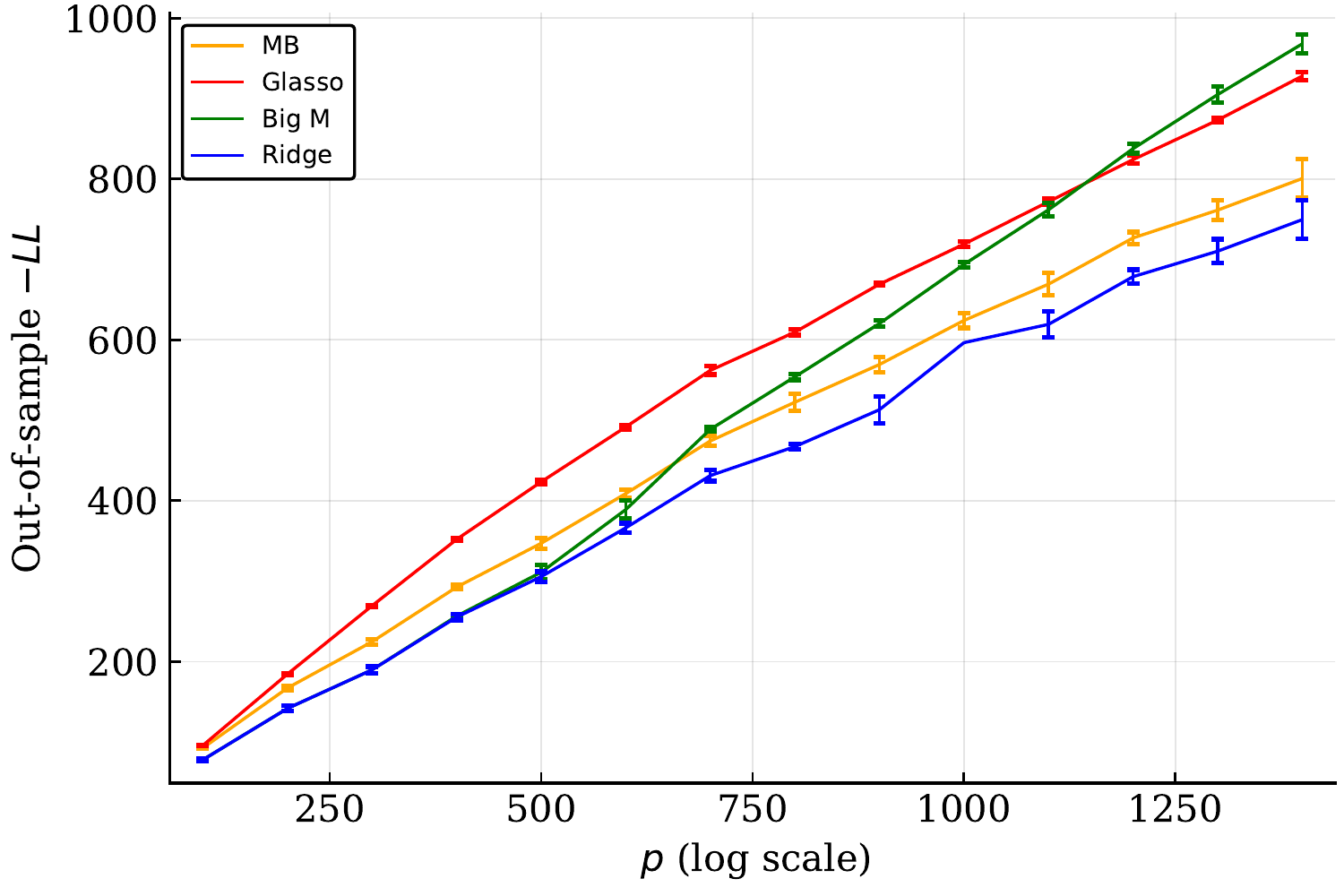}
	\caption{$-LL$ as a CV criterion.}
\end{subfigure}
\caption{Impact of the dimension $p$  on out-of-sample negative log-likelihood. Results are averaged over $10$ instances with $n=p$, $t=1\%$.}
\label{fig:scale.p.loglik}
\end{figure}
 \begin{figure}[H]
\centering
\begin{subfigure}[t]{.4\linewidth}
	\centering
	\includegraphics[width=\linewidth]{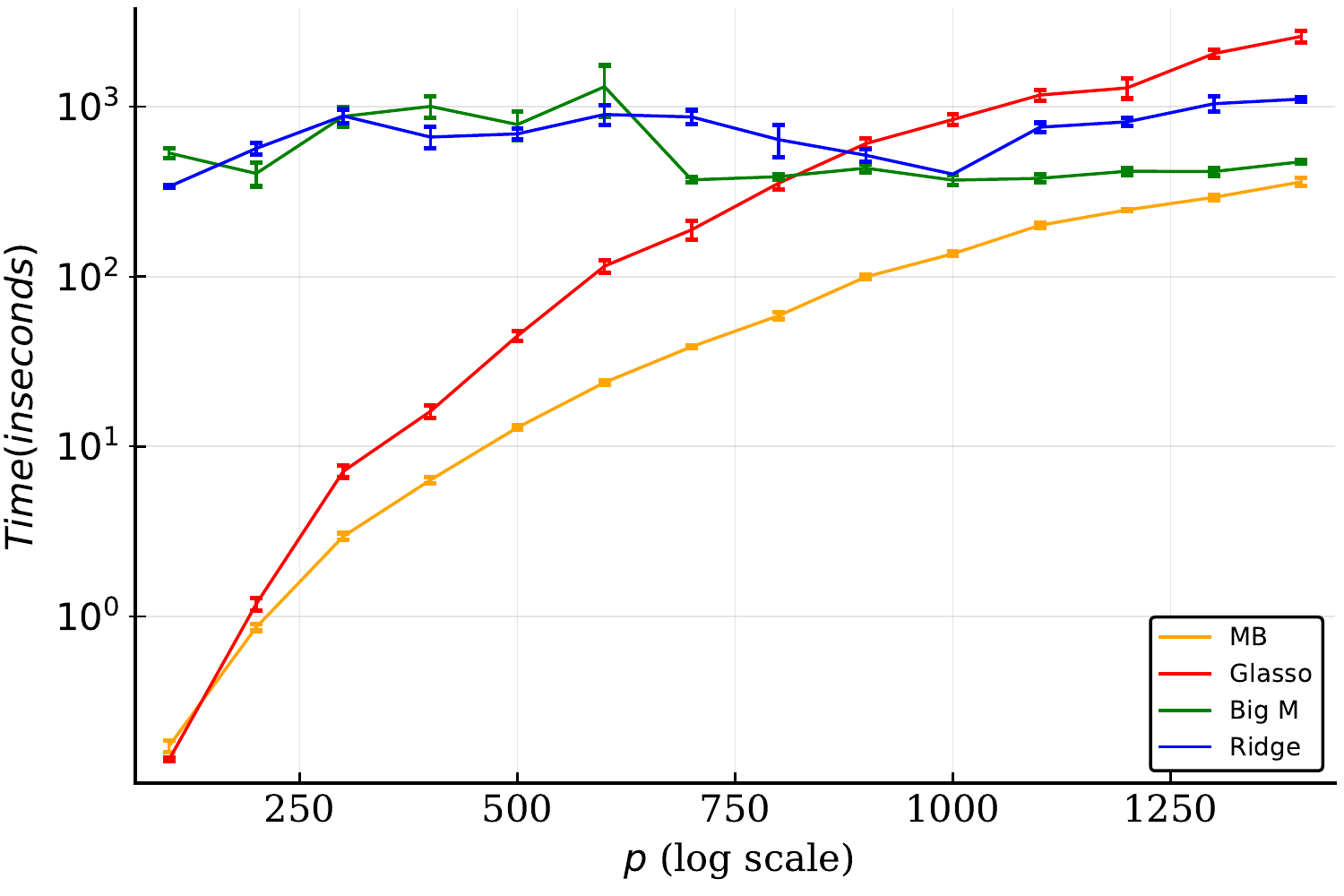}
	\caption{$BIC_{1/2}$ as a CV criterion.}
\end{subfigure} %
~
\begin{subfigure}[t]{.4\linewidth}
	\centering
	\includegraphics[width=\linewidth]{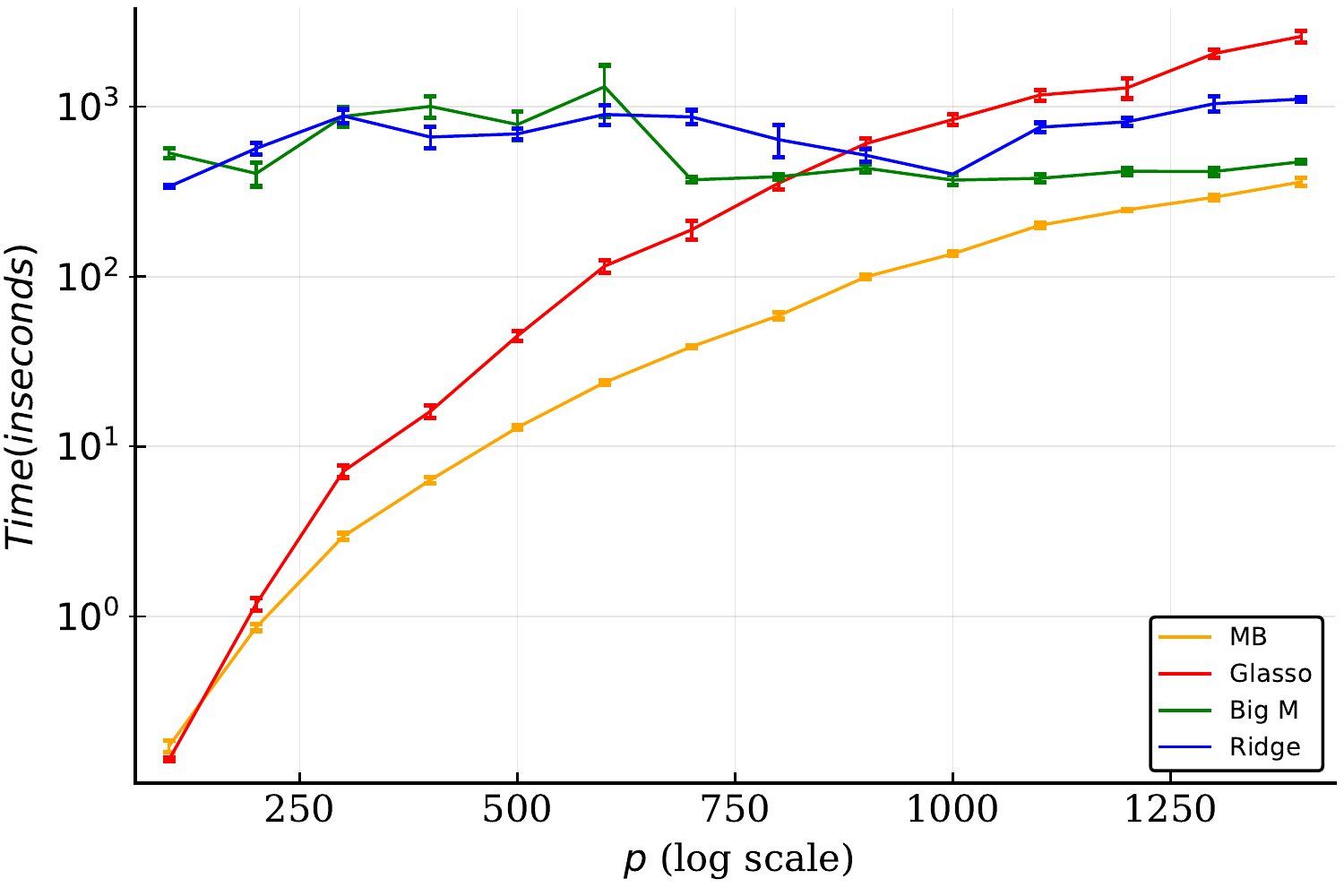}
	\caption{$-LL$ as a CV criterion.}
\end{subfigure}
\caption{Impact of the dimension $p$  on computational time. Results are averaged over $10$ instances with $n=p$, $t=1\%$. Recall that discrete formulations big-$M$ and ridge are stopped after $5$ minutes.}
\label{fig:scale.p.time}
\end{figure}

% -------------------------------------------------------------------------------------------------
% REFERENCES
% -------------------------------------------------------------------------------------------------

\bibliographystyle{plain}
\bibliography{biblio}
\end{document}